\documentclass[11pt]{article}

\usepackage[margin=1.2in]{geometry}

\usepackage{amsmath,amssymb,amsfonts,amsthm,amsbsy}
\usepackage{mathtools}
\usepackage{bm}
\usepackage{mathrsfs}
\usepackage{bbm}
\usepackage{cases}
\usepackage{wasysym}
\usepackage{xparse}

\usepackage{graphicx}
\usepackage{subcaption}
\usepackage{float}
\usepackage{placeins}
\usepackage{booktabs}
\usepackage{multirow}
\usepackage{tabularx}
\usepackage{array}
\usepackage{ragged2e}

\newcolumntype{L}[1]{>{\RaggedRight\arraybackslash}p{#1}}
\newcolumntype{Y}{>{\RaggedRight\arraybackslash}X}

\usepackage{algorithm}
\usepackage{algpseudocode}

\usepackage{xcolor}
\usepackage{url}
\usepackage{nicefrac}
\usepackage{microtype}
\usepackage{enumitem}
\usepackage{comment}
\usepackage{authblk}
\usepackage{todonotes}
\usepackage[normalem]{ulem}
\usepackage{soul}
\usepackage{tikz}
\usepackage{multicol}

\usepackage[colorlinks=true,linkcolor=blue,citecolor=blue,urlcolor=blue]{hyperref}

\theoremstyle{plain}
\newtheorem{theorem}{Theorem}[section]

\newtheorem*{thm*}{Theorem}

\newtheorem{proposition}[theorem]{Proposition}

\theoremstyle{definition}

\newtheorem{assumption}{Assumption}
\newtheorem*{assumption*}{Assumption}

\theoremstyle{remark}
\newtheorem{remark}[theorem]{Remark}

\NewDocumentCommand{\norm}{O{} m g}{%
  \left\lVert #2 \right\rVert%
  \IfNoValueTF{#3}{%
    \if\relax\detokenize{#1}\relax
    \else
      _{#1}%
    \fi
  }{%
    _{#3}%
  }%
}

\algrenewcommand{\algorithmicrequire}{\textbf{Input:}}
\algrenewcommand{\algorithmicensure}{\textbf{Output:}}

\providecommand{\keywords}[1]{\textbf{\textit{Keywords:}} #1}
\title{Inferring physical fields in coupled systems with unknown parameters from incomplete observations using physics-constrained attentive neural operators}

\author[1,2]{Shilun Wei\thanks{
  \texttt{weishlun@mail2.sysu.edu.cn};
  \texttt{253621013@stu.gbu.edu.cn}
}}
\author[1]{Xiaoqiang Sun\thanks{\texttt{sunxq6@mail.sysu.edu.cn}}}
\author[3]{Wei Li\thanks{\texttt{weili.cs@sjtu.edu.cn}}}
\author[2]{Kejun Tang\thanks{Corresponding author: \texttt{tangkj@gbu.edu.cn}}}

\affil[1]{School of Mathematics, Sun Yat-sen University, Guangzhou, China}
\affil[2]{School of Sciences, Great Bay University, Dongguan, China}
\affil[3]{John Hopcroft Center for Computer Science, School of Computer Science,
Shanghai Jiao Tong University, Shanghai, China}

\date{}

\begin{document}

\maketitle

\begin{abstract}
Given incomplete measurements of a single physical field in a coupled system with unknown parameters, can we infer its full physical state and identify the underlying parameters? This problem is challenging because multiple coupled fields must be reconstructed simultaneously from limited observations of only one, while the system parameters are unknown. In this work, we propose a machine learning framework for full-field reconstruction and parameter identification of unknown physical systems from sparse observations of a single physical field. Specifically, the cross-attention encoder propagates sparse sensor observations onto a regular grid to construct a sensor-conditioned latent representation, while a Fourier neural operator (FNO) decoder captures global spatial dependencies to reconstruct all coupled physical fields. The network parameters and unknown physical parameters are jointly optimized by minimizing observation losses, governing equation residuals, and boundary/initial condition constraints. The proposed approach is validated on two- and three-dimensional lid-driven cavity flows, a two-dimensional cylinder wake, and a two-dimensional non-ideal magnetohydrodynamics problem, demonstrating the recovery performance of unobserved fields and physical parameters from incomplete observations.
\end{abstract}

\section{Introduction}
\label{sec:introduction}

Let \(\mathcal D\) denote a computational domain (e.g. a spatial domain \(\Omega\) for a steady
problem or a space--time domain
\(\Omega_T=(0,T]\times\Omega\) for an unsteady problem). We consider the following coupled
physical system
\begin{equation}
\begin{cases}
    \mathcal F_1(u_1,u_2,\lambda)=0,\\
    \mathcal F_2(u_1,u_2,\lambda)=0,\\
    \mathcal C(u_1,u_2,\lambda)=0,
\end{cases}
\qquad \text{in } \mathcal D,
\label{eq:intro_coupled_system}
\end{equation}
where \(U=(u_1,u_2)\) denotes the coupled physical state associated with the spatial variable $x$,
\(\lambda \in \mathbb R^m\) is a vector collecting some unknown physical parameters in the system, and \(\mathcal C\) collects the associated boundary and initial conditions. We consider a challenging scenario that only a sparse set of possibly irregular observations from \emph{only one} field ($u_1$ or $u_2$) is available. Without loss of generality, we can assume only $u_1$ is observed and some sparse observations are collected
\begin{equation}
    \mathsf S
    =
    \left\{
        (x^{(i)}, u_1^{(i)})
    \right\}_{i=1}^{N_s},
\label{eq:intro_sparse_observations}
\end{equation}
where $x^{(i)} \in \mathcal{D}$ is the sensor location, $u_1^{(i)}$ is the measurement corresponding to $x^{(i)}$. The task is to reconstruct the complete state $U$ through $\mathsf S$, including unobserved physical fields, while simultaneously identifying the unknown
parameters $\lambda$.

This inverse problem is challenging because we need to reconstruct the whole physical field from a few measurements in a coupled physical system. The hidden variables and physical parameters
must therefore be inferred indirectly through the governing equations and the
cross-variable coupling. Reconstruction quality depends on the sensor
configuration, the strength of the physical coupling, the sensitivity of the
observations to the unknown parameters, and the level of measurement noise.

\subsection{Related work}
\label{sec:intro_related_work}
In coupled systems, the fields accessible to measurement do not necessarily include the quantities needed to characterize the state.
For Navier--Stokes flows, sparse velocity observations have been used to infer pressure and unknown coefficients~\cite{raissi2019physics,parfenyev2024inferring}. Hidden Fluid Mechanics instead uses measurements of a transported tracer concentration to recover velocity and pressure through the coupled
transport and flow equations~\cite{raissi2020hidden}. Related work
on vortex-induced vibrations combines flow and structural-motion data to identify pressure and parameters \cite{raissi2019deep}. Thermal measurements offer another route: temperature fields obtained by tomographic background-oriented Schlieren imaging have been used to infer three dimensional velocity and pressure \cite{cai2021espresso}, whereas velocity measurements have been used to infer temperature in turbulent convection \cite{toscano2025aivt}.

The same measurement mismatch arises in magnetohydrodynamics (MHD). Externally measured magnetic fields have been used to reconstruct velocity in a low-magnetic-Reynolds-number regime \cite{stefani2006inverse}; sparse-data field and parameter inversion has also been studied for incompressible magnetohydrodynamics~\cite{guan2025trajectory}. In subsurface flow and transport, hydraulic head, concentration, and conductivity can be inferred jointly from incomplete measurements
\cite{he2020physics}. These examples establish the practical motivation for recovering an unobserved field through its physical coupling with an observed one. While these applications motivate recovering hidden fields from observable ones, they usually assume well-placed sensors and fixed system parameters.

Classical inverse formulations combine a forward physical model with observations and assumptions about the unknown state or parameters. Bayesian formulations have, for example, been developed
for function-space inverse problems in fluid mechanics \cite{cotter2009bayesian}; broader connections between inverse problems and reduced models are reviewed in \cite{ghattas2021learning}. The quality of an inversion depends on the forward model and on the information supplied by the measurements. In physics-constrained learning, Xu and Darve
\cite{xu2022physics} used differentiable numerical solvers to enforce PDE constraints during inversion from sparse observations.
Our formulation instead fits a neural representation of the fields and parameters using observation and equation residuals. This
distinction is relevant when comparing the role of the PDE and the computational cost of the two approaches.

POD and Gappy POD estimate coefficients in a reduced basis from incomplete measurements
\cite{sirovich1987turbulence,willcox2006unsteady}. Voronoi-assisted reconstruction maps irregular sensors to spatial representations \cite{fukami2021global}, while Senseiver processes variable-size sensor sets using attention \cite{vaswani2017attention,santos2023development}. DeepONet and the Fourier Neural Operator provide representations of mappings between functions \cite{lu2021learning,li2020fourier,kovachki2023neural}. RecFNO applies Fourier operator layers to sparse flow and temperature reconstruction \cite{zhao2024recfno}. These models show how local measurements can be incorporated into
a global field representation, although their training requirements depend on the particular reconstruction setup.

Physics-informed neural networks combine measurement errors with governing-equation and boundary-condition residuals \cite{raissi2019physics}. Recent approaches to flow reconstruction use wavelet-based residual networks, attention-based representations, and domain decomposition \cite{barman2026efficient,barman2026physicsformer,
qian2026distributed}. Physics-constrained convolutional networks have also reconstructed unsteady flow from sparse and noisy data
without complete field labels \cite{mo2025reconstructing}.
Thus, lack of full-field supervision is not unique to the present work. Training can nevertheless be sensitive to PDE difficulty
and to the relative behavior of different loss terms \cite{krishnapriyan2021characterizing,wang2022when}.

Physics-informed DeepONet and Physics-Informed Neural Operator incorporate equation constraints into operator representations \cite{wang2021learning,li2024physics}. Physics-informed neural
operators have also been applied to two-dimensional incompressible
MHD simulation \cite{rosofsky2023mhd}; this is a forward modeling setting rather than joint inversion of hidden fields and physical
parameters. Operator-based parameter identification from sparse sensors has been considered as well, including the use of a
pretrained DeepONet to initialize a physics-constrained inverse model \cite{sarkar2026learning}. Despite these advances, existing methods rarely tackle the joint challenge of handling irregular single-field sensors, recovering unobserved physics, and identifying unknown system parameters within a unified framework. We therefore focus on the combined setting studied here: reconstructing the full coupled state and identifying unknown parameters from irregular, single-channel measurements under PDE constraints.

\subsection{Proposed framework and contributions}
\label{sec:intro_connections}
To make progress, our goal is to recover coupled physical fields and unknown parameters from sparse observations of one single field. The primary setting uses measurements of one field; selected MHD experiments additionally use sparse pressure observations to assess the value of auxiliary
information. We represent the observations as sensor tokens, transfer their information to a computational grid through cross-attention,
and use an FNO to parameterize the complete state. The network and physical parameters are optimized for each inverse instance using
observation, equation, boundary and initial condition constraints. The main contributions of this work are summarized as follows:

\begin{enumerate}
    \item We design a novel framework that integrates irregular sensor encoding and residual-aware attention within an FNO representation to perform joint field reconstruction and parameter identification.

    \item Under local regularity and sensor-observability assumptions, we derive a conditional stability estimate bounding reconstruction errors by PDE residuals, observation mismatch, discretization errors, and noise, while providing a computational method to validate the key assumption.

    \item We evaluate our method on lid-driven cavity flows, a cylinder wake, and a non-ideal MHD system under different observation setups and noise levels.
\end{enumerate}

\subsection{Organization}
\label{sec:intro_organization}
The remainder of this paper is organized as follows. In section~\ref{sec:theory_sparse_inverse}, we present the abstract formulation of the studied problem, the local state-stability structure shared by the flow and MHD systems, sensor-dependent parameter identifiability, and the resulting conditional stability estimate. After that, the proposed physics-constrained attentive FNO framework, including the sensor-token representation, cross-attention encoder,
spectral decoder, physical parameterization, and joint optimization objective, is introduced in section~\ref{sec:method}. In section~\ref{sec:numerical_study},  we demonstrate the effectiveness of our approach with numerical experiments and ablation studies on the flow and magnetohydrodynamic benchmarks.
Finally, section~\ref{sec:conclusion} summarizes the main findings and discusses the limitations and directions for future work.

\section{Theoretical analysis for sparse inverse reconstruction in coupled multiphysics systems}
\label{sec:theory_sparse_inverse}

We now formulate the inverse problem introduced in
Section~\ref{sec:introduction} in a form suitable for the stability analysis.
The notation \(u_1\) and \(u_2\) is deliberately abstract: each may be a scalar-
or vector-valued physical field, and their concrete meaning changes from one
benchmark to another. In the basic setting, \(u_1\) denotes the observed field
and \(u_2\) denotes the remaining coupled field, or a block of fields, to be
reconstructed.

Let
\begin{equation*}
    U=(u_1,u_2)\in\mathcal X_U:=\mathcal X_1\times\mathcal X_2,
    \qquad
    \lambda\in\Lambda\subset\mathbb R^m,
    \qquad
    Z=(U,\lambda)\in\mathcal X:=\mathcal X_U\times\mathbb R^m,
    \label{eq:all_at_once_variable}
\end{equation*}
where \(\mathcal X_1\) and \(\mathcal X_2\) are two proper function spaces for the two
field blocks, and \(\Lambda\) is an admissible parameter set. The coupled physical state is denoted by $U$, while the joint variable is denoted by $Z$. Throughout this section, $|\cdot|$ denotes the Euclidean norm on
\(\mathbb R^m\) and $\Vert \cdot \Vert_{\mathcal{F}}$ denotes the norm in the space $\mathcal{F}$. We define
\(\mathcal X_{\rm ad}:=\mathcal X_U\times\Lambda\) as the admissible
subset of the ambient space \(\mathcal X\). With this convention, the
residual map associated with
\eqref{eq:intro_coupled_system} is
\begin{equation}
    \mathcal R:\mathcal X\to\mathcal Y,
    \qquad
    \mathcal R(Z)
    :=
    \begin{pmatrix}
        \mathcal F_1(u_1,u_2,\lambda)\\
        \mathcal F_2(u_1,u_2,\lambda)\\
        \mathcal C(u_1,u_2,\lambda)
    \end{pmatrix},
    \label{eq:general_all_at_once_residual}
\end{equation}
where \(\mathcal Y\) is the product space of the PDE residual, the boundary residual and the initial residual. Let \(Z^\dagger=(U^\dagger,\lambda^\dagger)\) be the physical state and parameter pair \emph{to be recovered}. It naturally satisfies \(\mathcal R(Z^\dagger)=0\). The feasible set of the joint inverse problem is
\begin{equation*}
    \mathcal M
    :=\{Z\in\mathcal X_{\rm ad}:\mathcal R(Z)=0\}.
    \label{eq:coupled_solution_set}
\end{equation*}
Local identifiability at \(Z^\dagger\) means that the restriction
\(\mathcal O_S|_{\mathcal M}\) is injective in a neighborhood of
\(Z^\dagger\).

The sensor set
\(\mathsf S\) in \eqref{eq:intro_sparse_observations} induces a bounded
observation map
\begin{equation*}
    \mathcal O_S:\mathcal X\to\mathcal Y_{\rm obs}.
    \label{eq:general_observation_operator}
\end{equation*}
The subscript \(S\) indicates the dependence on \(\mathsf S\). In the
single-field setting, \(\mathcal O_S(Z)\) depends on \(u_1\) only; allowing a
direct dependence on \(\lambda\) costs no additional argument and is retained
for completeness. The noisy data are written as
\begin{equation}
    y^\delta
    =\mathcal O_S(Z^\dagger)+\varepsilon^\delta,
    \qquad
    \|\varepsilon^\delta\|_{\mathcal Y_{\rm obs}}\leq\delta.
    \label{eq:general_noisy_observation_model}
\end{equation}
For independent Gaussian measurement errors, the last inequality can be read
for the realized noise vector, or on an event with the prescribed confidence
level. Boundedness of \(\mathcal O_S\) is automatic for finite-dimensional
grid states. In a function-space formulation with point sensors, it requires enough regularity for point evaluation, or the sensors may instead be modeled as local averages. For the differentiability arguments below, \(\lambda^\dagger\) is assumed to be an interior point of \(\Lambda\). Parameters on the boundary require a one-sided or tangent-cone formulation, which is not considered here.

\subsection{Local regularity and sensor observability}
\label{subsec:coupled_identifiability}
\label{subsec:linearized_coupled_observability}
\label{subsec:parameter_to_state_sufficient_condition}

Define the state and parameter derivatives of the residual map by
\begin{equation*}
    \mathcal L^\dagger
    :=D_U\mathcal R(Z^\dagger),
    \qquad
    \mathcal R_\lambda^\dagger
    :=D_\lambda\mathcal R(Z^\dagger).
    \label{eq:linearized_state_operator}
\end{equation*}
For example, for the Navier--Stokes and non-ideal MHD systems considered below, the
differential part of \(\mathcal L^\dagger\) has the schematic decomposition
\begin{equation*}
    \mathcal L^\dagger
    =\mathcal L_{\rm diss}^\dagger
    +\mathcal L_{\rm tr}^\dagger
    +\mathcal L_{\rm cpl}^\dagger,
    \label{eq:dissipative_transport_coupling_decomposition}
\end{equation*}
where the three terms represent dissipation, transport, and cross-field
coupling. In suitable energy spaces
\(\mathcal V\hookrightarrow\mathcal H\), their typical behavior is
\begin{align}
    \langle\mathcal L_{\rm diss}^\dagger h,h\rangle
    &\geq c_D\|\Pi_{\rm d}h\|_{\mathcal V}^2,
    \label{eq:dissipative_coercivity}\\
    \left|\langle(\mathcal L_{\rm tr}^\dagger
    +\mathcal L_{\rm cpl}^\dagger)h,h\rangle\right|
    &\leq
    \zeta\|\Pi_{\rm d}h\|_{\mathcal V}^2
    +C_\zeta\|h\|_{\mathcal H}^2,
    \qquad \zeta>0,
    \label{eq:transport_coupling_relative_bound}
\end{align}
where \(h\in\mathcal V\) is an admissible state perturbation,
\(\langle\cdot,\cdot\rangle\) is the natural
\(\mathcal V^\ast\)--\(\mathcal V\) duality pairing for the differential
part of the linearized equations, \(c_D>0\), and \(C_\zeta\geq0\).
The projection \(\Pi_{\rm d}\) selects the directly diffusive components.
These relations only describe the PDE mechanism behind the stability
assumption; they are not, by themselves, an invertibility proof. The required
system-specific calculations are collected in
Appendix~\ref{app:system_structural_details}.

\begin{assumption}
\label{ass:constrained_state_stability}
The map \(\mathcal R\) is continuously Fr\'echet differentiable near
\(Z^\dagger\). After the boundary, initial, compatibility conditions included in \(\mathcal R\) have been imposed,
\(\mathcal L^\dagger\) is a bounded isomorphism from
\(\mathcal X_U\) onto \(\mathcal Y\), and there exists a constant $C_F$ such that
\begin{equation*}
    \|(\mathcal L^\dagger)^{-1}r\|_{\mathcal X_U}
    \leq C_F\|r\|_{\mathcal Y}
    \qquad \text{for all }r\in\mathcal Y.
    \label{eq:forward_linearized_stability}
\end{equation*}
\end{assumption}

This assumption is local to the regular solution branch through
\(Z^\dagger\). For example, it rules out bifurcation points, nontrivial constrained
kernels, loss of regularity, and degenerate dissipation. The implicit
function theorem then gives a local state map
\(\mathscr U:\Lambda\to\mathcal X_U\) such that
\begin{equation}
    U=\mathscr U(\lambda),
    \qquad
    \mathcal T^\dagger
    :=D\mathscr U(\lambda^\dagger)
    =-(\mathcal L^\dagger)^{-1}\mathcal R_\lambda^\dagger.
    \label{eq:parameter_to_state_derivative}
\end{equation}
The symbol \(\mathscr U\) is used here to avoid confusing the PDE solution
map with the sensor set \(\mathsf S\).

The derivatives of the observation operator at the reference pair are
\begin{equation*}
    \mathcal H_U^\dagger:=D_U\mathcal O_S(Z^\dagger),
    \qquad
    \mathcal H_\lambda^\dagger:=D_\lambda\mathcal O_S(Z^\dagger).
    \label{eq:state_parameter_observation_derivatives}
\end{equation*}
For direct measurements of state variables,
\(\mathcal H_\lambda^\dagger=0\). The reduced parameter-to-observation map
and its derivative are
\begin{equation*}
    \Phi(\lambda)
    :=\mathcal O_S(\mathscr U(\lambda),\lambda),
    \qquad
    J_\lambda
    :=D\Phi(\lambda^\dagger)
    =\mathcal H_U^\dagger\mathcal T^\dagger
    +\mathcal H_\lambda^\dagger.
    \label{eq:reduced_sensitivity_matrix}
\end{equation*}

Here, $J_\lambda\in
\mathcal L(\mathbb R^m,\mathcal Y_{\mathrm{obs}})$ is the local
parameter-to-observation sensitivity operator. For a parameter
perturbation direction $\xi\in\mathbb R^m$, the quantity
$J_\lambda\xi$ represents the first-order change in the sensor
observations induced by $\xi$ along the local PDE solution branch.
For a finite set of observations with
$\mathcal Y_{\mathrm{obs}}\simeq\mathbb R^{n_y}$,
$J_\lambda$ is represented by a matrix in
$\mathbb R^{n_y\times m}$, whose $j$th column gives the observation
sensitivity with respect to the $j$th parameter.

\begin{assumption}
\label{ass:parameter_observation_sensitivity}
For the selected observed field, sensor locations, temporal coverage,
and parameterization, there exists a constant $\sigma_*>0$ such that
\begin{equation*}
    \sigma_{\min}(J_\lambda)
    :=
    \inf_{\substack{\xi\in\mathbb R^m\\|\xi|=1}}
    \|J_\lambda\xi\|_{\mathcal Y_{\mathrm{obs}}}
    \geq \sigma_*>0.
    \label{eq:sensitivity_min_singular_value}
\end{equation*}
\end{assumption}

Assumption~\ref{ass:parameter_observation_sensitivity} depends on the
particular experiment. A small \(\sigma_*\) means that at least one parameter
direction produces little change in the measured field. When parameters or
observations have different units, the singular values should be computed
after non-dimensionalization. It is worth noting that this assumption can be numerically verified and we give a concrete example in Appendix~\ref {app:a2_diagnostic} to show this in detail. In Appendix~\ref {app:a2_diagnostic}, we report an explicit finite-difference check for the two-dimensional cavity problem at \(\mathrm{Re}=5000\) using the boundary pressure sensor configurations employed in the numerical study.

\begin{proposition}
\label{prop:forward_sensitivity_implies_observability}
Under Assumptions~\ref{ass:constrained_state_stability} and
\ref{ass:parameter_observation_sensitivity}, there is a constant
\(C_{\rm obs}>0\) such that, for every
\((h,\xi)\in\mathcal X_U\times\mathbb R^m\),
\begin{equation}
\begin{aligned}
    \|h\|_{\mathcal X_U}+|\xi|
    \leq C_{\rm obs}\Big(&
        \|\mathcal L^\dagger h
          +\mathcal R_\lambda^\dagger\xi\|_{\mathcal Y}
        \\
        &+
        \|\mathcal H_U^\dagger h
          +\mathcal H_\lambda^\dagger\xi
        \|_{\mathcal Y_{\rm obs}}
    \Big).
\end{aligned}
\label{eq:general_linearized_observability_estimate}
\end{equation}
Thus a nonzero state--parameter perturbation cannot be invisible to both the
linearized equations and the measurements.
\end{proposition}

The proof is given in Appendix~\ref{app:proof_linear_observability}. The
constant \(C_F\) in Assumption~\ref{ass:constrained_state_stability} describes the conditioning of the state problem, whereas
\(\sigma_*^{-1}\) Assumption~\ref{ass:parameter_observation_sensitivity} describes the conditioning of parameter recovery. It should be noted that good state reconstruction does not imply accurate parameter identification.

\subsection{Joint local stability and the training objective}
\label{subsec:joint_stability_loss}
Now we give the relationship between the training loss and the reconstruction error. 
For \(Z^\dagger=(u_1^\dagger,u_2^\dagger,\lambda^\dagger)\) and any \(Z=(u_1,u_2,\lambda)\), we define the reconstruction error
\begin{equation*}
    d_{\mathcal X}(Z,Z^\dagger)
    :=\|u_1-u_1^\dagger\|_{\mathcal X_1}
    +\|u_2-u_2^\dagger\|_{\mathcal X_2}
    +|\lambda-\lambda^\dagger|.
    \label{eq:joint_state_parameter_distance}
\end{equation*}
Consider the residual map
\(\mathcal R=(\mathcal R_1, \mathcal R_2,\mathcal R_3)\) in \eqref{eq:general_all_at_once_residual},  with
\begin{equation}
    \|\mathcal R(Z)\|_{\mathcal Y}^2
    :=\sum_{k=1}^{3}\|\mathcal R_k(Z)\|_{\mathcal Y_k}^2,
    \qquad
    \mathcal Y:=\prod_{k=1}^{3}\mathcal Y_k.
    \label{eq:residual_product_norm}
\end{equation}
The continuous physics-constrained objective is
\begin{equation}
\begin{aligned}
    \mathcal J_\delta(Z)
    :={}&
    \sum_{k=1}^{3}\gamma_k
    \|\mathcal R_k(Z)\|_{\mathcal Y_k}^2
    +\gamma_{\rm obs}
    \|\mathcal O_S(Z)-y^\delta\|_{\mathcal Y_{\rm obs}}^2
    +\gamma_{\rm reg}\mathcal L_{\rm reg}(Z),
\end{aligned}
\label{eq:general_continuous_physics_loss}
\end{equation}
where \(\gamma_k,\gamma_{\rm obs}>0\),
\(\gamma_{\rm reg}\geq0\) are scalar coefficients to balance the three terms appropriately, and \(\mathcal L_{\rm reg}\) is a regularization term. The optional
regularizer is absent when \(\gamma_{\rm reg}=0\). In practical implementation, we can use the weighted observation norm for \(\|\cdot\|_{\mathcal Y_{\rm obs}}\). More precisely, if \(\mathsf W_{\rm obs}\) is a positive sensor-weight operator, then we can define a weighted norm \(\|q\|_{\mathcal Y_{\rm obs}, \mathsf W_{\rm obs}}:=
\|\mathsf W_{\rm obs}^{1/2}q\|_{\mathcal Y_{\rm obs}}\) with the eigenvalues of \(\mathsf W_{\rm obs}\) bounded away from zero and infinity. In the implementation, this follows from clipping the sensor weights.

\begin{theorem}
\label{thm:joint_local_stability}
Let
\(\mathcal A=(\mathcal R,\mathcal O_S)\). Suppose that \(\mathcal A\) is
continuously Fr\'echet differentiable near \(Z^\dagger\), its derivative is
locally Lipschitz, and Assumptions~\ref{ass:constrained_state_stability} and
\ref{ass:parameter_observation_sensitivity} hold. Then there are a
neighborhood \(\mathcal N_0\) of \(Z^\dagger\) and a constant
\(C_{\rm loc}\) such that
\begin{equation}
    d_{\mathcal X}(Z,Z^\dagger)
    \leq C_{\rm loc}
    \left(
        \|\mathcal R(Z)\|_{\mathcal Y}
        +\|\mathcal O_S(Z)-\mathcal O_S(Z^\dagger)
        \|_{\mathcal Y_{\rm obs}}
    \right). \label{eq:general_local_nonlinear_stability}
\end{equation}
for every \(Z\in\mathcal N_0\). Moreover,
if \eqref{eq:general_noisy_observation_model} holds and a
joint prediction
\(\widehat Z:=(\widehat U,\widehat\lambda)\) belongs to
\(\mathcal N_0\), then there is a constant $C > 0$ such that
\begin{equation}
    d_{\mathcal X}(\widehat Z, Z^\dagger)
    \leq
    C\left(\mathcal J_\delta(\widehat Z)^{1/2}+\delta\right).
    \label{eq:general_lipschitz_loss_to_error}
\end{equation}
\end{theorem}

The theorem is local: it gives an error certificate once the candidate lies
on the regular branch near \(Z^\dagger\), but it does not assert that a
nonconvex training procedure will enter that neighborhood. The proof of Theorem \ref{thm:joint_local_stability} is given
in Appendix~\ref{app:proof_nonlinear_loss}.

\begin{remark}
The estimate does not cover distant solution branches, bifurcations,
symmetry-related solutions, or parameter directions that are globally weak.
If only a conditional estimate with modulus of continuity \(\omega\) is
available, the same argument gives
\[
    d_{\mathcal X}(\widehat Z,Z^\dagger)
    \leq C_{\rm stab}\,
    \omega\!\left(
        C_\gamma\mathcal J_\delta(\widehat Z)^{1/2}+\delta
    \right).
\]
This includes H\"older and logarithmic stability estimates.
\end{remark}

While our theoretical analysis considers continuous loss terms in \eqref{eq:general_continuous_physics_loss}, the empirical objective must be implemented using sampled residuals and discrete derivatives in practice. So, there is a mild gap between the continuous theory and the practical training loss introduced in section~\ref{sec:method}. However, if these discrete quantities effectively control their continuous counterparts as shown in Appendix~\ref{app:empirical_consistency}, then the theoretical bound remains valid and every $\widehat Z \in \mathcal{N}_0$ satisfies

\begin{equation}
    d_{\mathcal X}(\widehat Z,Z^\dagger)
    \leq C\left(
        \widehat{\mathcal J}_\delta(\widehat Z)^{1/2}
        +\varepsilon_{\rm disc}+\delta
    \right),
    \label{eq:empirical_multiphysics_lipschitz_bound}
\end{equation}
where \(\varepsilon_{\rm disc}\) collects derivative, quadrature,
interpolation, and constraint-enforcement errors. Approximation and
optimization errors are reflected in the attained empirical loss. The bound
is independent of the particular network architecture: the network supplies
the candidate \(\widehat Z\), while the PDE and the observation configuration
determine whether a small loss controls the reconstruction error
\cite{mishra2022estimates}.

\subsection{Instantiation for the flow and MHD systems}
\label{subsec:theory_instantiations}
We now illustrate how the abstract operators are instantiated in specific physical applications. We focus on three benchmark problems: lid-driven cavity flow, flow past a circular cylinder, and non-ideal magnetohydrodynamics (MHD). For each problem, we explicitly define the corresponding differential operators, unknown parameters, and sparse observed measurements in the inverse problems.

\subsubsection{Incompressible Navier--Stokes systems.}
The flow systems considered here satisfy the incompressible
Navier--Stokes equations
\begin{equation}
\begin{cases}
    \partial_t\mathbf u
    +\lambda_1(\mathbf u\cdot\nabla)\mathbf u
    +\nabla p-\nu\Delta\mathbf u=0,\\
    \nabla\cdot\mathbf u=0.
\end{cases}
\label{eq:ns_governing_section23}
\end{equation}
Here \(\lambda_1\) is the convection coefficient and \(\nu\) is the
kinematic viscosity used in \eqref{eq:ns_governing_section23}. This work studies two settings for this physical system. The first one is the steady lid-driven cavity problem, corresponding to
\(\partial_t\mathbf u=0\) and \(\lambda_1=1\). The second one is the cylinder wake problem. 

For the lid-driven cavity problem, the abstract fields in \eqref{eq:intro_coupled_system} are
assigned as
\begin{equation}
    u_1=p,
    \qquad
    u_2=\mathbf u,
    \qquad
    \lambda=\nu,
    \label{eq:cavity_two_field_grouping}
\end{equation}
where \(u_1 = p\) is the only sparse measured pressure field and \(u_2\) is the unobserved
velocity field. Using the standard physical notation, the residual is
\begin{equation}
    \mathcal R_{\rm NS}(\mathbf u,p,\nu)
    =
    \left(
        (\mathbf u\cdot\nabla)\mathbf u+\nabla p-\nu\Delta\mathbf u,
        \nabla\cdot\mathbf u,
        \mathbf u|_{\partial\Omega}-\mathbf g,
        \int_\Omega p\,dx
    \right).
    \label{eq:ns_instance_general_residual}
\end{equation}
Here \(\mathbf g\) is the prescribed boundary velocity. The final component
fixes the additive pressure gauge. For an admissible
velocity perturbation \(\delta\mathbf u\), viscosity gives
\(\nu^\dagger\|\nabla\delta\mathbf u\|_{L^2}^2\), while
\begin{equation*}
    ((\mathbf u^\dagger\cdot\nabla)\delta\mathbf u,
       \delta\mathbf u)_{L^2}=0,
    \qquad
    |((\delta\mathbf u\cdot\nabla)\mathbf u^\dagger,
       \delta\mathbf u)_{L^2}|
    \leq
    \|\nabla\mathbf u^\dagger\|_{L^\infty}
    \|\delta\mathbf u\|_{L^2}^2.
    \label{eq:ns_representative_energy_bounds}
\end{equation*}
These are the concrete counterparts of
\eqref{eq:dissipative_coercivity}--
\eqref{eq:transport_coupling_relative_bound}.
Assumption~\ref{ass:constrained_state_stability} additionally
requires the mixed velocity--pressure stability and a trivial constrained
kernel on the selected branch.

For the cylinder wake, the measured quantity is the velocity field, so the
assignment is instead
\begin{equation*}
    u_1=\mathbf u,
    \qquad
    u_2=p,
    \qquad
    \lambda=(\lambda_1,\nu).
    \label{eq:cylinder_two_field_grouping}
\end{equation*}
The corresponding space--time estimate uses the unsteady linearized
Navier--Stokes operator. When the initial condition is included in
\(\mathcal C\), it removes an independent perturbation at \(t=0\); if it is
not prescribed, it must be treated as part of the unknown state and the
invertibility assumption must be checked for that enlarged problem.

\subsubsection{Non-ideal magnetohydrodynamics.}
The non-ideal MHD system considered here is governed by
\begin{equation*}
\left\{
\begin{aligned}
    &\partial_t\rho+\nabla\cdot(\rho\mathbf v)=0,\\
    &\partial_t(\rho\mathbf v)
    +\nabla\cdot\Bigl[
        \rho\mathbf v\otimes\mathbf v
        +\Bigl(p+\tfrac12|\mathbf B|^2\Bigr)\mathbf I
        -\mathbf B\otimes\mathbf B-\boldsymbol\tau
    \Bigr]=0,\\
    &\partial_t\mathbf B
    -\nabla\times(\mathbf v\times\mathbf B)
    +\nabla\times(\eta\mathbf J)=0,\\
    &\partial_t E
    +\nabla\cdot\Bigl[
        \Bigl(E+p+\tfrac12|\mathbf B|^2\Bigr)\mathbf v
        -(\mathbf v\cdot\mathbf B)\mathbf B\\
    &\qquad\qquad\quad
        -\boldsymbol\tau\mathbf v-\eta\mathbf J\times\mathbf B
    \Bigr]=0,\\
    &\nabla\cdot\mathbf B=0.
\end{aligned}
\right.
\label{eq:mhd_governing_section23}
\end{equation*}
The current density, total energy density, and viscous stress tensor are
defined by
\begin{equation*}
\begin{aligned}
    \mathbf J
    &=\nabla\times\mathbf B,\\
    E
    &=\frac{p}{\gamma-1}
      +\frac12\rho|\mathbf v|^2
      +\frac12|\mathbf B|^2,\\
    \boldsymbol\tau
    &=\nu\left[
        \nabla\mathbf v+(\nabla\mathbf v)^{T}
        -\frac23(\nabla\cdot\mathbf v)\mathbf I
    \right],
\end{aligned}
\label{eq:mhd_definitions_section23}
\end{equation*}
where \(\gamma>1\) is the adiabatic index and \(\mathbf I\) is the identity
tensor. For the magnetic-field-observed MHD problem, we set
\begin{equation*}
    u_1=\mathbf B,
    \qquad
    u_2=(\rho,\mathbf v,p),
    \qquad
    \lambda=(\nu,\eta).
    \label{eq:mhd_two_field_grouping}
\end{equation*}
Here the vector field \(\mathbf B\) is one physical field in the abstract
notation of Section~\ref{sec:introduction}, where \(\mathbf B\) is only observed on some sensor locations. \(\rho\), \(\mathbf v\), and
\(p\) are density, velocity, and pressure respectively, while \(\nu\) and \(\eta\) are
unknown viscosity and magnetic diffusivity. In the pressure-assisted cases we retain
the same state grouping and augment the observation map by
\begin{equation*}
    \mathcal O_S^{\rm aug}(U)
    :=\bigl(\mathcal O_B(\mathbf B),
            \mathcal O_p(\Pi_pu_2)\bigr),
    \qquad \Pi_pu_2=p.
    \label{eq:mhd_augmented_observation_operator}
\end{equation*}
where \(\mathcal O_B\) and \(\mathcal O_p\) sample the magnetic and pressure
fields at their respective sensor sets.
The proof above only uses boundedness and differentiability of the observation
map, so it applies to this augmented operator without further changes.

The viscous and resistive terms directly control
\(\nabla\delta\mathbf v\) and \(\nabla\delta\mathbf B\). Let
\(\delta U=(\delta\mathbf B,\delta\rho,\delta\mathbf v,\delta p)\) and let
\(\mathcal E^\dagger(\delta U)\) denote the quadratic energy of this
linearized state. On a regular branch with positive density and pressure, we
assume that the energy is equivalent to the chosen state norm,
\begin{equation*}
    c_E\|\delta U\|_{\mathcal H}^2
    \leq\mathcal E^\dagger(\delta U)
    \leq C_E\|\delta U\|_{\mathcal H}^2,
    \label{eq:mhd_energy_equivalence}
\end{equation*}
for constants \(c_E,C_E>0\). The linearized equations then have the
representative estimate
\begin{equation}
\begin{aligned}
    \frac{d}{dt}\mathcal E^\dagger(\delta U)
    &+c_\nu\nu^\dagger
        \|\nabla\delta\mathbf v\|_{L^2}^2
    +c_\eta\eta^\dagger
        \|\nabla\delta\mathbf B\|_{L^2}^2\\
    &\leq
    C^\dagger\mathcal E^\dagger(\delta U)
    +C_r\|r\|_{\mathcal Y}^2,
\end{aligned}
\label{eq:mhd_linearized_energy_estimate}
\end{equation}
where \(r\) is the linearized equation and constraint residual, and the
positive constants depend on bounds for the reference solution. Together
with the initial, periodic, and magnetic-divergence constraints,
Gr\"onwall's inequality gives an a priori state estimate over every time
interval on which the reference solution remains regular. The coupling terms
are written out in Appendix~\ref{app:mhd_coupling_details}. They explain how
magnetic observations can constrain the hydrodynamic fields, but they do not
by themselves verify
Assumption~\ref{ass:parameter_observation_sensitivity} for a particular
sensor layout.

\section{Method}
\label{sec:method}

In this section, we present the framework to reconstruct the full states of the coupled physical system and identify unknown parameters from sparse single-field observations. To bridge the gap between sparse point measurements and full-field reconstruction, the observation set $\mathsf S$ is first embedded into a latent space. A cross-attention mechanism then transfers the sparse sensor information onto a target reconstruction grid. The resulting representation is concatenated with the grid features and fed into a Fourier Neural Operator (FNO), which serves as a decoder to predict all physical channels. 
Sparse measurements and grid-based neural operators use different data
structures. We connect them in three steps: sensor-token construction,
cross-attention from the sensors to the grid, and FNO decoding followed by
physics-constrained optimization.

Specifically, each measurement in \(\mathsf S\) is first written as a sensor token
\(s_i\), containing its location \(x^{(i)}\), measured value
\(y^{(i)}:=u_1^{(i)}\), and, when needed, a channel-availability mask
\(m^{(i)}\). A cross-attention encoder transfers these irregular sensor
tokens to the computational grid. The resulting dense latent field is passed
to an FNO decoder, which returns all components of \(U\) in \eqref{eq:intro_coupled_system}. The network and the
unknown parameter vector \(\lambda\) are then fitted together using the
observation, equation, and constraint residuals. An overview of the framework is illustrated in Figure~\ref{fig:overview}.

\begin{figure}[!htb]
    \centering
    \includegraphics[width=1.0\textwidth]
{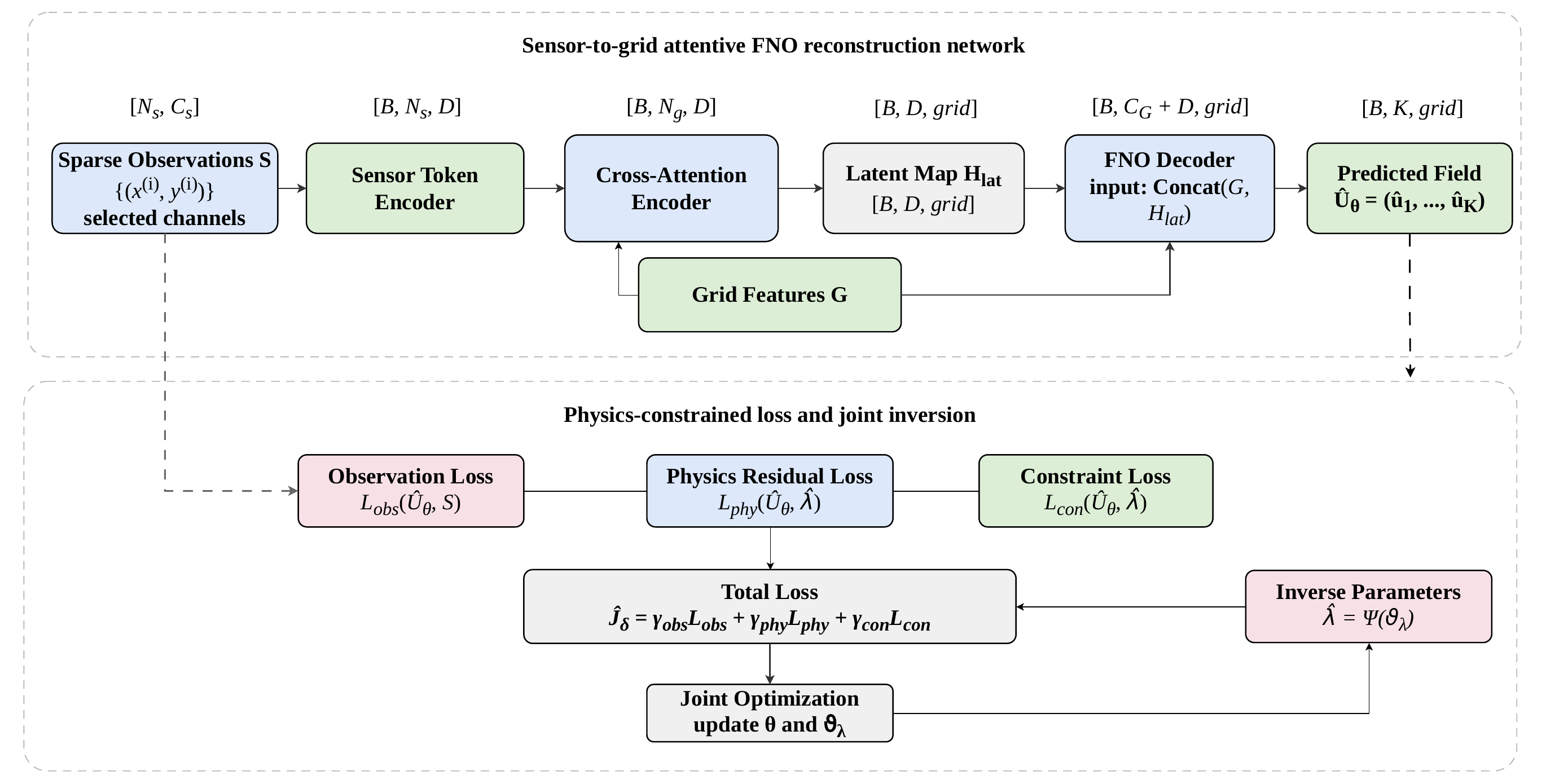}
    \caption{Overall architecture of the proposed physics-constrained attentive FNO framework.}
    \label{fig:overview}
\end{figure}

\subsection{Sensor and Grid Token Embedding}

Sensors can be arbitrarily located without requiring grid alignment. To process these unstructured point data, the observations are first mapped into a latent space. Let \(d_x\)
be the coordinate dimension and let \(d_y\) be the number of scalar
components measured at each sensor. Letting
\[
    x^{(i)}\in\mathbb R^{d_x},\qquad
    y^{(i)}:=u_1^{(i)}\in\mathbb R^{d_y},\qquad
    m^{(i)}\in\{0,1\}^{d_y},
\]
\(m^{(i)}\) records the available components. For a scalar observation
without missing channels, \(d_y=1\) and \(m^{(i)}=1\). The symbol
\(\odot\) denotes componentwise multiplication. A sensor token and a grid
token are defined by
\begin{equation}
\begin{aligned}
    s_i
    &:=
    \bigl[
        \psi_x(x^{(i)}),
        \;m^{(i)}\odot y^{(i)},
        \;m^{(i)}
    \bigr],
    \qquad i=1,\ldots,N_s,\\
    \chi_j^g&:=\psi_g(x_j^g),
    \qquad j=1,\ldots,N_g,\\
    H_s&:=\phi_s(\{s_i\}_{i=1}^{N_s})
    \in\mathbb R^{B\times N_s\times D},\\
    H_g&:=\phi_g(\{\chi_j^g\}_{j=1}^{N_g})
    \in\mathbb R^{B\times N_g\times D}.
\end{aligned}
\label{eq:method_token_embedding}
\end{equation}
Here \(x_j^g\) is the \(j\)-th coordinate in the grid set
\(\mathsf G:=\{x_j^g\}_{j=1}^{N_g}\), \(N_g\) is the number of grid
points, \(D\) is the latent width, and \(B\) is the batch size.
The maps \(\psi_x\) and \(\psi_g\) encode coordinates, while \(\phi_s\)
and \(\phi_g\) are trainable feature maps.

\subsection{Cross-Attention Encoder}

The encoder aims to turn sparse observations into a dense latent field, we here use multi-head cross-attention \cite{vaswani2017attention}. The grid tokens $H_g$ are used as as queries, and the sensor tokens $H_s$ are used as keys and values respectively:
\begin{equation}
    Q=H_gW_Q,\qquad K=H_sW_K,\qquad V=H_sW_V,
    \label{eq:method_qkv}
\end{equation}
where \(W_Q,W_K,W_V\) are trainable projection matrices. Their output
features are split into \(N_h\) heads of width \(d_a\).
The sensor-conditioned grid representation is obtained through
multi-head cross-attention. For the \(h\)th attention head,
\begin{equation}
    A^{(h)}
    =
    \operatorname{softmax}
    \left(
        \frac{Q^{(h)}(K^{(h)})^\top}{\sqrt{d_a}}
        +\beta\log\mathbf w
        +\mathcal M_s
    \right)V^{(h)},
    \label{eq:method_cross_attention}
\end{equation}
where \(d_a\) is the feature dimension of each head and
\(\mathbf w=(w_1,\ldots,w_{N_s})\) contains the strictly positive
residual sensor weights. These weights are initialized to one, and they
periodically updated from the normalized local squared PDE residuals
interpolated at the sensor locations. The coefficient
\(\beta\) controls the influence of the residual prior, with
\(\beta=0.5\) in our experiments. The mask \(\mathcal M_s\) is zero for
active sensors and negative infinity for inactive sensors.

More explicitly, let \(\mathcal I_s\) denote the set of active sensors.
For the \(j\)th grid query, the attention coefficient assigned to the
\(i\)th sensor by the \(h\)th head can be written as
\begin{equation}
    \alpha_{ji}^{(h)}
    =
    \frac{
        w_i^\beta
        \exp\!\left(
            \bigl(q_j^{(h)}\bigr)^\top k_i^{(h)}/\sqrt{d_a}
        \right)
    }{
        \displaystyle
        \sum_{r\in\mathcal I_s}
        w_r^\beta
        \exp\!\left(
            \bigl(q_j^{(h)}\bigr)^\top k_r^{(h)}/\sqrt{d_a}
        \right)
    },
    \qquad i\in\mathcal I_s,
    \label{eq:method_attention_coefficient}
\end{equation}
Here \(q_j^{(h)}\), \(k_i^{(h)}\), and \(v_i^{(h)}\) are the query, key,
and value vectors in head \(h\), respectively.
For inactive sensors, \(\alpha_{ji}^{(h)}=0\). Thus,
\(w_i^\beta\) acts as a multiplicative sensor-importance prior, assigning
a larger contribution to sensors with larger residual weights.
The output of each head and the final multi-head output are
\begin{equation}
\begin{aligned}
    a_j^{(h)}
    &=
    \sum_{i\in\mathcal I_s}
    \alpha_{ji}^{(h)}v_i^{(h)},\\
    A_j
    &=
    W_O\operatorname{Concat}
    \left(
        a_j^{(1)},\ldots,a_j^{(N_h)}
    \right),
\end{aligned}
\label{eq:method_multihead_output}
\end{equation}
where \(N_h\) is the number of attention heads and \(W_O\) is a
trainable output projection. When sensor weighting is disabled ($w_i = 1$), the formulation reduces to standard cross-attention. Equation \eqref{eq:method_cross_attention} corresponds to the log-bias mode. In the feature and value-gate modes, the feature and value-gate modes either incorporate sensor weights as additional input features or use them to scale the sensor tokens.

The attention output is further refined through two residual updates:

\begin{equation}
\begin{aligned}
    \widetilde H_g
    &= H_g + A,\\
    H_g^{\mathrm{out}}
    &=
    \widetilde H_g+
    \operatorname{FFN}
    \left(\operatorname{LN}(\widetilde H_g)\right),\\
    H_{\rm lat}
    &=
    \operatorname{Reshape}\left(H_g^{\mathrm{out}}\right)
    \in\mathbb R^{B\times D\times N_1\times N_2}.
\end{aligned}
\label{eq:method_latent_map}
\end{equation}
Here, the first residual connection combines the original grid
representation \(H_g\) with the sensor information \(A\) obtained from
cross-attention, thereby preserving the grid features while
incorporating the observations. The layer-normalization operator
\(\operatorname{LN}\) normalizes the feature channels of each grid token
to improve numerical stability. The position-wise feed-forward network
\(\operatorname{FFN}\), implemented as a two-layer multilayer perceptron
with a nonlinear activation, independently transforms each grid token
and mixes its latent feature channels. Its parameters are shared across
all grid locations. The second residual connection preserves the fused
representation and facilitates stable optimization.

Finally, the token sequence \(H_g^{\mathrm{out}}\in
\mathbb R^{B\times N_g\times D}\), where \(N_g=N_1N_2\) and
\(N_1,N_2\) are the two grid dimensions, is reshaped into the latent
field \(H_{\rm lat}\).

For three-dimensional or time-dependent problems, \(H_{\rm lat}\) is reshaped
into the corresponding higher-dimensional latent field. When
bidirectional refinement is enabled, the attention blocks follow
\[
    \mathsf S\rightarrow\mathsf G\rightarrow
    \mathsf S\rightarrow\mathsf G,
\]
allowing the sensor tokens to incorporate global grid context before
producing the final latent field.

\subsection{FNO Decoder and Physical Parameterization}

The latent field is concatenated with the grid features:
\begin{equation*}
    H_{\rm dec}
    :=
    \operatorname{Concat}(G,H_{\rm lat}),
    \label{eq:method_decoder_input}
\end{equation*}
where \(G\) contains the grid coordinates and any fixed positional
features. The FNO decoder is
\begin{equation*}\label{eq_fno_decoder}
    \widehat U_\theta
    =
    D_\theta^{\rm FNO}(H_{\rm dec}).
\end{equation*}

The decoder returns every component of \(U\), including the fields that are not measured. For example, in the MHD problem it reconstructs the hydrodynamic variables from magnetic observations. Its \(\ell\)-th spectral block is
\cite{li2020fourier}
\begin{equation*}
    H^{(\ell+1)}
    =
    \sigma\!\left(
        \mathcal K_\ell H^{(\ell)}
        +\mathcal P_\ell H^{(\ell)}
    \right),
    \qquad
    \widehat{\mathcal K_\ell H^{(\ell)}}(k)
    =
    R_\ell(k)\widehat H^{(\ell)}(k),
    \quad k\in\mathcal M_{\rm trunc}.
    \label{eq:method_spectral_block}
\end{equation*}
Here \(H^{(\ell)}\) is the latent feature field, \(\sigma\) is the
activation function, \(\mathcal P_\ell\) is a pointwise linear map, and
\(\mathcal K_\ell\) is a spectral convolution. The hat denotes the discrete
Fourier transform, \(R_\ell(k)\) is the learned multiplier for mode \(k\),
and \(\mathcal M_{\rm trunc}\) is the retained set of Fourier modes.

\subsection{Physics-Constrained Loss Function}

The sensors used in the numerical examples measure state variables directly,
so \(\mathcal O_S\) has no direct dependence on \(\lambda\). We therefore
write its state restriction as
\(\mathcal O_S^U(U):=\mathcal O_S(U,\lambda)\). The predicted measurements
are
\begin{equation*}
    \widehat y
    =
    \mathcal O_S^U(\widehat U_\theta)
    \in\mathcal Y_{\rm obs}.
    \label{eq:method_sensor_prediction}
\end{equation*}

This gives the weighted observation loss
\begin{equation*}
    \mathcal L_{\mathrm{obs}}
    =
    \frac{
        \displaystyle
        \sum_{i\in\mathcal I_s}
        w_i\left\|\widehat y^{(i)}-y^{(i)}\right\|_2^2
    }{
        \displaystyle
        \sum_{i\in\mathcal I_s}w_i
    },
    \label{eq:observation_loss}
\end{equation*}
where \(\mathcal I_s\subseteq\{1,\ldots,N_s\}\) is the set of active
sensors, \(\widehat y^{(i)}\) is the prediction at sensor \(i\), and
\(w_i>0\) is its residual-aware weight. The denominator is nonzero
whenever at least one sensor is active.

The unknown parameter vector is represented by an unconstrained trainable
variable \(\vartheta_\lambda\in\mathbb R^m\) and an admissibility map
\begin{equation*}
    \widehat\lambda
    =
    \Psi(\vartheta_\lambda)\in\Lambda.
    \label{eq:positive_parameterization}
\end{equation*}
If the unknown parameters are positive, one can apply $\Psi(\vartheta_\lambda)=\exp(\vartheta_\lambda)$ component-wise to ensure positivity. A bounded admissible interval can instead be enforced using a scaled sigmoid function.

The physics residual and the boundary or initial-condition residual are
evaluated on the structured grid:
\begin{equation*}
\begin{aligned}
    \mathcal R_{\mathrm{phy}}(x)
    &:=
    \mathcal F
    \left(
        x,\widehat U_\theta,
        \nabla\widehat U_\theta,
        \nabla^2\widehat U_\theta;
        \widehat\lambda
    \right),\\
    \mathcal R_{\mathrm{con}}(x)
    &:=
    \mathcal B(x,\widehat U_\theta;\widehat\lambda)-g(x),
    \qquad x\in\mathcal D.
\end{aligned}
\label{eq:physical_residuals}
\end{equation*}
Here \(\mathcal F\) stacks the governing-equation residuals,
\(\mathcal B\) is the relevant boundary or initial operator, and \(g\)
contains the prescribed data. These are components of the general residual
\(\mathcal R\) in \eqref{eq:general_all_at_once_residual}. Spatial
derivatives, and temporal derivatives when present, are evaluated using
second-order central finite differences. Let
\(\mathsf G_{\mathrm{int}}\) and \(\mathsf G_{\mathrm{con}}\) denote the
interior and constraint grid-point sets. The corresponding losses are
\begin{equation*}
\begin{aligned}
    \mathcal L_{\mathrm{phy}}
    &:=
    \frac{1}{|\mathsf G_{\mathrm{int}}|}
    \sum_{x_j^g\in\mathsf G_{\mathrm{int}}}
    \left\|\mathcal R_{\mathrm{phy}}(x_j^g)\right\|_2^2,\\
    \mathcal L_{\mathrm{con}}
    &:=
    \frac{1}{|\mathsf G_{\mathrm{con}}|}
    \sum_{x_j^g\in\mathsf G_{\mathrm{con}}}
    \left\|\mathcal R_{\mathrm{con}}(x_j^g)\right\|_2^2.
\end{aligned}
\label{eq:physical_constraint_losses}
\end{equation*}
The sampled objective in
\eqref{eq:empirical_multiphysics_lipschitz_bound} is implemented as
\begin{equation*}
    \widehat{\mathcal J}_\delta(\theta,\vartheta_\lambda)
    :=
    \gamma_{\mathrm{obs}}\mathcal L_{\mathrm{obs}}
    +
    \gamma_{\mathrm{phy}}\mathcal L_{\mathrm{phy}}
    +
    \gamma_{\mathrm{con}}\mathcal L_{\mathrm{con}}.
    \label{eq:total_training_loss}
\end{equation*}
The positive weights control the three contributions. The network
parameters \(\theta\) and inverse variables \(\vartheta_\lambda\) are optimized
jointly. We write \(E_\theta\) for the complete sensor-to-grid encoder
defined by \eqref{eq:method_token_embedding}--\eqref{eq:method_latent_map}. The training procedure of our method is summarized in Algorithm~\ref{alg:method}.

\begin{algorithm}[t]
\caption{Training procedure of the proposed physics-constrained attentive FNO}
\label{alg:method}
\small
\begin{algorithmic}[1]

\Require Sensor set
\(\mathsf S=\{(x^{(i)},y^{(i)})\}_{i=1}^{N_s}\),
grid features \(G\), residual operators \(\mathcal F\) and \(\mathcal B\),
observation operator \(\mathcal O_S^U\), loss weights
\(\gamma_{\rm obs},\gamma_{\rm phy},\gamma_{\rm con}\), and iteration count
\(N_{\rm it}\)
\State Initialize \(\theta\), \(\vartheta_\lambda\), and
\(w=(w_1,\ldots,w_{N_s})\)

\For{\(n=1,\ldots,N_{\rm it}\)}

    \State \(\widehat\lambda\gets\Psi(\vartheta_\lambda)\)

    \State
    \(H_{\rm lat}\gets E_\theta(\mathsf S,G;w)\)

    \State
    \(\widehat U_\theta\gets
      D_\theta^{\rm FNO}
      \bigl(\operatorname{Concat}(G,H_{\rm lat})\bigr)\)

    \State
    \(\widehat y\gets\mathcal O_S^U(\widehat U_\theta)\)

    \State Evaluate
    \(\mathcal L_{\rm obs}\), \(\mathcal L_{\rm phy}\), and
    \(\mathcal L_{\rm con}\) using
    \((\widehat U_\theta,\widehat\lambda)\)

    \State
    \(\widehat{\mathcal J}_\delta\gets
      \gamma_{\rm obs}\mathcal L_{\rm obs}
      +\gamma_{\rm phy}\mathcal L_{\rm phy}
      +\gamma_{\rm con}\mathcal L_{\rm con}\)

    \State Update \((\theta,\vartheta_\lambda)\) by one gradient step on
    \(\widehat{\mathcal J}_\delta\)

    \If{residual-aware sensor weighting is enabled}
        \State Update, clip, and normalize \(w\)
    \EndIf

\EndFor

\State \Return
\(\widehat U_{\theta^\ast}\) and
\(\widehat\lambda=\Psi(\vartheta_\lambda^\ast)\)
\Ensure Reconstructed state \(\widehat U = \widehat U_{\theta^\ast}\) and identified parameters
\(\widehat\lambda\)
\end{algorithmic}
\end{algorithm}

\section{Numerical Study}
\label{sec:numerical_study}
We evaluate the performance of the proposed framework on four test problems, including two- and three-dimensional lid-driven cavity test problems~\cite{ghia1982high,bruneau2006two,albensoeder2005accurate}, a two-dimensional cylinder wake test problem~\cite{schafer1996benchmark}, and a two-dimensional non-ideal MHD Orszag--Tang vortex test problem~\cite{orszag1979small,toth2000constraint}. To ensure comprehensive evaluation, we cover a spectrum of scenarios, including steady and unsteady systems, 2D and 3D spatial domains, boundary- and interior-dominated observations, as well as the reconstruction of various physical fields. For clarity, we summarize the settings of all numerical test problems in Table~\ref{tab:example_summary}. The code for this study will be released upon paper acceptance.

\begin{table}[H]
\centering
\caption{Inverse-problem settings for the four benchmarks.}
\label{tab:example_summary}
\small
\setlength{\tabcolsep}{5pt}
\renewcommand{\arraystretch}{1.12}

\begin{tabularx}{\linewidth}{
  @{}
  >{\raggedright\arraybackslash}p{0.19\linewidth}
  >{\raggedright\arraybackslash}p{0.40\linewidth}
  >{\raggedright\arraybackslash}X
  @{}
}
\toprule
\textbf{Benchmark}
& \textbf{Sparse observations}
& \textbf{Inferred fields and parameters} \\
\midrule

2D lid-driven cavity
& Boundary pressure $p|_{\partial\Omega}$.
  $N_s\in\{100,200,300\}$ at $Re=3200$;
  $N_s\in\{200,300,400\}$ at $Re=5000$.
& Full field $(u,v,p)$ and $Re$. \\
\midrule

3D lid-driven cavity
& Boundary pressure $p|_{\partial\Omega}$;
  $N_s\in\{500,1000,1500\}$.
& Full field $(u,v,w,p)$ and $Re$. \\
\midrule

2D cylinder wake
& Boundary velocity
  $(u,v)|_{\partial\Omega\times(0,T]}$;
  $N_s\in\{13{,}000,14{,}000,15{,}000\}$.
  No pressure observations.
& Time-dependent field $(u,v,p)(x,y,t)$,
  $Re$, and $\lambda_1$. \\
\midrule

2D non-ideal MHD
& Magnetic field $(B_x,B_y)$;
  $N_B\in\{5{,}000,10{,}000,15{,}000\}$.
  Optional pressure observations: none, boundary, or interior;
  $N_p=10{,}000$ when used.
& Space--time fields
  $(\rho,v_x,v_y,B_x,B_y,p)$,
  current density $j_z$, and parameters $\nu,\eta$. \\
\bottomrule
\end{tabularx}

\vspace{4pt}
\begin{minipage}{\linewidth}
\footnotesize
\textit{Optimization.}
The two cavity problems use joint AdamW~\cite{loshchilov2019decoupled} optimization of the network
and $\nu$ for 10{,}000 and 5{,}000 epochs, respectively.
The cylinder wake additionally optimizes $\lambda_1$ for 10{,}000 epochs.
The MHD problem alternates field, parameter, and joint Muon updates
for 7{,}000 epochs.
\end{minipage}
\end{table}

For the noise tests, independent Gaussian perturbations are added to each observed channel with standard deviation $\gamma\,\operatorname{std}(y^{\mathrm{obs}})$, where $\gamma\in\{0,0.01,0.05\}$. To access the reconstruction accuracy, we use the following relative errors defined as:
\begin{equation*}
    e_q=
    \frac{\|q_\theta-q_{\mathrm{ref}}\|_2}
    {\|q_{\mathrm{ref}}\|_2},
    \qquad
    e_\lambda=
    \frac{\|\lambda_\theta-\lambda_{\mathrm{ref}}\|}
    {\|\lambda_{\mathrm{ref}}\|}.
\end{equation*}
For incompressible flows, pressure errors are evaluated after removing the spatial mean.

\subsection{2D Lid-Driven Cavity Flow}

We first consider the inverse reconstruction of the two-dimensional lid-driven
cavity flow governed by the steady incompressible Navier--Stokes equations:
\begin{equation*}\label{eq:ldc_param}
\left\{
\begin{aligned}
(\mathbf{u}\cdot\nabla)\mathbf{u}+\nabla p
&=\nu\Delta \mathbf{u}, && \text{in } \Omega,\\
\nabla\cdot \mathbf{u} &=0, && \text{in } \Omega,\\
\mathbf{u} &=\mathbf{g}, && \text{on } \partial\Omega,
\end{aligned}
\right.
\end{equation*}
where $\mathbf{u}=(u,v)^{\top}$ denotes the velocity field,
$p$ denotes the scalar pressure, and $\nu=1/\mathrm{Re}$ is the kinematic
viscosity. The spatial domain is $\Omega=(0,1)^2$ and the boundary condition is given by
\begin{equation*}
\mathbf{g}(x_1,x_2)=
\begin{cases}
(1,0)^{\top}, & x_2=1,\\
(0,0)^{\top}, & \text{otherwise}.
\end{cases}
\end{equation*}
In this case, the sparse pressure observations are only given on the boundary, while the velocity field and the Reynolds number are entirely unknown. 

Our goal is to reconstruct $(u,v,p)$ and identify $\nu$ from those sparse pressure observations. The pressure observations are obtained by randomly selecting sensors on the boundary, where the reference solutions are generated with OpenFOAM~\cite{weller1998tensorial} on a uniform $129\times129$ finite-volume mesh.

For the \(\mathrm{Re}=5000\) case, a direct numerical check of the local
parameter-to-pressure sensitivity in
Assumption~\ref{ass:parameter_observation_sensitivity} is given in
Appendix~\ref{app:a2_diagnostic}. This check uses separately converged
OpenFOAM forward solutions and is independent of the neural-network training;
the resulting singular-value estimates are reported in
Table~\ref{tab:a2_ldc_re5000_diagnostic}.

\subsubsection{Inverse problem for 2D Lid-Driven Cavity Flow at
\texorpdfstring{$\mathrm{Re}=3200$}{Re=3200}}
For the $\mathrm{Re}=3200$ case, the encoder consists of a two-layer sensor MLP, four-head attention modules, and a bidirectional sensor--grid information path. The MLP has a hidden width of 64 and an output dimension of 128. The FNO decoder consists of four residual FNO blocks with a channel width of 64. Each block combines a spectral convolution retaining 24 Fourier modes in each spatial direction with a parallel $1\times1$ convolution, followed by group normalization and GELU activation. Replicate padding of size 8 is applied to the feature map before the FNO blocks. The predicted pressures at the sensor locations are obtained from the reconstructed grid field while the PDE residuals are evaluated at the interior points of the structured grid using second-order central finite differences. The initial learning rates are $10^{-3}$ and  $5\times10^{-3}$ for the network parameters and the trainable log-viscosity parameter, respectively. During the initial 500 warm-up epochs, all sensor weights are fixed at unity, while the network parameters and the log-viscosity parameter are updated according to the standard training procedure. Following this warm-up phase, the sensor weights are updated based on the instantaneous PDE residuals, with subsequent updates performed at every $500$ epochs. The updated weights are used in both the observation loss and the attention calculation. For the high-Reynolds-number cases, the trainable viscosity parameter is initialized according to a smaller initial Reynolds number: 2000 for the target $\mathrm{Re}=3200$ and 3000 for the target $\mathrm{Re}=5000$.

\begin{table}[htb]
\centering
\caption{Reconstruction and parameter inversion errors for the 2D lid-driven cavity inverse problem at $\mathrm{Re}=3200$ under different numbers of pressure observations and noise levels.}
\label{tab:ldc2d_re3200_sensor_noise_results}
\resizebox{\textwidth}{!}{
\begin{tabular}{ccccccc}
\toprule
\# Sensors & Noise level & $e_{|\mathbf{u}|}$ & $e_u$ & $e_v$ & $e_p$ & $e_{\mathrm{Re}}$ \\
\midrule
100 & 0.00 & $7.492\times 10^{-2}$ & $6.932\times 10^{-2}$ & $8.427\times 10^{-2}$ & $2.409\times 10^{-1}$ & $7.967\times 10^{-2}$ \\
200 & 0.00 & $4.854\times 10^{-2}$ & $4.947\times 10^{-2}$ & $5.316\times 10^{-2}$ & $1.078\times 10^{-1}$ & $6.269\times 10^{-2}$ \\
300 & 0.00 & $3.846\times 10^{-2}$ & $3.563\times 10^{-2}$ & $4.524\times 10^{-2}$ & $7.701\times 10^{-2}$ & $9.220\times 10^{-2}$ \\
\midrule
100 & 0.01 & $4.638\times 10^{-2}$ & $4.801\times 10^{-2}$ & $5.066\times 10^{-2}$ & $1.027\times 10^{-1}$ & $7.105\times 10^{-2}$ \\
200 & 0.01 & $4.941\times 10^{-2}$ & $5.093\times 10^{-2}$ & $5.364\times 10^{-2}$ & $1.051\times 10^{-1}$ & $8.565\times 10^{-2}$ \\
300 & 0.01 & $4.064\times 10^{-2}$ & $3.746\times 10^{-2}$ & $4.779\times 10^{-2}$ & $8.460\times 10^{-2}$ & $1.105\times 10^{-1}$ \\
\midrule
100 & 0.05 & $1.083\times 10^{-1}$ & $1.031\times 10^{-1}$ & $1.191\times 10^{-1}$ & $2.415\times 10^{-1}$ & $2.455\times 10^{-1}$ \\
200 & 0.05 & $4.489\times 10^{-2}$ & $4.284\times 10^{-2}$ & $5.138\times 10^{-2}$ & $8.831\times 10^{-2}$ & $1.303\times 10^{-1}$ \\
300 & 0.05 & $4.030\times 10^{-2}$ & $3.929\times 10^{-2}$ & $4.576\times 10^{-2}$ & $8.681\times 10^{-2}$ & $9.524\times 10^{-2}$ \\
\bottomrule
\end{tabular}
}
\end{table}

We further compare the proposed method with a vanilla PINN at $\mathrm{Re}=3200$ using 300 boundary pressure observations without noise. As shown in \autoref{tab:ldc2d_pinn_comparison}, the proposed method yields lower errors for the reconstructed velocity and pressure fields, as well as the Reynolds number. The improvement is particularly clear in the velocity reconstruction. These results indicate that the proposed method makes more effective use of sparse pressure observations in this case.

\begin{table}[htb]
\centering
\caption{Comparison with vanilla PINN for the 2D lid-driven cavity
inverse problem using sparse boundary pressure observations without
noise. The numbers of observations are 300 and 400 for
$\mathrm{Re}=3200$ and $5000$, respectively.}
\label{tab:ldc2d_pinn_comparison}
\small
\setlength{\tabcolsep}{5pt}
\renewcommand{\arraystretch}{1.12}
\resizebox{\linewidth}{!}{%
\begin{tabular}{llccccc}
\toprule
$\mathrm{Re}$ & Method
& $e_{|\mathbf{u}|}$ & $e_u$ & $e_v$
& $e_p$ & $e_{\mathrm{Re}}$ \\
\midrule
\multirow{2}{*}{$3200$}
& Vanilla PINN
& $7.712\times10^{-1}$
& $8.271\times10^{-1}$
& $9.387\times10^{-1}$
& $7.537\times10^{-1}$
& $2.155\times10^{-1}$ \\
& Ours
& $\mathbf{3.846\times10^{-2}}$
& $\mathbf{3.563\times10^{-2}}$
& $\mathbf{4.524\times10^{-2}}$
& $\mathbf{7.701\times10^{-2}}$
& $\mathbf{9.220\times10^{-2}}$ \\
\midrule
\multirow{2}{*}{$5000$}
& Vanilla PINN
& $7.451\times10^{-1}$
& $7.778\times10^{-1}$
& $9.309\times10^{-1}$
& $7.161\times10^{-1}$
& $4.290\times10^{-1}$ \\
& Ours
& $\mathbf{1.030\times10^{-1}}$
& $\mathbf{9.743\times10^{-2}}$
& $\mathbf{1.197\times10^{-1}}$
& $\mathbf{2.959\times10^{-1}}$
& $\mathbf{5.999\times10^{-2}}$ \\
\bottomrule
\end{tabular}%
}
\vspace{2pt}

\begin{minipage}{\linewidth}
\footnotesize

\end{minipage}
\end{table}

\begin{table}[htb]
\centering
\caption{Ablation study for the 2D lid-driven cavity inverse problem at $\mathrm{Re}=3200$ with 200 pressure observations and no observation noise.}
\label{tab:ldc2d_ablation_200sensors}
\resizebox{\textwidth}{!}{
\begin{tabular}{ccccccccc}
\toprule
Case & Feedback layers & Sensor weighting & 
$e_{|\mathbf{u}|}$ & $e_u$ & $e_v$ & $e_p$ & $e_{\mathrm{Re}}$ \\
\midrule
Full model 
& 2 
& Residual 
& $4.854\times 10^{-2}$ 
& $4.947\times 10^{-2}$ 
& $5.317\times 10^{-2}$ 
& $1.078\times 10^{-1}$ 
& $6.269\times 10^{-2}$ \\

No sensor weighting
& 2 
& None 
& $5.312\times 10^{-2}$ 
& $5.495\times 10^{-2}$ 
& $5.677\times 10^{-2}$ 
& $1.093\times 10^{-1}$ 
& $1.137\times 10^{-1}$ \\

One feedback layer
& 1 
& Residual 
& $5.217\times 10^{-1}$ 
& $5.678\times 10^{-1}$ 
& $5.744\times 10^{-1}$ 
& $4.536\times 10^{-1}$ 
& $1.055\times 10^{-2}$ \\

One feedback layer, no sensor weighting
& 1 
& None 
& $5.564\times 10^{-1}$ 
& $6.188\times 10^{-1}$ 
& $6.434\times 10^{-1}$ 
& $5.089\times 10^{-1}$ 
& $2.799\times 10^{-2}$ \\
\bottomrule
\end{tabular}
}
\end{table}

\autoref{tab:ldc2d_re3200_sensor_noise_results} shows the reconstruction error of the physical fields and the Reynolds number for the two-dimensional lid-driven cavity inverse problem, where different numbers of sensors and noise levels are considered. It can be seen that, in the noise-free case, increasing the number of sensors from 100 to 300 consistently reduces the velocity and pressure errors. The error for the Reynolds number, however, does not exhibit a strictly monotonic decrease, suggesting a possible saturation effect in parameter identifiability. When a mild noise of 1\% is added, the overall errors remain close to their noise-free scenarios, indicating that the proposed framework is largely insensitive to low-level perturbations. In contrast, at a 5\% noise level, the errors for the 100-sensor setting increase dramatically, especially for $e_{Re}$. Nevertheless, this deterioration is effectively mitigated by deploying 200 or 300 sensors, which bring the errors back to levels comparable to those obtained under much cleaner data. These results demonstrate that the proposed method achieves robust performance under moderate noise, and that increasing the number of sensors serves as an effective countermeasure against data corruption. \autoref{fig:ldc2d_re3200_clean_noise_uvspeedp} shows the point-wise errors of the reconstruction fields for $\mathrm{Re} = 3200$. It is seen that the proposed method is able to reconstruct the field information from sparse observations.

\autoref{tab:ldc2d_ablation_200sensors} presents the ablation study results for the 2D lid-driven cavity test problem for $\mathrm{Re} = 3200$ with 200 pressure observations and no observation noise. From \autoref{tab:ldc2d_ablation_200sensors}, we can see that the full model yields the lowest errors for all metrics except $e_{Re}$ where one‑layer cases give smaller values. The full model also provides the best overall field reconstruction performance. Reducing the feedback path can improve parameter estimation but results in a substantial loss in velocity and pressure accuracy.

\begin{figure}[htb]
    \centering
    \includegraphics[width=1.0\textwidth]{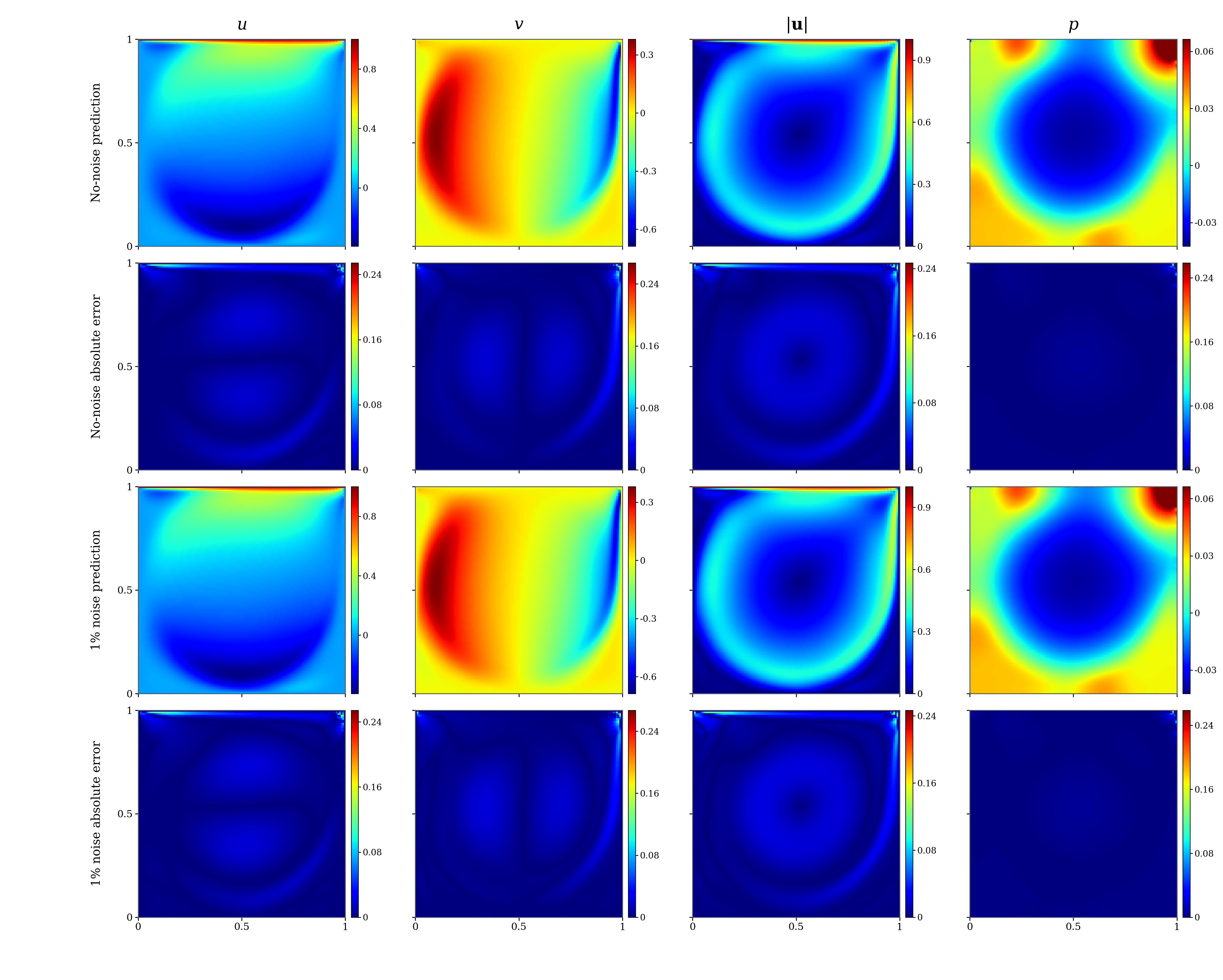}
    \caption{Predicted fields and point-wise errors for the 2D lid-driven cavity inverse problem at $\mathrm{Re}=3200$. The columns show $u$, $v$, $|\mathbf{u}|$, and $p$. The first two rows give the noise-free case. The last two rows give the case with $1\%$ pressure noise.}
    \label{fig:ldc2d_re3200_clean_noise_uvspeedp}
\end{figure}

\subsubsection{Inverse problem for 2D Lid-Driven Cavity Flow at
\texorpdfstring{$\mathrm{Re}=5000$}{Re=5000}}
We next consider $\mathrm{Re} = 5000$ using the same candidate sensors and training configuration. The large $Re$ makes the pressure-only inverse problem more ill-conditioned. The results are reported in \autoref{tab:ldc2d_re5000_sensor_noise_results}. Again, \autoref{tab:ldc2d_re5000_sensor_noise_results} shows the reconstruction error of the entire physical fields and the Reynolds number. The results are similar to that in \autoref{tab:ldc2d_re3200_sensor_noise_results}. The ablation study in \autoref{tab:ldc2d_re5000_ablation_400sensors} shows that both residual-based weighting and the longer feedback path are important in this more challenging regime, which again demonstrates the effectiveness of the proposed block. So, we will use the full model for all remaining numerical experiments. We plot the reconstruction fields and the absolute errors in \autoref{fig:ldc2d_re5000_clean_noise_uvspeedp}. It is clear that, even for $\mathrm{Re} = 5000$ with 1\% noise level, one can still recover the velocity filed from the sparse observations of pressures on the boundary.

We also compare the two methods at $\mathrm{Re}=5000$ using 400 boundary pressure observations without noise. \autoref{tab:ldc2d_pinn_comparison} shows that the proposed method gives lower errors for all reported fields and the Reynolds number. The improvement in velocity reconstruction and parameter identification remains clear at this higher Reynolds number, indicating that the proposed method performs better than the vanilla PINN in the more challenging inverse setting.

We next examine the reconstruction performance for the 2D lid-driven cavity inverse problem at both $\mathrm{Re} = 3200$ and $\mathrm{Re} = 5000$. \autoref{fig:app_ldc_error_convergence} reports the relative $L_2$ errors of the velocity components and the pressure. It can be observed that, for both Reynolds numbers, all error metrics decay rapidly during the early training stage and stabilize at low levels, indicating that the proposed framework achieves accurate field reconstruction despite the pressure-only observations. The errors at $\mathrm{Re} = 5000$ are slightly higher than those at $\mathrm{Re} = 3200$, which is expected since the inverse problem becomes more ill-conditioned as the Reynolds number increases. \autoref{fig:app_reynolds_identification} shows the histories of the identified Reynolds number during training. In both cases, the estimated Reynolds number converges smoothly and remains close to the ground-truth values (indicated by the horizontal dashed lines), with only minor fluctuations after convergence. These results further confirm that the proposed method can reliably recover both the flow fields and the unknown parameter from sparse boundary pressure measurements, even under challenging high‑Reynolds‑number conditions.

\begin{figure}[htb]
    \centering
    \includegraphics[width=1.0\textwidth]{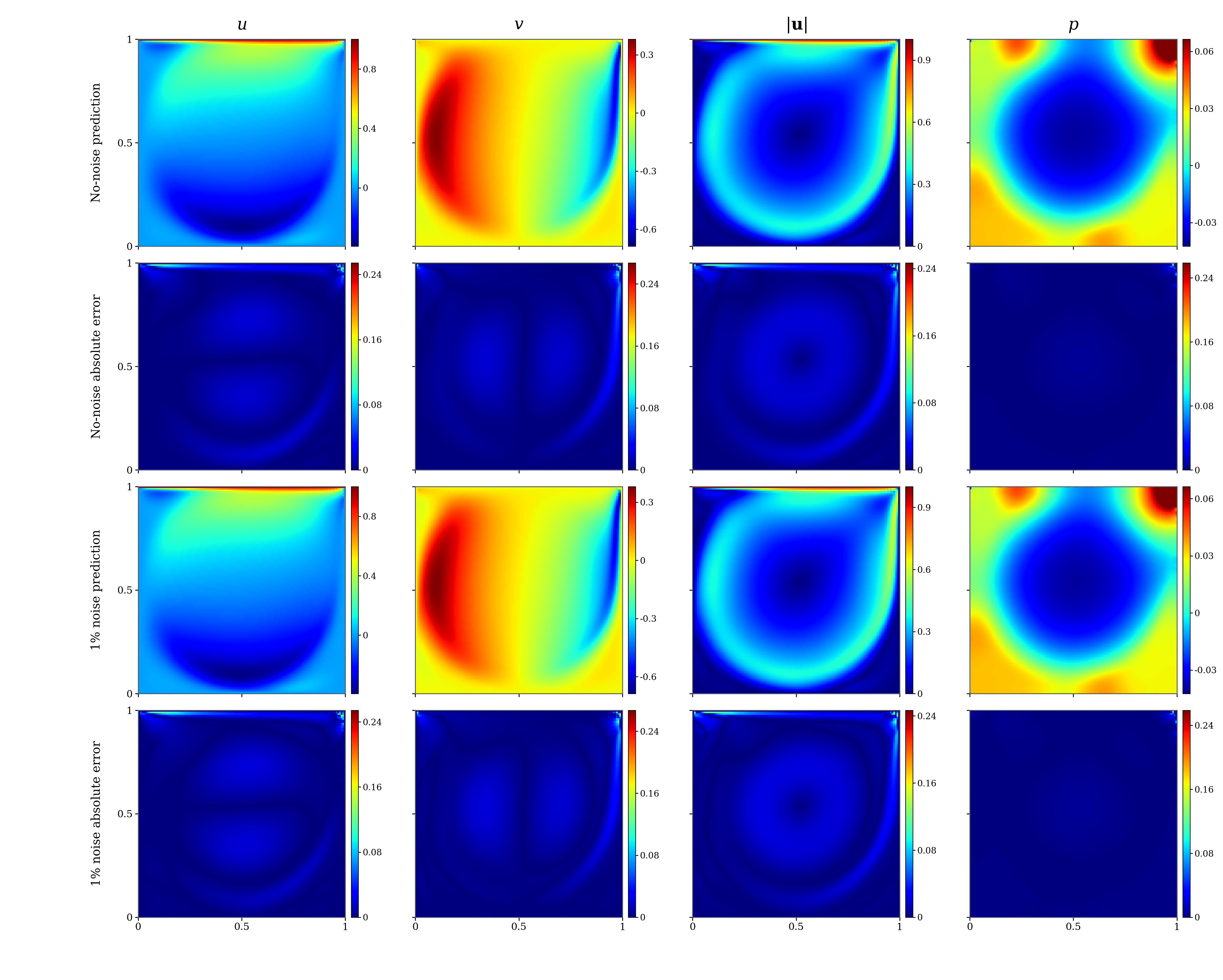}
    \caption{Predicted fields and point-wise errors for the 2D lid-driven cavity inverse problem at $\mathrm{Re}=5000$. The columns show $u$, $v$, $|\mathbf{u}|$, and $p$. The first two rows give the noise-free case. The last two rows give the case with $1\%$ pressure noise.}
    \label{fig:ldc2d_re5000_clean_noise_uvspeedp}
\end{figure}

\begin{table*}[htb]
\centering
\caption{Reconstruction and parameter inversion errors for the 2D lid-driven cavity inverse problem at $\mathrm{Re}=5000$ under different numbers of pressure observations and noise levels.}
\label{tab:ldc2d_re5000_sensor_noise_results}
\resizebox{\textwidth}{!}{
\begin{tabular}{ccccccc}
\toprule
\# Sensors & Noise level & $e_{|\mathbf{u}|}$ & $e_u$ & $e_v$ & $e_p$ & $e_{\mathrm{Re}}$ \\
\midrule
200 & 0.00 & $6.072\times 10^{-1}$ & $7.594\times 10^{-1}$ & $8.456\times 10^{-1}$ & $7.072\times 10^{-1}$ & $1.081\times 10^{-1}$ \\
300 & 0.00 & $1.502\times 10^{-1}$ & $1.408\times 10^{-1}$ & $1.655\times 10^{-1}$ & $2.412\times 10^{-1}$ & $3.688\times 10^{-1}$ \\
400 & 0.00 & $1.030\times 10^{-1}$ & $9.743\times 10^{-2}$ & $1.197\times 10^{-1}$ & $2.959\times 10^{-1}$ & $5.999\times 10^{-2}$ \\
\midrule
200 & 0.01 & $6.037\times 10^{-1}$ & $7.796\times 10^{-1}$ & $8.824\times 10^{-1}$ & $7.437\times 10^{-1}$ & $2.962\times 10^{-3}$ \\
300 & 0.01 & $1.451\times 10^{-1}$ & $1.364\times 10^{-1}$ & $1.596\times 10^{-1}$ & $2.337\times 10^{-1}$ & $3.531\times 10^{-1}$ \\
400 & 0.01 & $1.436\times 10^{-1}$ & $1.374\times 10^{-1}$ & $1.590\times 10^{-1}$ & $2.906\times 10^{-1}$ & $4.181\times 10^{-2}$ \\
\midrule
200 & 0.05 & $6.124\times 10^{-1}$ & $7.260\times 10^{-1}$ & $8.043\times 10^{-1}$ & $6.709\times 10^{-1}$ & $1.857\times 10^{-1}$ \\
300 & 0.05 & $1.582\times 10^{-1}$ & $1.490\times 10^{-1}$ & $1.733\times 10^{-1}$ & $2.496\times 10^{-1}$ & $3.834\times 10^{-1}$ \\
400 & 0.05 & $1.153\times 10^{-1}$ & $1.134\times 10^{-1}$ & $1.287\times 10^{-1}$ & $2.755\times 10^{-1}$ & $9.922\times 10^{-2}$ \\
\bottomrule
\end{tabular}
}
\end{table*}

\begin{table*}[htb]
\centering
\caption{Ablation study for the 2D lid-driven cavity inverse problem at $\mathrm{Re}=5000$ with 400 pressure observations and no observation noise.}
\label{tab:ldc2d_re5000_ablation_400sensors}
\resizebox{\textwidth}{!}{
\begin{tabular}{cccccccc}
\toprule
Case & Feedback layers & Sensor weighting & 
$e_{|\mathbf{u}|}$ & $e_u$ & $e_v$ & $e_p$ & $e_{\mathrm{Re}}$ \\
\midrule
Full model 
& 2 
& Residual 
& $1.030\times 10^{-1}$ 
& $9.743\times 10^{-2}$ 
& $1.197\times 10^{-1}$ 
& $2.959\times 10^{-1}$ 
& $5.999\times 10^{-2}$ \\

No sensor weighting
& 2 
& None 
& $5.380\times 10^{-1}$ 
& $7.385\times 10^{-1}$ 
& $8.496\times 10^{-1}$ 
& $7.356\times 10^{-1}$ 
& $2.088\times 10^{-1}$ \\

One feedback layer
& 1 
& Residual 
& $5.548\times 10^{-1}$ 
& $8.461\times 10^{-1}$ 
& $1.058\times 10^{0}$ 
& $8.056\times 10^{-1}$ 
& $1.061\times 10^{0}$ \\

One feedback layer, no sensor weighting
& 1 
& None 
& $5.608\times 10^{-1}$ 
& $8.719\times 10^{-1}$ 
& $1.117\times 10^{0}$ 
& $9.838\times 10^{-1}$ 
& $8.701\times 10^{-1}$ \\
\bottomrule
\end{tabular}
}
\end{table*}

\begin{figure}[htb]
    \centering

    \begin{subfigure}[t]{0.32\textwidth}
        \centering
        \includegraphics[width=\linewidth]
        {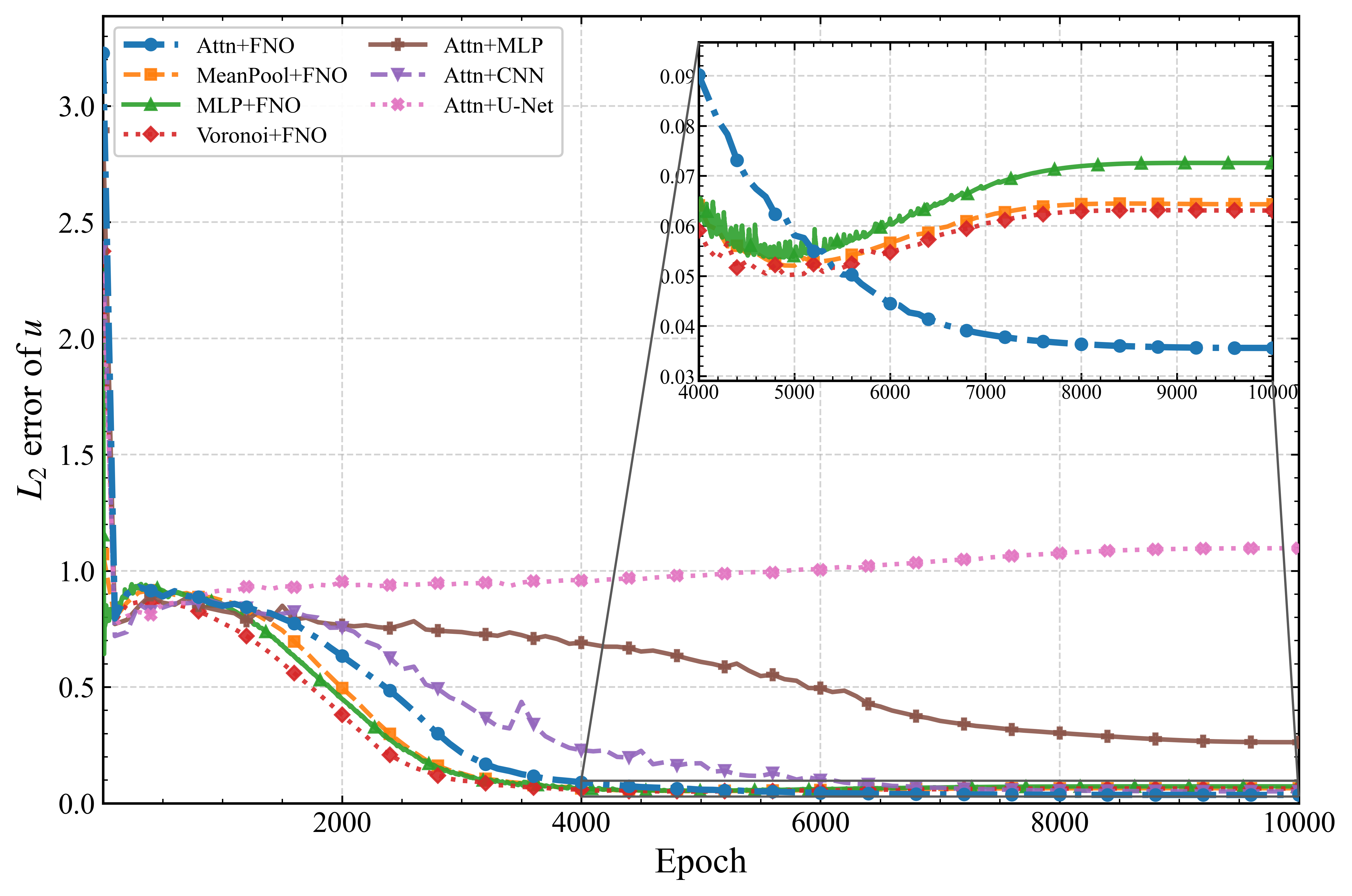}
        \caption{$u$, $\mathrm{Re}=3200$}
        \label{fig:app_l2_u_re3200}
    \end{subfigure}
    \hfill
    \begin{subfigure}[t]{0.32\textwidth}
        \centering
        \includegraphics[width=\linewidth]
        {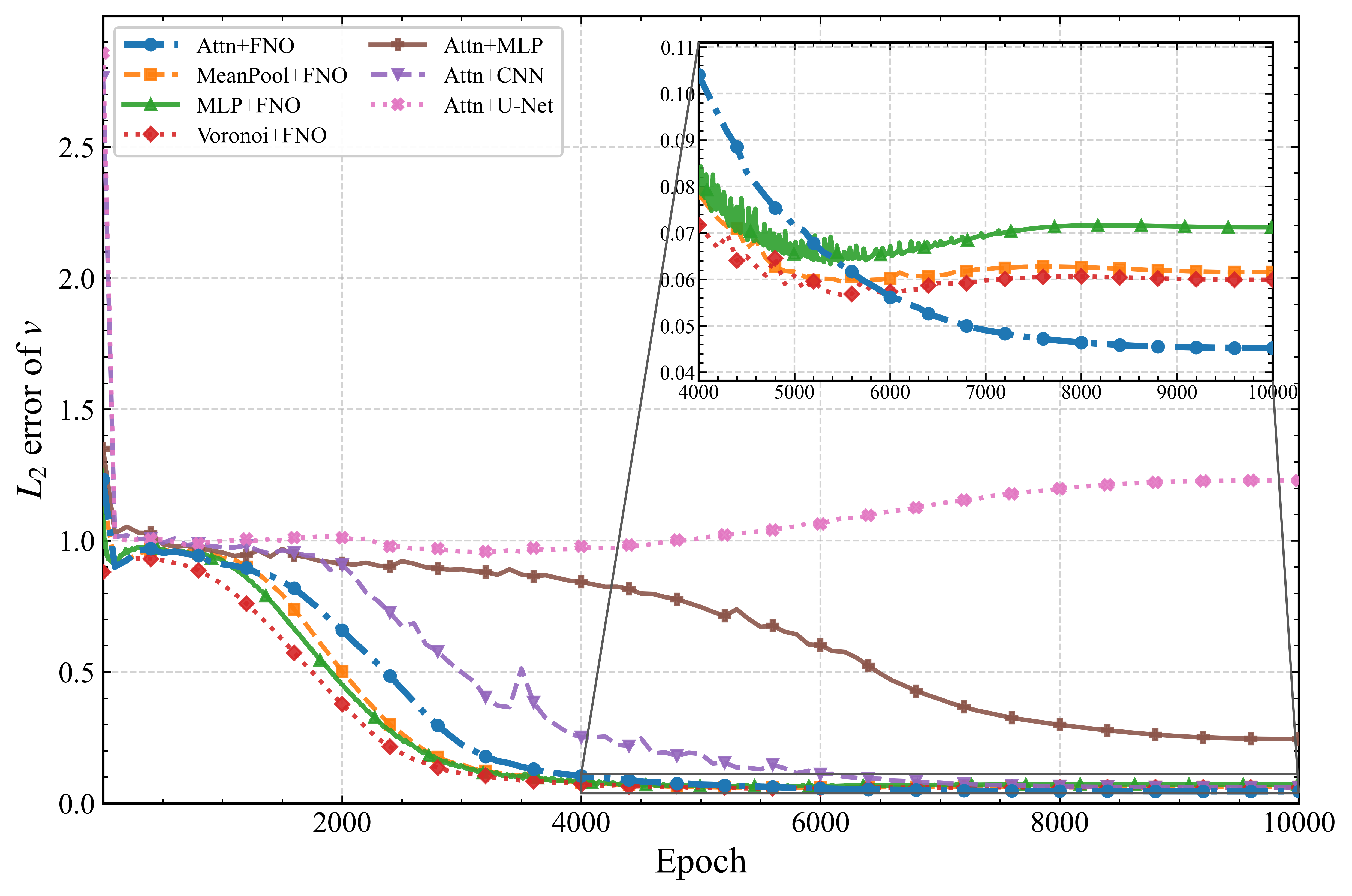}
        \caption{$v$, $\mathrm{Re}=3200$}
        \label{fig:app_l2_v_re3200}
    \end{subfigure}
    \hfill
    \begin{subfigure}[t]{0.32\textwidth}
        \centering
        \includegraphics[width=\linewidth]
        {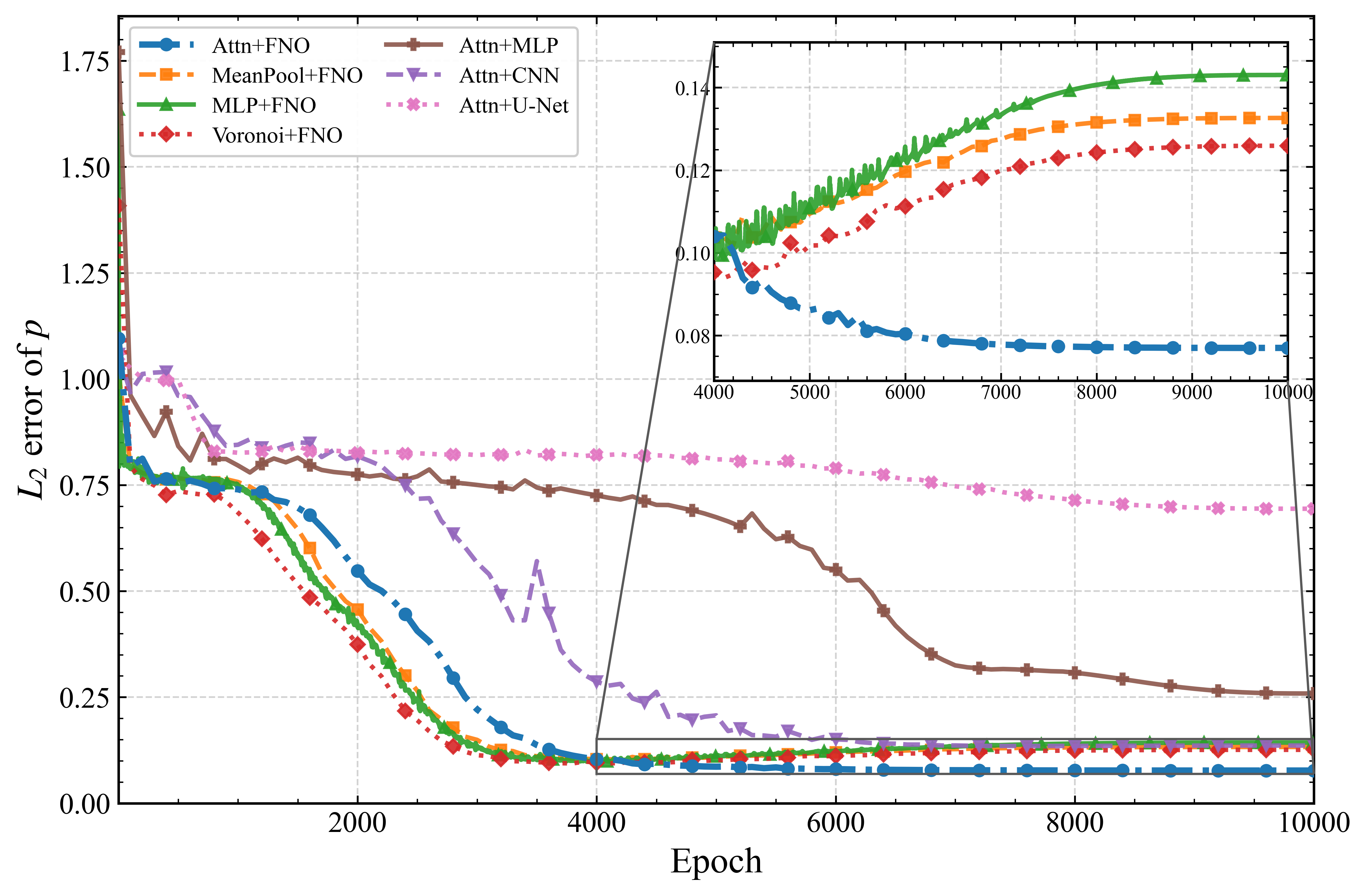}
        \caption{$p$, $\mathrm{Re}=3200$}
        \label{fig:app_l2_p_re3200}
    \end{subfigure}

    \vspace{0.6em}

    \begin{subfigure}[t]{0.32\textwidth}
        \centering
        \includegraphics[width=\linewidth]
        {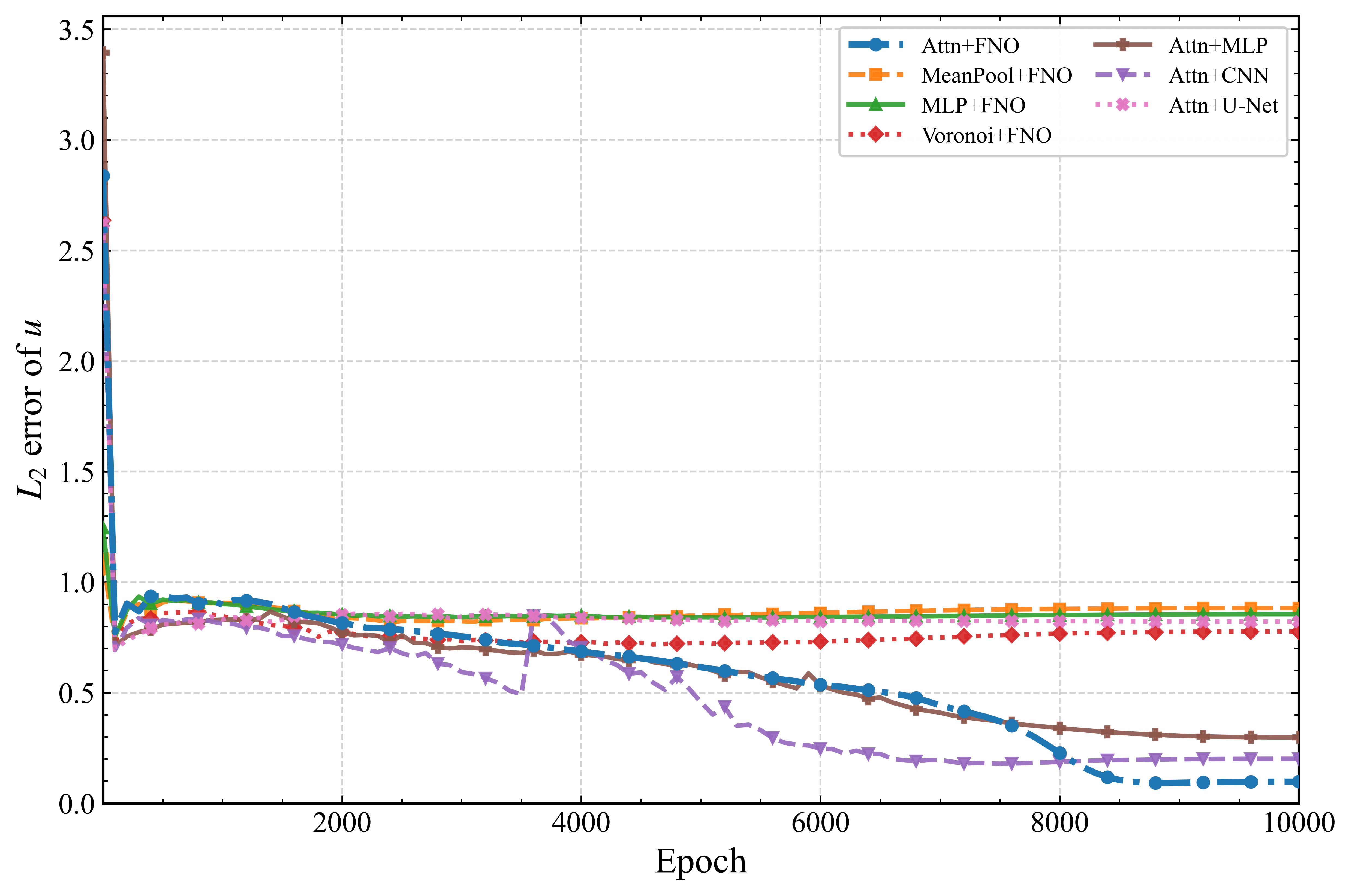}
        \caption{$u$, $\mathrm{Re}=5000$}
        \label{fig:app_l2_u_re5000}
    \end{subfigure}
    \hfill
    \begin{subfigure}[t]{0.32\textwidth}
        \centering
        \includegraphics[width=\linewidth]
        {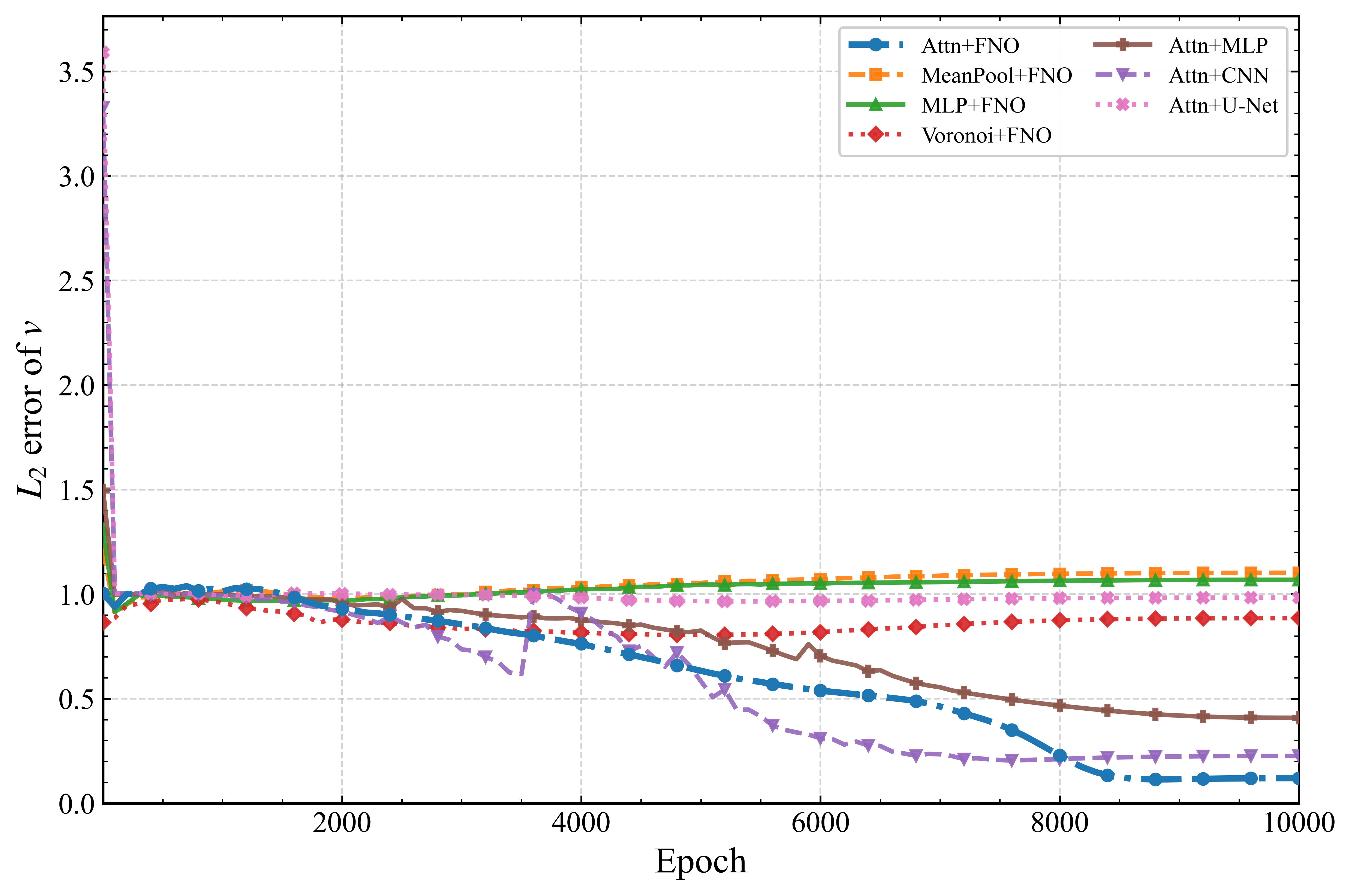}
        \caption{$v$, $\mathrm{Re}=5000$}
        \label{fig:app_l2_v_re5000}
    \end{subfigure}
    \hfill
    \begin{subfigure}[t]{0.32\textwidth}
        \centering
        \includegraphics[width=\linewidth]
        {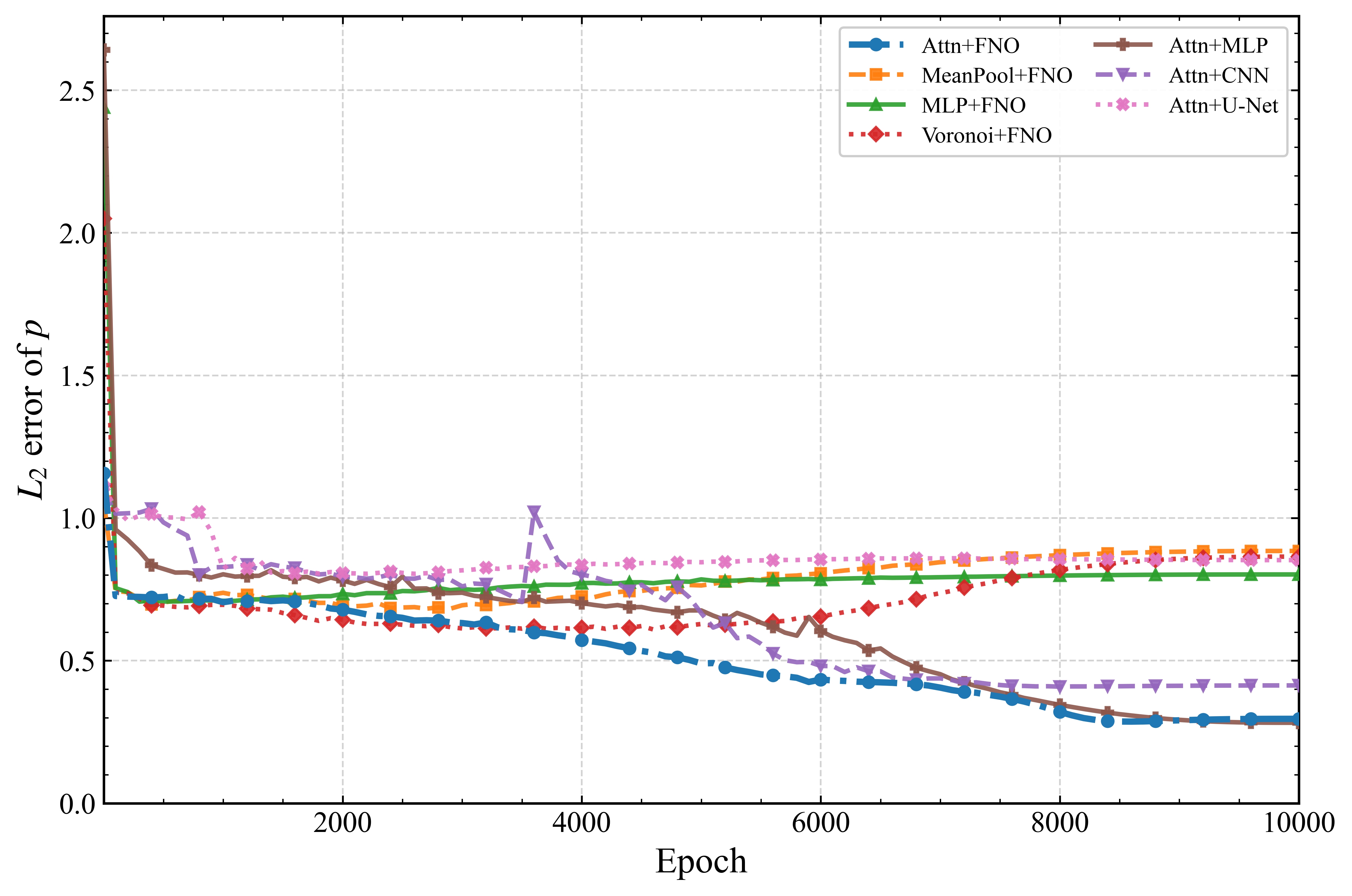}
        \caption{$p$, $\mathrm{Re}=5000$}
        \label{fig:app_l2_p_re5000}
    \end{subfigure}

    \caption{
    Relative $L_2$ errors for the
    two-dimensional lid-driven cavity inverse problem.
    The first row reports the results at $\mathrm{Re}=3200$, and the
    second row reports the results at $\mathrm{Re}=5000$.
    The columns correspond to the errors of $u$, $v$, and $p$,
    respectively.
    }
    \label{fig:app_ldc_error_convergence}
\end{figure}

\begin{figure}[htb]
    \centering

    \begin{subfigure}[t]{0.485\textwidth}
        \centering
        \includegraphics[width=\linewidth]
        {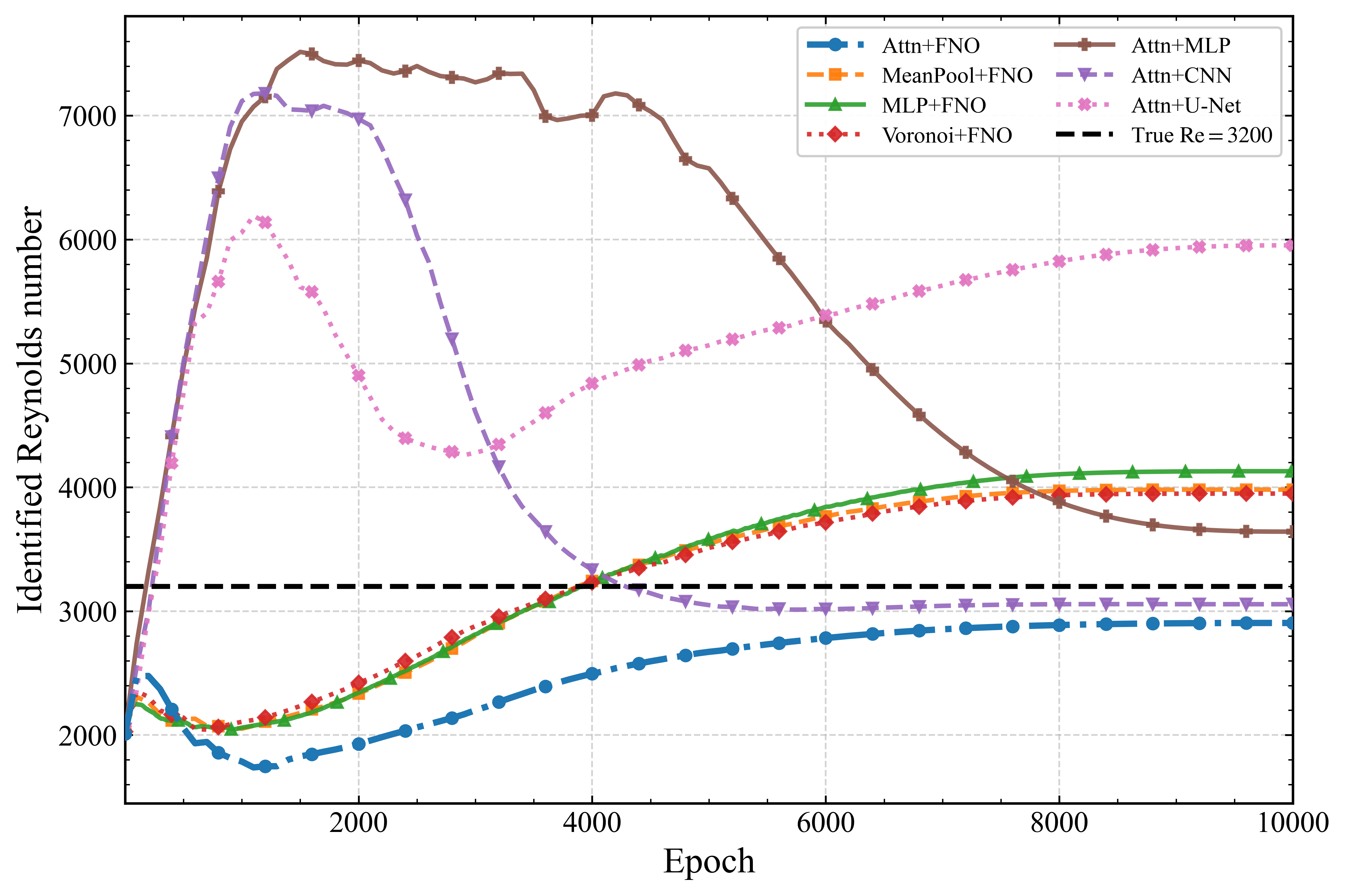}
        \caption{$\mathrm{Re}=3200$}
        \label{fig:app_identified_re3200}
    \end{subfigure}
    \hfill
    \begin{subfigure}[t]{0.485\textwidth}
        \centering
        \includegraphics[width=\linewidth]
        {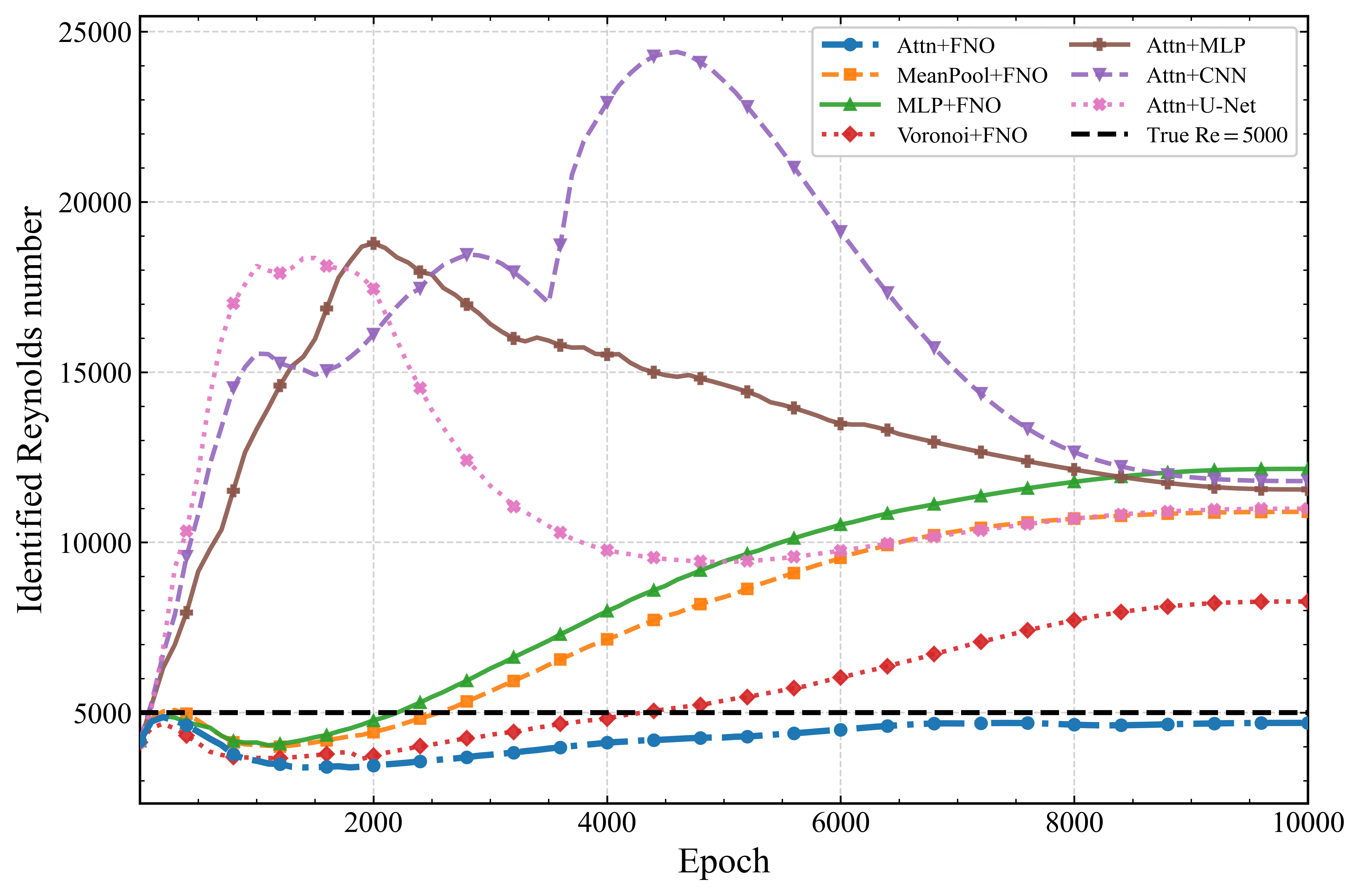}
        \caption{$\mathrm{Re}=5000$}
        \label{fig:app_identified_re5000}
    \end{subfigure}

    \caption{
    Histories of the identified Reynolds number for the
    two-dimensional lid-driven cavity inverse problem. The horizontal
    dashed lines indicate the corresponding ground-truth values.
    }
    \label{fig:app_reynolds_identification}
\end{figure}

\FloatBarrier

\subsection{3D Lid-Driven Cavity Flow}

We further extend the inverse reconstruction task to the three-dimensional lid-driven cavity flow governed by the steady incompressible Navier--Stokes equations:
\begin{equation*}\label{eq:ldc3d_param}
\left\{
\begin{aligned}
(\mathbf{u}\cdot\nabla)\mathbf{u}+\nabla p
&=\nu\Delta\mathbf{u}, && \text{in } \Omega,\\
\nabla\cdot\mathbf{u} &=0, && \text{in } \Omega,\\
\mathbf{u} &=\mathbf{g}, && \text{on } \partial\Omega,
\end{aligned}
\right.
\end{equation*}
where $\Omega=(0,1)^3$, $\mathbf{u}=(u,v,w)^{\top}$ denotes the three-dimensional velocity field, $p$ denotes the pressure field, and $\nu=1/\mathrm{Re}$ is the kinematic viscosity. Following the numerical setup, the moving lid is imposed on the top boundary, while the other five walls satisfy no-slip boundary conditions:
\begin{equation*}
\mathbf{g}(x,y,z)=
\begin{cases}
(1,0,0)^{\top}, & y=1,\\
(0,0,0)^{\top}, & \text{on the remaining boundary}.
\end{cases}
\end{equation*}

Only sparse boundary pressure observations are used, while all three velocity components remain unobserved. Our goal is to reconstruct $(u,v,w,p)$ and identify $\nu$ for the target case $\mathrm{Re}=500$. The reference solution is generated with OpenFOAM~\cite{weller1998tensorial} on a uniform $129\times129\times129$ finite volume mesh, from which 500, 1000, or 1500 boundary pressure locations are randomly selected.

The encoder consists of a two-layer sensor MLP with a hidden width and output dimension of 64, four-head attention modules, one sensor self-attention layer, and a bidirectional sensor--grid information path with two feedback stages. The grid tokens used for feedback are downsampled with a stride of 4 in each spatial direction. The FNO decoder consists of four residual FNO blocks with a channel width of 32. Each block combines a spectral convolution retaining 16 Fourier modes in each spatial direction with a parallel $1\times1\times1$ convolution, followed by group normalization and GELU activation. Replicate padding of size 4 is applied before the FNO blocks. The predicted pressures at the sensor locations are obtained through trilinear sampling, while the PDE residuals are evaluated using second-order central finite differences. The initial learning rates are $10^{-3}$ for the network and $5\times10^{-3}$ for $\log\nu$. The Reynolds number is initialized uniformly within $[50,1000]$ and optimized through $\log\nu$ without a prior. During the first 500 epochs, all sensor weights are fixed at one; thereafter, they are updated every 200 epochs according to the PDE residuals and used in both the observation loss and the attention calculation. The reference velocity field is excluded from the training loss.

\begin{table*}[!t]
\centering
\caption{Reconstruction and parameter inversion errors for the 3D lid-driven cavity inverse problem under different numbers of boundary pressure observations and noise levels.}
\label{tab:ldc3d_pressure_noise_results}
\resizebox{\textwidth}{!}{
\begin{tabular}{cccccccc}
\toprule
Noise level & \# Pressure obs. & $e_{|\mathbf{u}|}$ & $e_u$ & $e_v$ & $e_w$ & $e_p$ & $e_{\mathrm{Re}}$ \\
\midrule
\multirow{3}{*}{0.00}
& 500  & $1.086\times 10^{-1}$ & $1.211\times 10^{-1}$ & $9.763\times 10^{-2}$ & $1.048\times 10^{-1}$ & $1.489\times 10^{-1}$ & $1.471\times 10^{-2}$ \\
& 1000 & $1.034\times 10^{-1}$ & $1.157\times 10^{-1}$ & $9.309\times 10^{-2}$ & $8.912\times 10^{-2}$ & $1.311\times 10^{-1}$ & $6.479\times 10^{-2}$ \\
& 1500 & $1.019\times 10^{-1}$ & $1.142\times 10^{-1}$ & $9.209\times 10^{-2}$ & $7.017\times 10^{-2}$ & $1.304\times 10^{-1}$ & $7.833\times 10^{-2}$ \\
\midrule
\multirow{3}{*}{0.01}
& 500  & $1.073\times 10^{-1}$ & $1.197\times 10^{-1}$ & $9.591\times 10^{-2}$ & $1.113\times 10^{-1}$ & $1.444\times 10^{-1}$ & $2.912\times 10^{-2}$ \\
& 1000 & $1.025\times 10^{-1}$ & $1.149\times 10^{-1}$ & $9.176\times 10^{-2}$ & $8.767\times 10^{-2}$ & $1.296\times 10^{-1}$ & $7.311\times 10^{-2}$ \\
& 1500 & $1.040\times 10^{-1}$ & $1.158\times 10^{-1}$ & $9.534\times 10^{-2}$ & $7.447\times 10^{-2}$ & $1.317\times 10^{-1}$ & $6.282\times 10^{-2}$ \\
\midrule
\multirow{3}{*}{0.05}
& 500  & $1.084\times 10^{-1}$ & $1.209\times 10^{-1}$ & $9.761\times 10^{-2}$ & $1.114\times 10^{-1}$ & $1.471\times 10^{-1}$ & $2.147\times 10^{-2}$ \\
& 1000 & $1.068\times 10^{-1}$ & $1.186\times 10^{-1}$ & $9.776\times 10^{-2}$ & $1.055\times 10^{-1}$ & $1.371\times 10^{-1}$ & $4.758\times 10^{-2}$ \\
& 1500 & $1.057\times 10^{-1}$ & $1.174\times 10^{-1}$ & $9.796\times 10^{-2}$ & $8.339\times 10^{-2}$ & $1.368\times 10^{-1}$ & $5.251\times 10^{-2}$ \\
\bottomrule
\end{tabular}
}
\end{table*}

\autoref{tab:ldc3d_pressure_noise_results} shows the reconstruction errors of the physical fields and the Reynolds number for the three-dimensional lid-driven cavity inverse problem, still different numbers of boundary pressure sensors and noise levels are considered. It can be seen that, in the noise-free case, increasing the number of sensors from 500 to 1500 reduces the velocity and pressure errors slightly. The error for the Reynolds number does not exhibit a monotonic decrease as the number of sensors increases. When a noise of 1\% is added, the overall field errors remain close to their noise-free scenarios, indicating that the proposed framework is largely insensitive to low-level perturbations. Even at a 5\% noise level, the velocity errors remain close to $10^{-1}$, and increasing the number of sensors reduces most of the reconstruction errors, especially $e_w$ and $e_p$. The Reynolds-number error remains more sensitive and non-monotone. These results demonstrate that the proposed method achieves robust field reconstruction under moderate noise, and that increasing the number of pressure sensors serves as an effective means of improving the reconstruction accuracy, although the Reynolds number remains more difficult to identify reliably.

\autoref{tab:ldc3d_re500_ablation_1500sensors} presents the ablation study results for the three-dimensional lid-driven cavity test problem at $Re=500$ with 1500 pressure observations and no observation noise. From \autoref{tab:ldc3d_re500_ablation_1500sensors}, we can see that the full model yields the lowest pressure error, whereas the cases without the feedback layer give smaller errors for most velocity metrics. The differences among the four variants, however, remain modest, indicating that they provide comparable overall field reconstruction performance. Reducing the feedback path slightly improves the velocity reconstruction but results in larger pressure and Reynolds-number errors. Removing sensor weighting while retaining the feedback layer gives the lowest $e_{Re}$, but slightly increases most of the field errors. These results indicate that the feedback module and residual-based sensor weighting mainly balance the accuracy of field reconstruction and Reynolds-number identification in this test problem. \autoref{fig:ldc3d_re500_noise_uvspeedp_pred_error} shows the absolute errors of the reconstruction fields, which can be seen that our method is able to reconstruct the three-dimensional velocity field from sparse boundary observations.

We further compare the proposed method with a vanilla PINN for the 3D lid-driven cavity inverse problem using 1500 boundary pressure observations without noise. As shown in \autoref{tab:ldc3d_re500_pinn_comparison_1500}, the proposed method yields lower errors for all reconstructed fields and the Reynolds number. The improvement is particularly pronounced for the $w$ component. These results indicate that the proposed method provides more accurate field reconstruction and parameter identification from sparse pressure observations in three dimensions. From \autoref{fig:ldc3d_attn_vs_pinn_curves}, the proposed method reaches lower relative $L_2$ errors for $u$, $v$, and $p$, while the vanilla PINN levels off at higher errors. The identified Reynolds number also stays close to the true value of 500 for the proposed method, whereas the PINN estimate eventually rises well above it.

\begin{table}[H]
\centering
\caption{Ablation study for the 3D lid-driven cavity inverse problem at $\mathrm{Re}=500$ with 1500 boundary pressure observations and no observation noise.}
\label{tab:ldc3d_re500_ablation_1500sensors}
\resizebox{\textwidth}{!}{
\begin{tabular}{ccccccccc}
\toprule
Case & Feedback layers & Sensor weighting &
$e_{|\mathbf{u}|}$ & $e_u$ & $e_v$ & $e_w$ & $e_p$ & $e_{\mathrm{Re}}$ \\
\midrule
Full model
& 1
& Residual
& $1.019\times 10^{-1}$
& $1.142\times 10^{-1}$
& $9.209\times 10^{-2}$
& $7.017\times 10^{-2}$
& $1.304\times 10^{-1}$
& $7.833\times 10^{-2}$ \\

No sensor weighting
& 1
& None
& $1.022\times 10^{-1}$
& $1.145\times 10^{-1}$
& $9.223\times 10^{-2}$
& $7.005\times 10^{-2}$
& $1.325\times 10^{-1}$
& $7.253\times 10^{-2}$ \\

No feedback layer
& 0
& Residual
& $9.955\times 10^{-2}$
& $1.126\times 10^{-1}$
& $8.810\times 10^{-2}$
& $6.890\times 10^{-2}$
& $1.312\times 10^{-1}$
& $9.605\times 10^{-2}$ \\

No feedback layer, no sensor weighting
& 0
& None
& $9.971\times 10^{-2}$
& $1.128\times 10^{-1}$
& $8.817\times 10^{-2}$
& $6.791\times 10^{-2}$
& $1.324\times 10^{-1}$
& $9.226\times 10^{-2}$ \\
\bottomrule
\end{tabular}
}
\end{table}

\begin{table*}[!t]
\centering
\caption{Comparison with vanilla PINN for the 3D lid-driven cavity inverse problem at $\mathrm{Re}=500$ using 1500 boundary pressure observations without noise.}
\label{tab:ldc3d_re500_pinn_comparison_1500}
\resizebox{\textwidth}{!}{
\begin{tabular}{lcccccc}
\toprule
Method & $e_{|\mathbf{u}|}$ & $e_u$ & $e_v$ & $e_w$ & $e_p$ & $e_{\mathrm{Re}}$ \\
\midrule
Vanilla PINN
& $5.851\times 10^{-1}$
& $5.947\times 10^{-1}$
& $9.285\times 10^{-1}$
& $2.490\times 10^{0}$
& $4.789\times 10^{-1}$
& $2.862\times 10^{-1}$ \\
Ours
& $\mathbf{1.019\times 10^{-1}}$
& $\mathbf{1.142\times 10^{-1}}$
& $\mathbf{9.209\times 10^{-2}}$
& $\mathbf{7.017\times 10^{-2}}$
& $\mathbf{1.304\times 10^{-1}}$
& $\mathbf{7.833\times 10^{-2}}$ \\
\bottomrule
\end{tabular}
}
\end{table*}

\begin{figure}[H]
    \centering
    \includegraphics[width=1.0\textwidth]{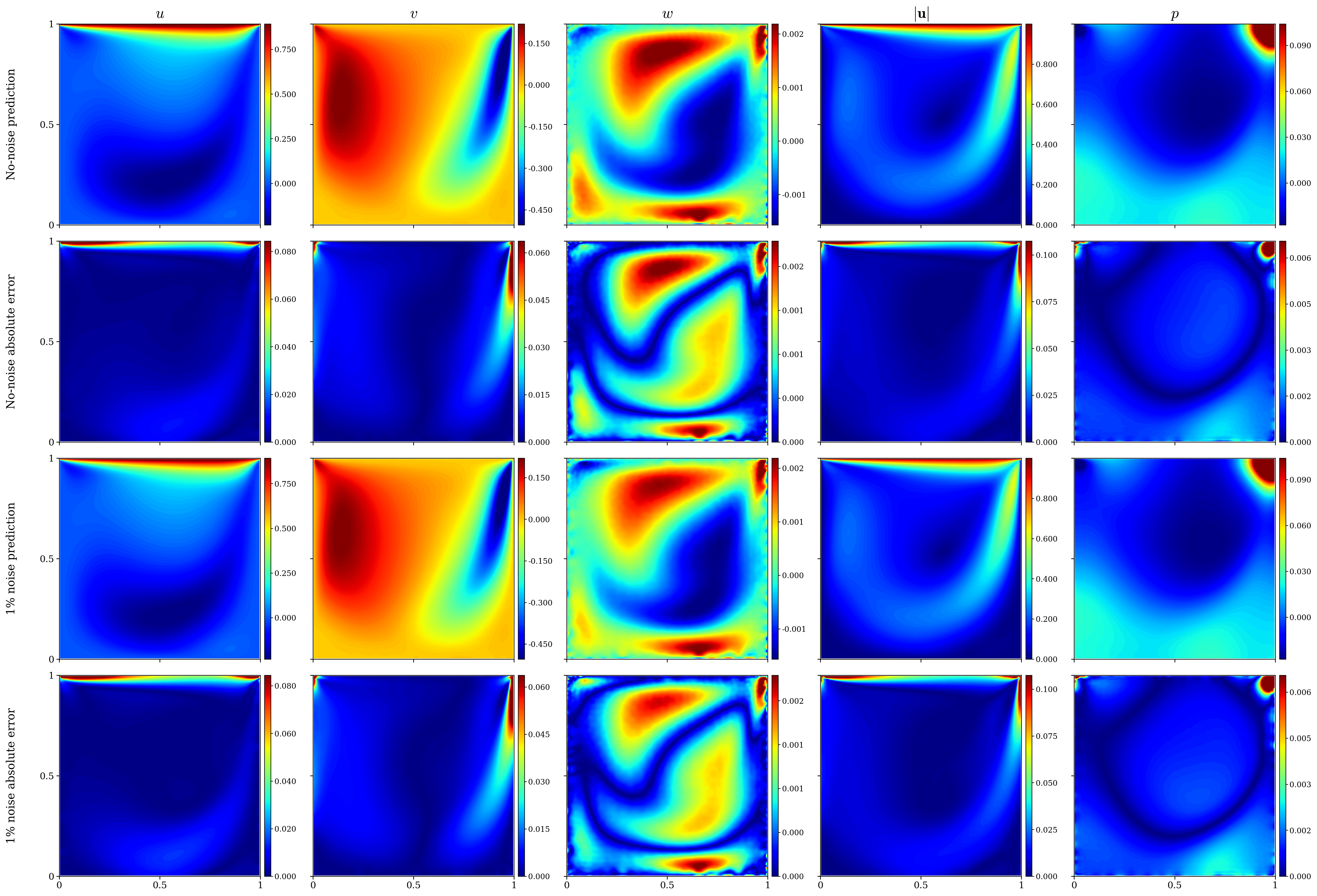}
    \caption{Predicted fields and absolute errors for the 3D lid-driven cavity inverse problem at $\mathrm{Re}=500$. The plots are shown on the mid-plane slice. The columns show $u$, $v$, $w$, $|\mathbf{u}|$, and $p$. The first two rows give the noise-free case, and the last two rows show the case with $1\%$ pressure noise.}
    \label{fig:ldc3d_re500_noise_uvspeedp_pred_error}
\end{figure}

\begin{figure}[H]
    \centering

    \begin{subfigure}[t]{0.48\textwidth}
        \centering
        \includegraphics[width=\linewidth]{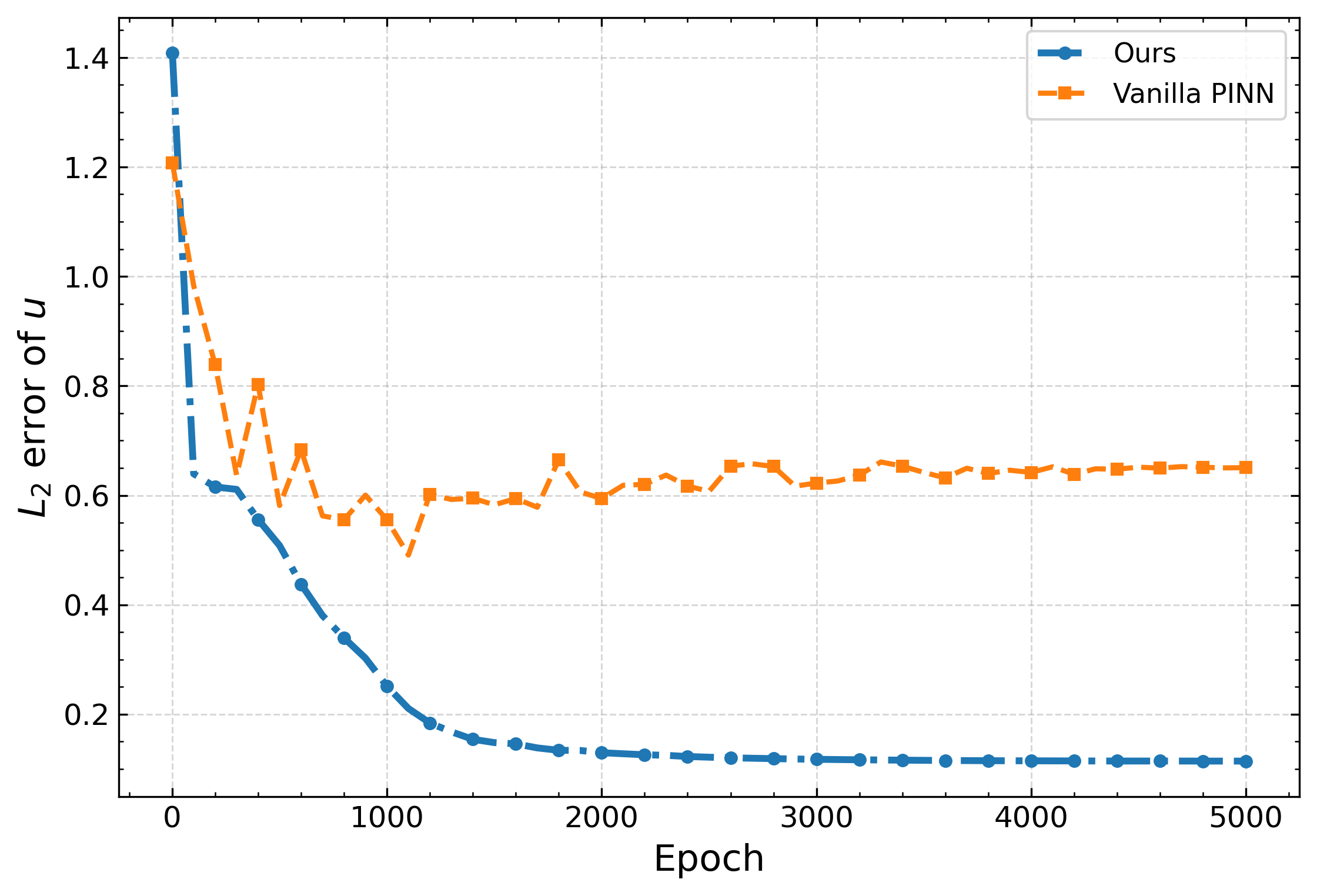}
        \caption{$u$}
        \label{fig:ldc3d_l2_u_attn_vs_pinn}
    \end{subfigure}
    \hfill
    \begin{subfigure}[t]{0.48\textwidth}
        \centering
        \includegraphics[width=\linewidth]{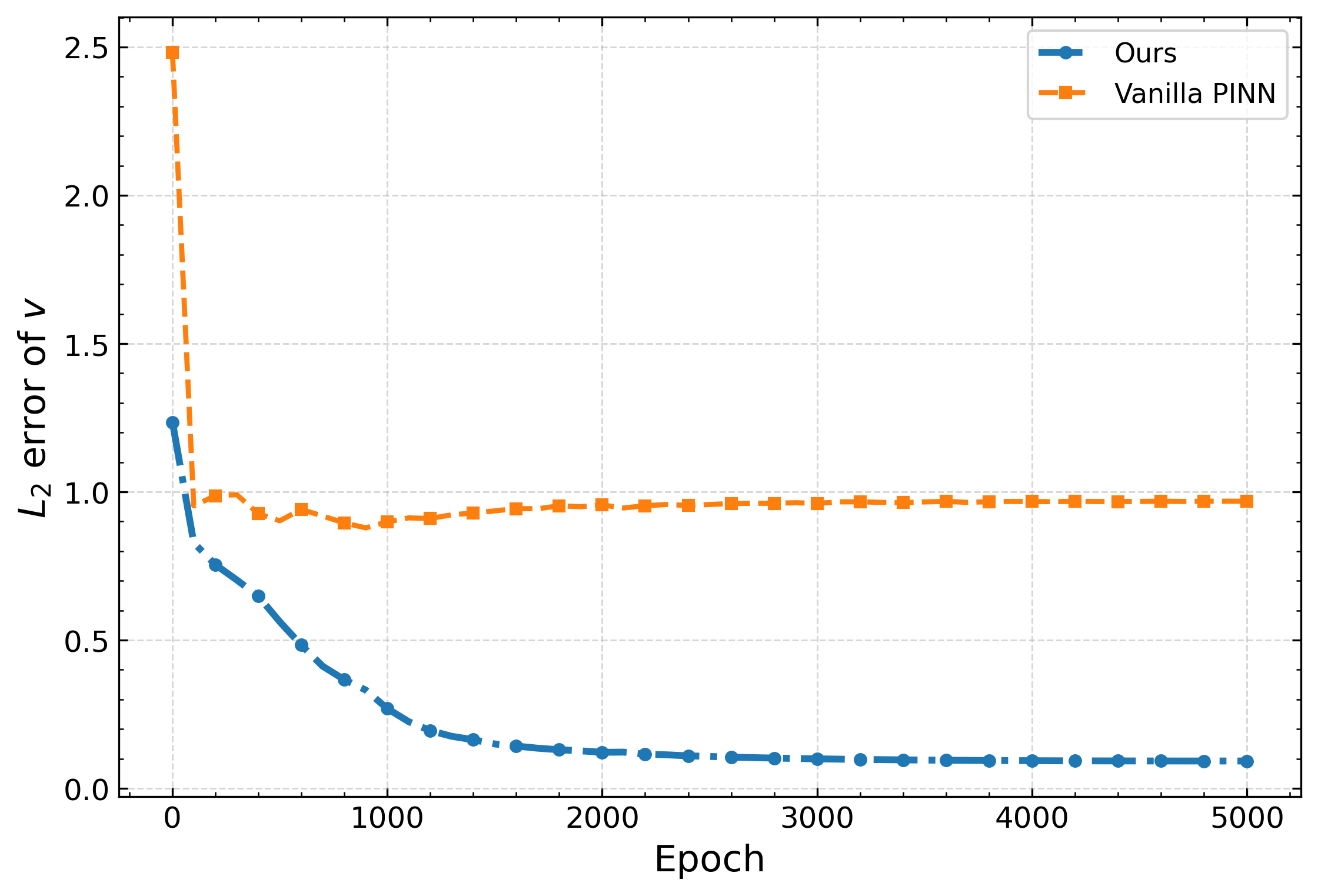}
        \caption{$v$}
        \label{fig:ldc3d_l2_v_attn_vs_pinn}
    \end{subfigure}

    \vspace{0.3cm}

    \begin{subfigure}[t]{0.48\textwidth}
        \centering
        \includegraphics[width=\linewidth]{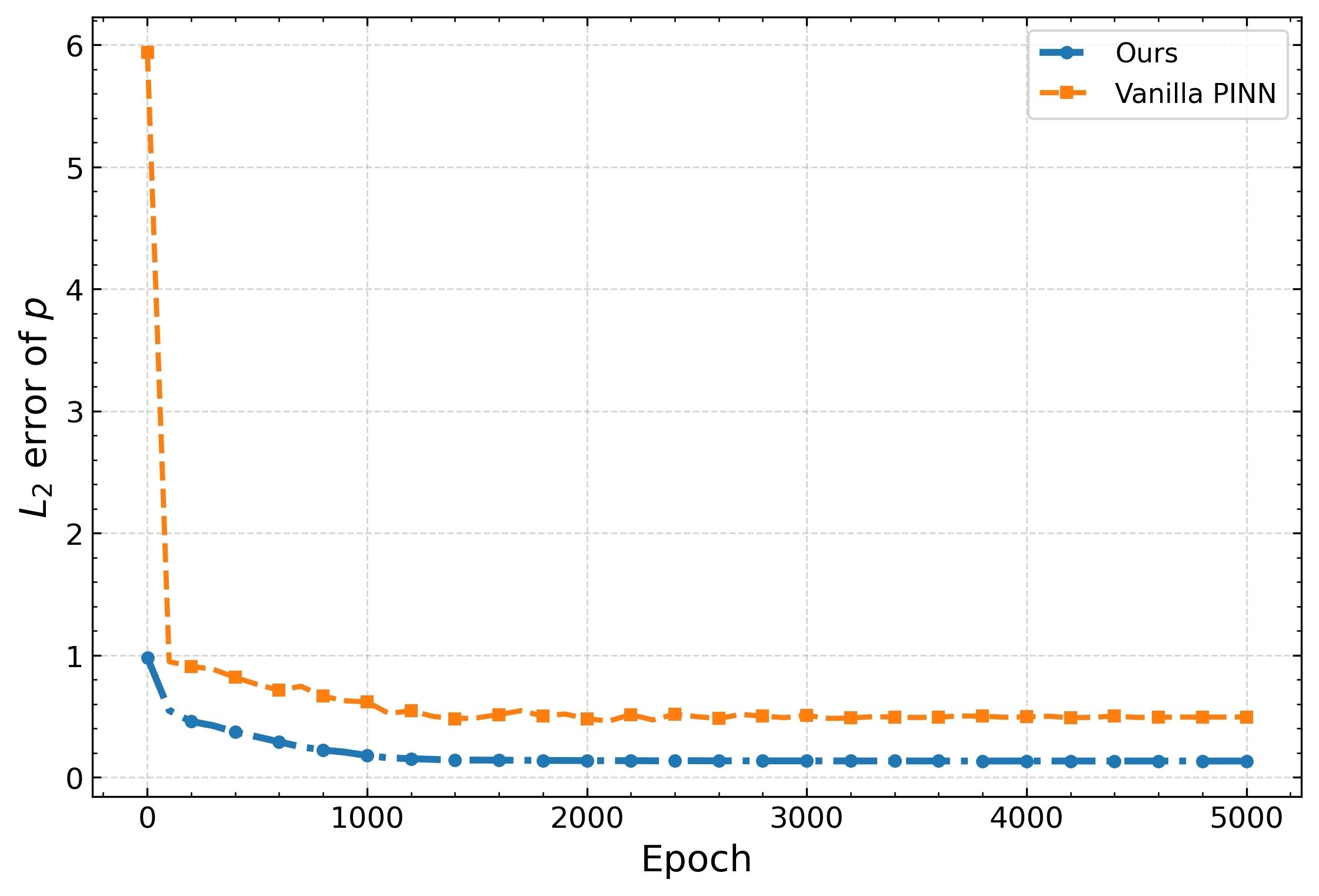}
        \caption{$p$}
        \label{fig:ldc3d_l2_p_attn_vs_pinn}
    \end{subfigure}
    \hfill
    \begin{subfigure}[t]{0.48\textwidth}
        \centering
        \includegraphics[width=\linewidth]{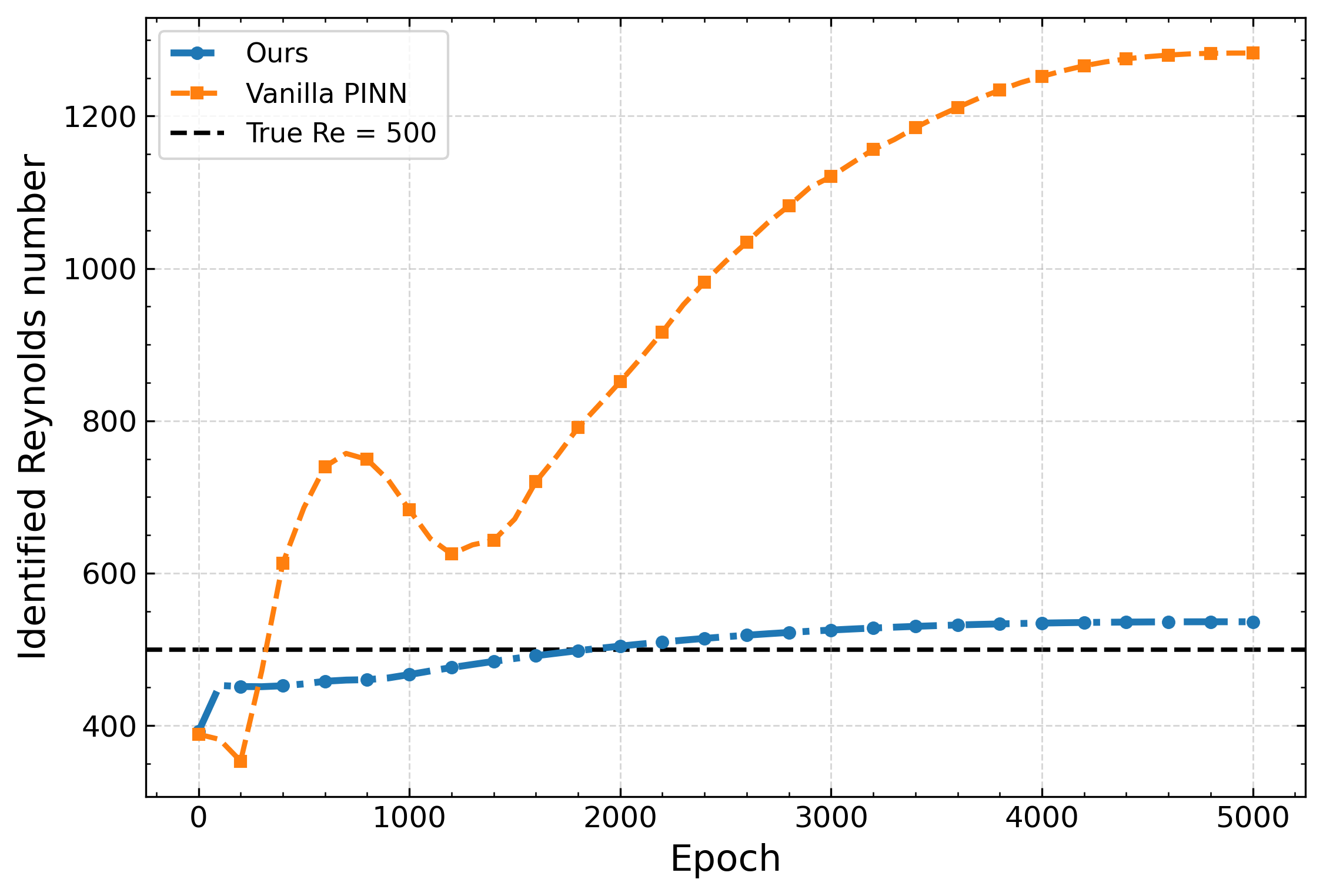}
        \caption{Identified $\mathrm{Re}$}
        \label{fig:ldc3d_re_attn_vs_pinn}
    \end{subfigure}

    \caption{Training histories for the 3D lid-driven cavity inverse problem at
    $\mathrm{Re}=500$. The first three panels compare the relative $L_2$ errors
    of $u$, $v$, and the mean-free pressure $p$ between Attn+FNO and Vanilla
    PINN. The last panel shows their identified Reynolds numbers.}
    \label{fig:ldc3d_attn_vs_pinn_curves}
\end{figure}

\subsection{2D Cylinder Wake}

We next consider the unsteady two-dimensional cylinder wake inverse problem. The data are taken from the dataset introduced by Raissi et al.~\cite{raissi2019physics} and are interpolated from scattered spectral-element points onto a regular space-time grid. The computational domain is
\[
\Omega=[1,8]\times[-2,2],\qquad t\in[0,19.9],
\]
with a regular grid of size $128\times64$ and 199 time frames. The flow is governed by the unsteady incompressible Navier--Stokes equations:
\begin{equation}\label{eq:cylinder_param}
\left\{
\begin{aligned}
u_t+\lambda_1(uu_x+vu_y)+p_x-\nu\Delta u &=0,\\
v_t+\lambda_1(uv_x+vv_y)+p_y-\nu\Delta v &=0,\\
u_x+v_y &=0.
\end{aligned}
\right.
\end{equation}
Here, $\lambda_1$ denotes the convection coefficient and $\nu=1/\mathrm{Re}$ is the kinematic viscosity, both of which are treated as unknown learnable physical parameters. Unlike the Lid-driven cavity cases, the cylinder wake is an unsteady open flow problem dominated by strong convective transport and periodic vortex shedding. Once the velocity field and the corresponding boundary conditions are given, pressure can be recovered, up to an additive constant, from the momentum equations or the associated pressure Poisson equation. Conversely, when only sparse boundary pressure observations are available without velocity initial or boundary data, such scalar and indirect information is generally insufficient to uniquely determine both velocity components and the phase of the interior vortical structures, making the problem illness in the present setting. Therefore, only sparse $(u,v)$ observations are used, which are randomly sampled from the four boundaries over the full time interval. We consider $13{,}000$, $14{,}000$, and $15{,}000$ space--time observation locations. Our goal is to reconstruct $(u,v,p)(x,y,t)$ and identify $\mathrm{Re}$ and $\lambda_1$, whose reference values are 100 and 1, respectively.

The encoder consists of a two-layer MLP with a hidden and attention dimension of 128, a latent dimension of 64, four attention heads, one attention layer, and one feedback stage. Residual aware sensor weights are introduced into the attention logits through a $0.5\log\omega_i$ bias. The FNO decoder consists of four residual FNO blocks with a channel width of 96. Each block combines a spectral convolution retaining 24 Fourier modes in
each spatial direction with a parallel $1\times1$ convolution, followed by group normalization and GELU activation. Replicate padding of size 8 is applied before the FNO blocks. The predicted boundary velocities are obtained through bilinear sampling, while the PDE residuals are evaluated using second-order central finite differences in space and time. The initial values are sampled from
$\mathrm{Re}_0\in[30,300]$ and $\lambda_{1,0}\in[0.3,2.0]$. The loss weights are 1 for momentum, 5 for incompressibility, 20 for sparse boundary velocity, and $10^{-4}$ for temporal smoothness. The initial learning rates are $10^{-3}$ for the network and $5\times10^{-4}$ for both $\log\nu$ and $\log\lambda_1$. During the first 500 epochs, all sensor weights are fixed at one; thereafter, they are updated every 200 epochs according to the PDE residuals and used in the attention calculation.

\begin{table*}[!t]
\centering
\caption{Reconstruction and parameter inversion errors for the 2D cylinder wake inverse problem under different numbers of velocity boundary observations and noise levels.}
\label{tab:cylinder2d_uvbc_noise_results}
\resizebox{\textwidth}{!}{
\begin{tabular}{cccccccc}
\toprule
Noise level & \# Velocity obs. & $e_{|\mathbf{u}|}$ & $e_u$ & $e_v$ & $e_p$ & $e_{\mathrm{Re}}$ & $e_{\lambda_1}$ \\
\midrule
\multirow{3}{*}{0.00}
& 13{,}000 & $3.146\times 10^{-2}$ & $2.637\times 10^{-2}$ & $1.156\times 10^{-1}$ & $3.386\times 10^{-1}$ & $3.130\times 10^{-1}$ & $5.907\times 10^{-3}$ \\
& 14{,}000 & $2.677\times 10^{-2}$ & $2.245\times 10^{-2}$ & $9.804\times 10^{-2}$ & $2.329\times 10^{-1}$ & $2.771\times 10^{-1}$ & $5.517\times 10^{-3}$ \\
& 15{,}000 & $5.736\times 10^{-3}$ & $5.498\times 10^{-3}$ & $1.871\times 10^{-2}$ & $1.136\times 10^{-1}$ & $4.917\times 10^{-2}$ & $7.368\times 10^{-4}$ \\
\midrule
\multirow{3}{*}{0.01}
& 13{,}000 & $5.612\times 10^{-2}$ & $4.754\times 10^{-2}$ & $2.086\times 10^{-1}$ & $3.787\times 10^{-1}$ & $4.477\times 10^{-1}$ & $3.417\times 10^{-3}$ \\
& 14{,}000 & $1.580\times 10^{-2}$ & $1.360\times 10^{-2}$ & $5.672\times 10^{-2}$ & $2.448\times 10^{-1}$ & $1.459\times 10^{-1}$ & $4.799\times 10^{-4}$ \\
& 15{,}000 & $3.427\times 10^{-2}$ & $2.885\times 10^{-2}$ & $1.260\times 10^{-1}$ & $2.122\times 10^{-1}$ & $3.176\times 10^{-1}$ & $3.574\times 10^{-3}$ \\
\midrule
\multirow{3}{*}{0.05}
& 13{,}000 & $1.828\times 10^{-1}$ & $1.616\times 10^{-1}$ & $6.003\times 10^{-1}$ & $7.758\times 10^{-1}$ & $1.488\times 10^{-1}$ & $4.746\times 10^{-2}$ \\
& 14{,}000 & $2.082\times 10^{-1}$ & $1.854\times 10^{-1}$ & $6.983\times 10^{-1}$ & $8.519\times 10^{-1}$ & $2.214\times 10^{-1}$ & $8.205\times 10^{-2}$ \\
& 15{,}000 & $1.333\times 10^{-1}$ & $1.168\times 10^{-1}$ & $4.442\times 10^{-1}$ & $6.462\times 10^{-1}$ & $2.339\times 10^{-1}$ & $3.050\times 10^{-2}$ \\
\bottomrule
\end{tabular}
}
\end{table*}

\autoref{tab:cylinder2d_uvbc_noise_results} shows the reconstruction 
errors of the physical fields and the physical parameters for the 
two-dimensional cylinder wake inverse problem, where different numbers 
of boundary velocity observations and noise levels are considered. 
It can be seen that, in the noise-free case, increasing the number of 
observations from 13,000 to 15,000 consistently reduces the velocity, 
pressure, and parameter errors, with 15,000 observations yielding the 
lowest error for every reported quantity. When a mild noise of 1\% is 
added, the errors no longer decrease monotonically with the number of 
observations. The 14,000-observation setting gives the lowest velocity 
and parameter errors, whereas 15,000 observations produce the lowest 
pressure error. At a 5\% noise level, the overall field errors increase 
markedly. Nevertheless, this deterioration is effectively mitigated by 
deploying 15,000 observations, which give the lowest errors for all 
reconstructed fields. The Reynolds-number error, however, remains 
non-monotone and reaches its smallest value with 13,000 observations. 
These results demonstrate that increasing the amount of boundary velocity 
data generally improves the reconstruction of the time-dependent interior 
flow, while the identification of the physical parameters is more sensitive 
to observation noise and the particular sampling pattern. \autoref{fig:cylinder2d_t10_t19_uvp_pred_error_clean} and \autoref{fig:cylinder2d_t10_t19_uvp_pred_error_noise001} shows the  point-wise errors of the reconstruction fields for velocity fields and pressure fields with zero noise and 1\% noise respectively, we choose $t=10,19$ to show the reconstruction results, it is seen that the proposed method is able to reconstruct the field information from sparse observations.

We further compare the proposed method with a vanilla PINN for the 2D cylinder wake inverse problem using 15,000 boundary velocity observations without noise. As shown in \autoref{tab:cylinder2d_pinn_comparison_15000}, the proposed method gives lower errors for the reconstructed velocity and pressure fields, as well as for both identified parameters, $\mathrm{Re}$ and $\lambda_1$. The improvement in identifying $\lambda_1$ is particularly clear. These results indicate that the proposed method recovers the unobserved pressure field and governing parameters more accurately from boundary velocity observations.

\autoref{tab:cylinder2d_ablation_15000_uvbc} presents the ablation study 
results for the two-dimensional cylinder wake test problem with 15,000 
boundary velocity observations and no observation noise. From 
\autoref{tab:cylinder2d_ablation_15000_uvbc}, we can see that the full 
model yields the lowest errors for all reported field metrics and for 
the Reynolds number. The full model also provides the best overall field 
reconstruction performance. Removing either the feedback module or 
residual-based sensor weighting results in a substantial loss of velocity 
and pressure accuracy, as well as Reynolds-number identification accuracy. 
When both modules are removed, all reported errors further increase and 
reach their largest values. These results demonstrate that both the feedback 
module and residual-based sensor weighting are important for propagating 
sparse time-dependent boundary information into the unobserved interior flow.

\begin{table*}[!t]
\centering
\caption{Comparison with vanilla PINN for the 2D cylinder wake inverse problem using 15000 velocity boundary observations without noise.}
\label{tab:cylinder2d_pinn_comparison_15000}
\resizebox{\textwidth}{!}{
\begin{tabular}{lcccccc}
\toprule
Method & $e_{|\mathbf{u}|}$ & $e_u$ & $e_v$ & $e_p$ & $e_{\mathrm{Re}}$ & $e_{\lambda_1}$ \\
\midrule
Vanilla PINN
& $8.630\times 10^{-2}$
& $7.435\times 10^{-2}$
& $2.773\times 10^{-1}$
& $4.692\times 10^{-1}$
& $2.784\times 10^{-1}$
& $3.202\times 10^{-2}$ \\
Ours
& $\mathbf{5.736\times 10^{-3}}$
& $\mathbf{5.498\times 10^{-3}}$
& $\mathbf{1.871\times 10^{-2}}$
& $\mathbf{1.136\times 10^{-1}}$
& $\mathbf{4.917\times 10^{-2}}$
& $\mathbf{7.368\times 10^{-4}}$ \\
\bottomrule
\end{tabular}
}
\end{table*}

\begin{table*}[!t]
\centering
\caption{Ablation study for the 2D cylinder wake inverse problem with 15000 velocity boundary observations and no observation noise.}
\label{tab:cylinder2d_ablation_15000_uvbc}
\resizebox{\textwidth}{!}{
\begin{tabular}{cccccccc}
\toprule
Case & Feedback layers & Sensor weighting &
$e_{|\mathbf{u}|}$ & $e_u$ & $e_v$ & $e_p$ & $e_{\mathrm{Re}}$ \\
\midrule
Full model
& 1
& Residual
& $5.736\times 10^{-3}$
& $5.498\times 10^{-3}$
& $1.871\times 10^{-2}$
& $1.136\times 10^{-1}$
& $4.917\times 10^{-2}$\\

No sensor weighting
& 1
& None
& $4.360\times 10^{-2}$
& $3.643\times 10^{-2}$
& $1.608\times 10^{-1}$
& $3.644\times 10^{-1}$
& $3.994\times 10^{-1}$\\

No feedback
& 0
& Residual
& $3.823\times 10^{-2}$
& $3.204\times 10^{-2}$
& $1.405\times 10^{-1}$
& $2.271\times 10^{-1}$
& $3.561\times 10^{-1}$\\

No feedback, no sensor weighting
& 0
& None
& $6.422\times 10^{-2}$
& $5.326\times 10^{-2}$
& $2.381\times 10^{-1}$
& $3.725\times 10^{-1}$
& $5.005\times 10^{-1}$\\
\bottomrule
\end{tabular}
}
\end{table*}

\begin{figure}[H]
    \centering
    \includegraphics[width=0.9\textwidth]{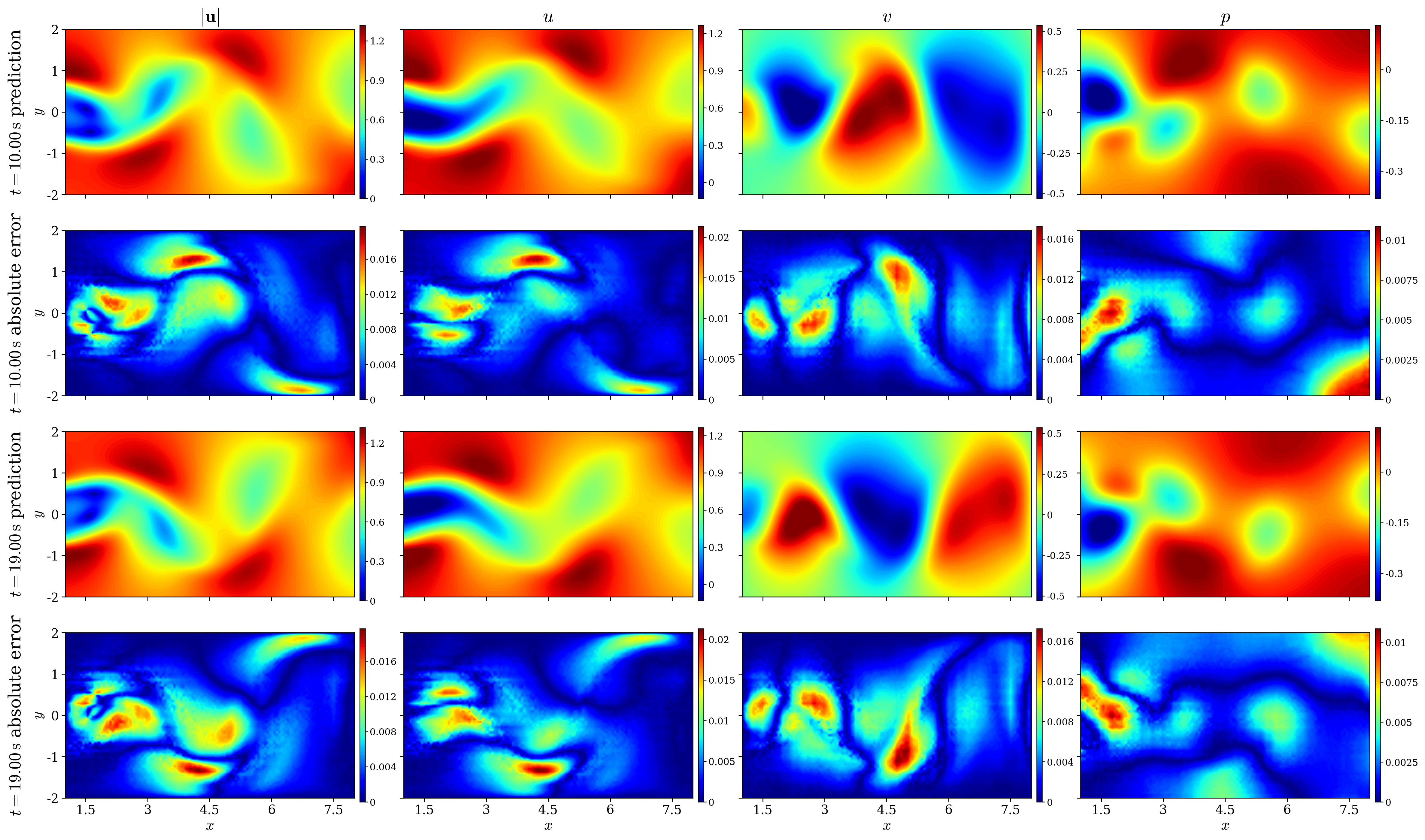}
    \caption{Predicted fields and absolute errors for the 2D cylinder wake inverse problem with clean boundary velocity data. The columns show $|\mathbf{u}|$, $u$, $v$, and mean-free pressure $p$. The first two rows give the results at $t=10$. The last two rows give the results at $t=19$.}
    \label{fig:cylinder2d_t10_t19_uvp_pred_error_clean}
\end{figure}

\begin{figure}[H]
    \centering
    \includegraphics[width=0.9\textwidth]{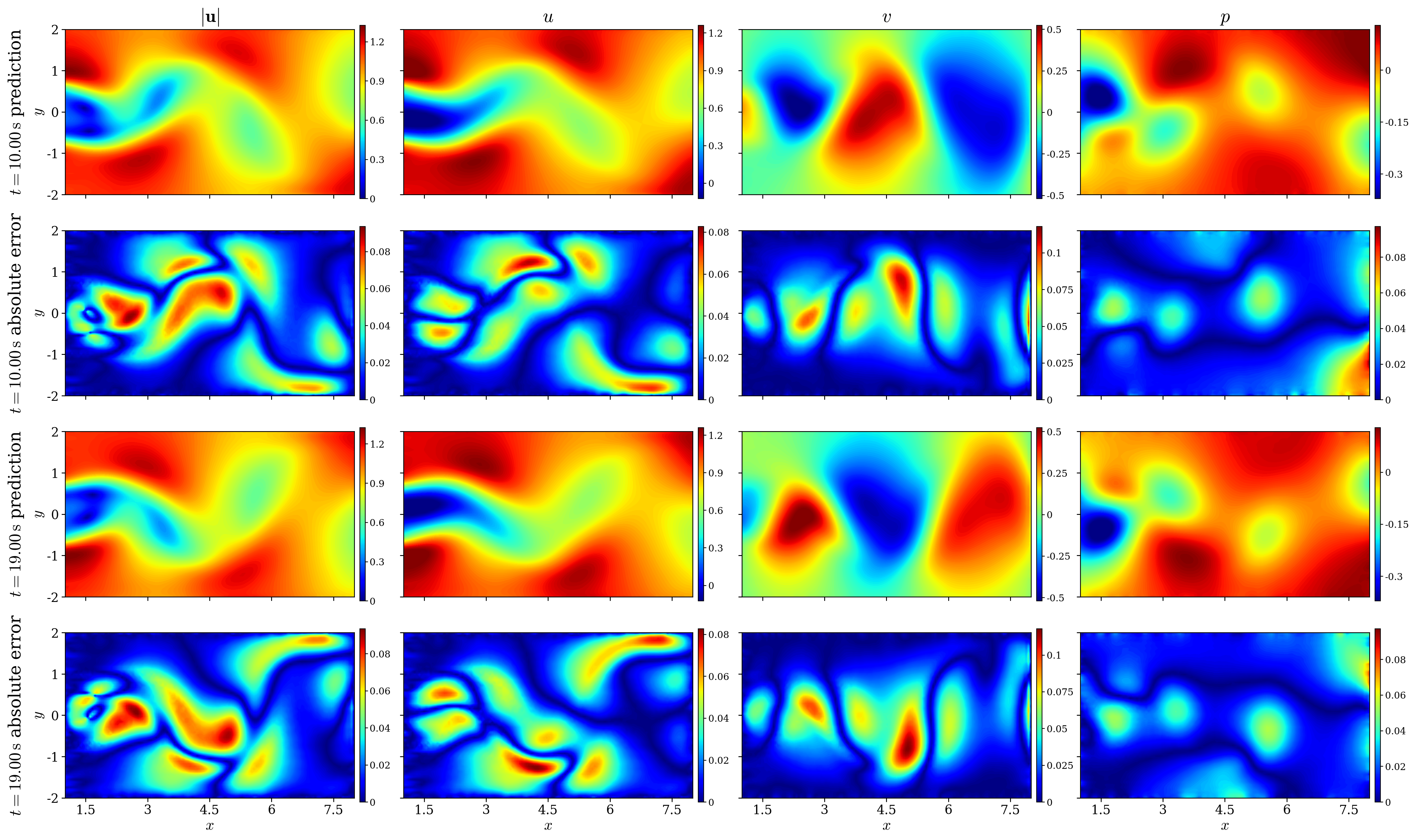}
    \caption{Predicted fields and absolute errors for the 2D cylinder wake inverse problem with $1\%$ noisy boundary velocity data. The columns show $|\mathbf{u}|$, $u$, $v$, and mean-free pressure $p$. The first two rows give the results at $t=10$. The last two rows give the results at $t=19$.}
    \label{fig:cylinder2d_t10_t19_uvp_pred_error_noise001}
\end{figure}

\begin{figure}[H]
    \centering

    \begin{subfigure}[t]{0.48\textwidth}
        \centering
        \includegraphics[width=\linewidth]{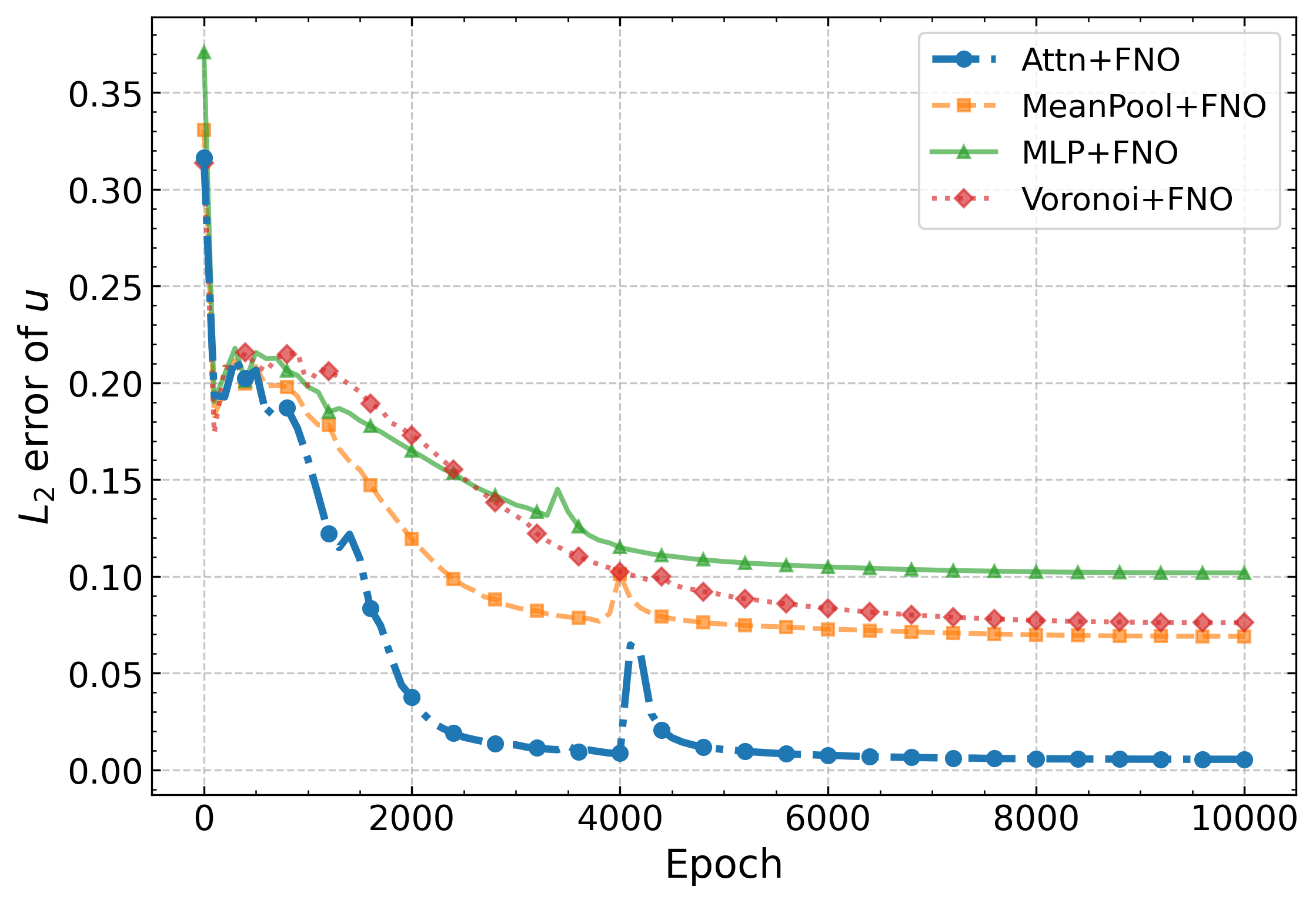}
        \caption{$u$}
        \label{fig:cylinder2d_l2_u_encoder_compare}
    \end{subfigure}
    \hfill
    \begin{subfigure}[t]{0.48\textwidth}
        \centering
        \includegraphics[width=\linewidth]{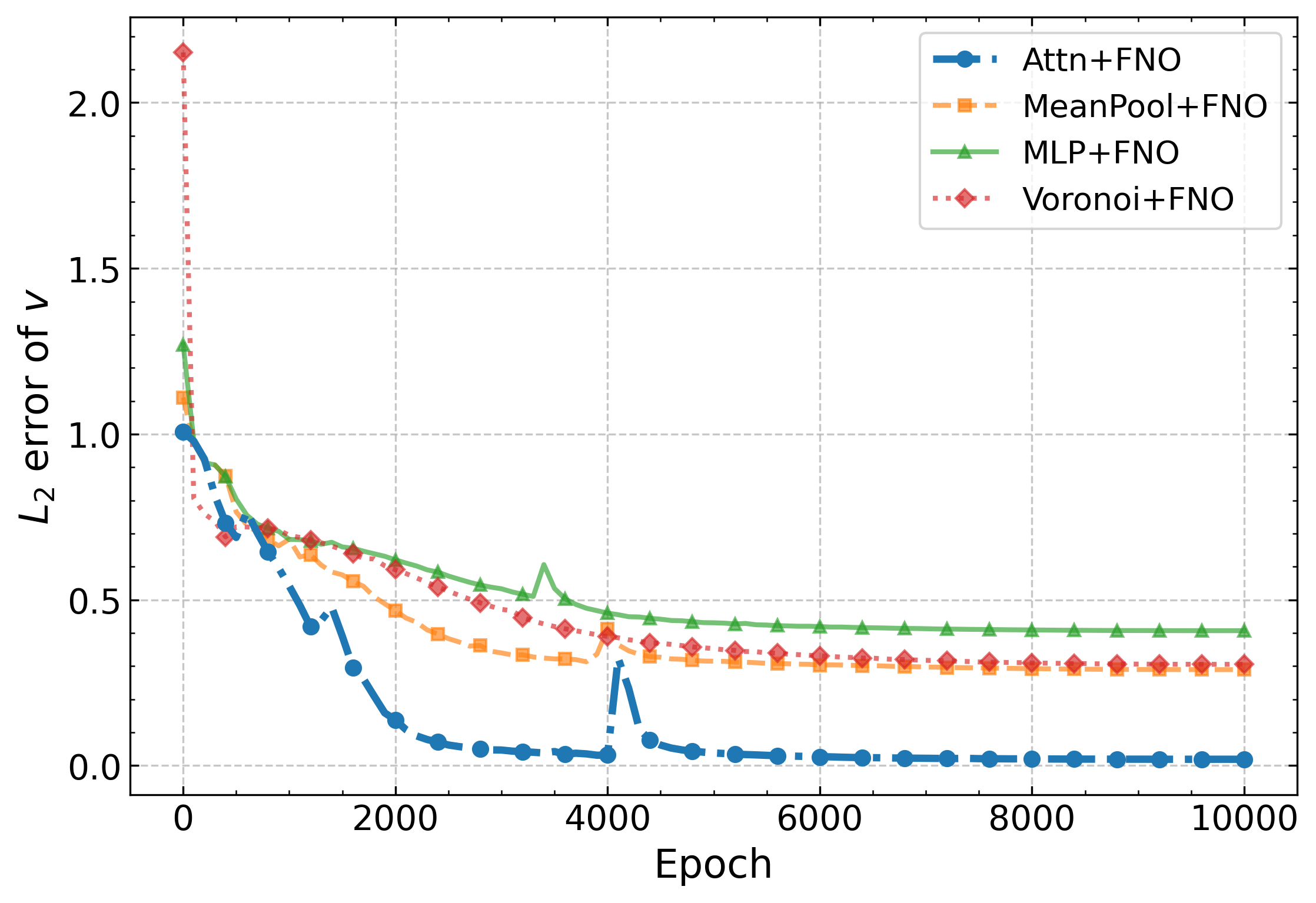}
        \caption{$v$}
        \label{fig:cylinder2d_l2_v_encoder_compare}
    \end{subfigure}

    \vspace{0.2cm}

    \begin{subfigure}[t]{0.48\textwidth}
        \centering
        \includegraphics[width=\linewidth]{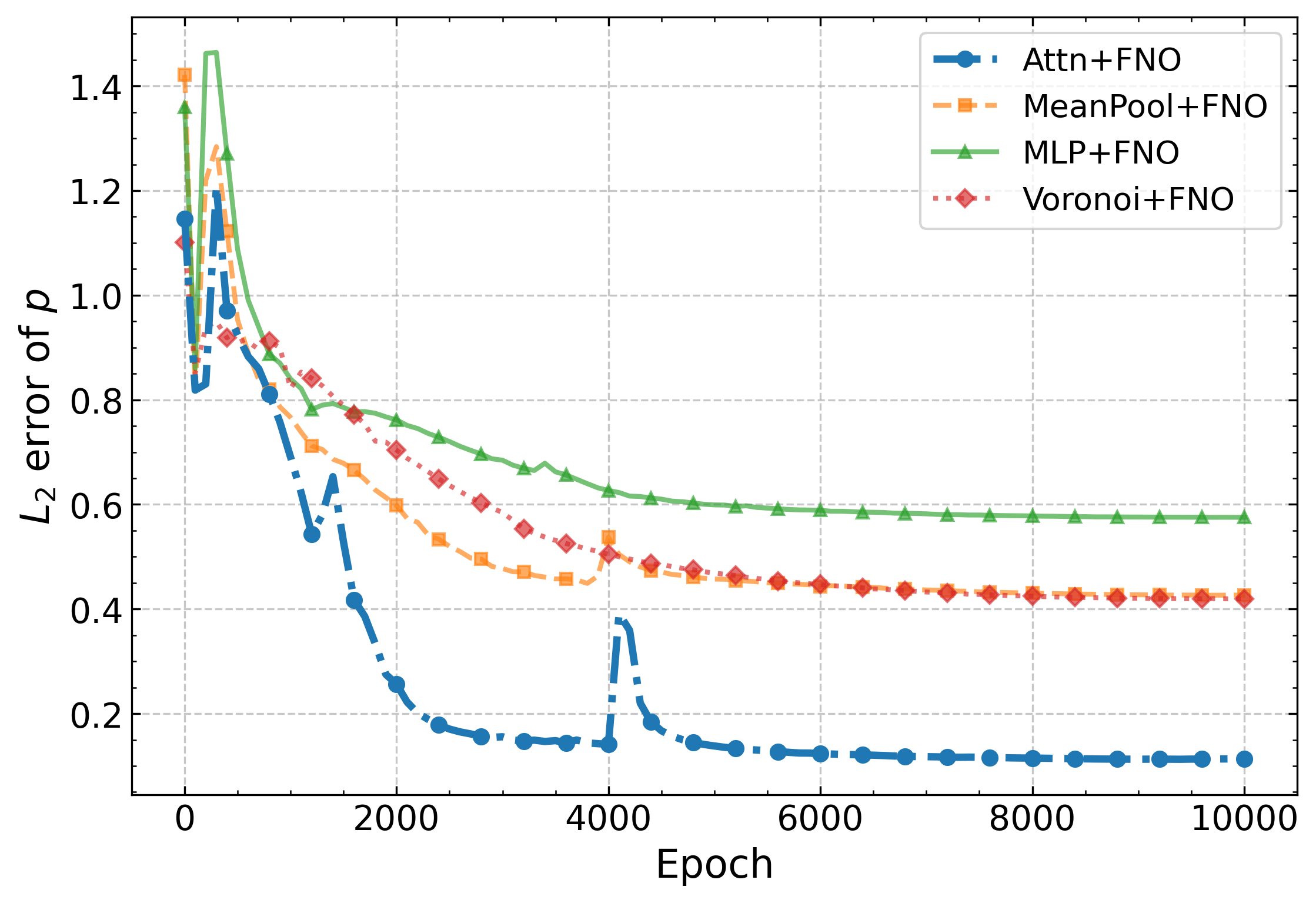}
        \caption{$p$}
        \label{fig:cylinder2d_l2_p_encoder_compare}
    \end{subfigure}
    \hfill
    \begin{subfigure}[t]{0.48\textwidth}
        \centering
        \includegraphics[width=\linewidth]{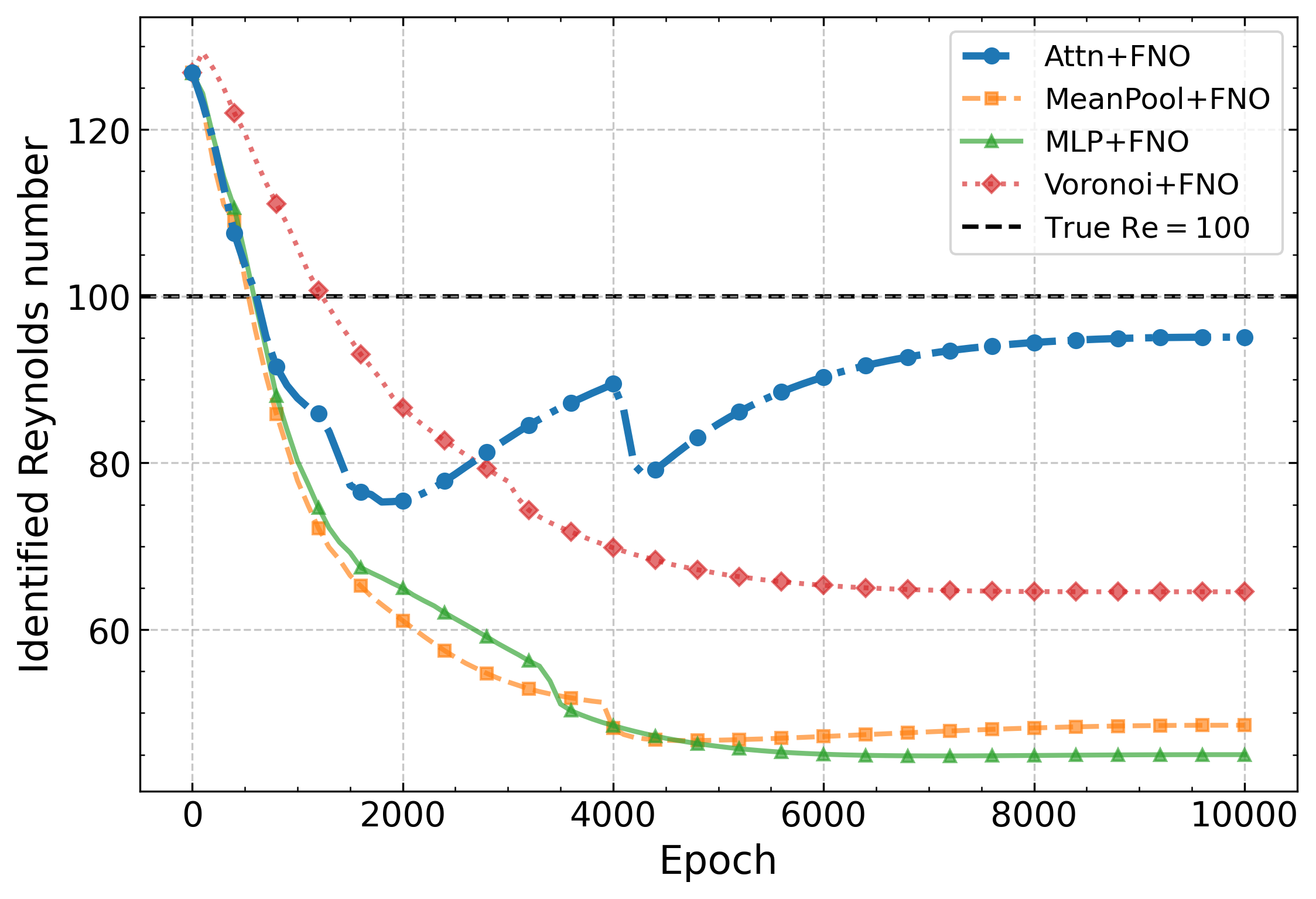}
        \caption{Identified $\mathrm{Re}$}
        \label{fig:cylinder2d_re_encoder_compare}
    \end{subfigure}

    \caption{Training histories for the 2D cylinder wake inverse problem.
    The first three panels show the relative $L_2$ errors of $u$, $v$,
    and the pressure $p$, while the last panel shows the
    identified Reynolds number. The compared methods are Attn+FNO,
    MeanPool+FNO, MLP+FNO, and Voronoi+FNO.}
    \label{fig:cylinder2d_l2_encoder_compare}
\end{figure}

\subsection{2D Non-ideal MHD Inverse Problem}
\label{subsec:mhd_inverse}

Finally, we consider a two-dimensional inverse problem for compressible non-ideal
magnetohydrodynamics (MHD) based on the Orszag--Tang vortex. The computational
domain is \(\Omega=[0,2\pi]\times[0,2\pi]\) with \(t\in[0,3.99]\), and periodic
boundary conditions in both spatial directions. The reference data were
generated from non-ideal MHD simulations performed using the PLUTO code
\cite{mignone2007pluto} on a uniform \(256\times256\) computational grid, with
400 temporal snapshots retained for the present study. The governing equations
with explicit viscosity and resistivity are
\begin{equation*}
\left\{
\begin{aligned}
    &\rho_t+\nabla\cdot(\rho\mathbf{v})=0,\\
    &(\rho\mathbf{v})_t
    +\nabla\cdot\Bigl[
        \rho\mathbf{v}\mathbf{v}
        +\bigl(p+\tfrac12|\mathbf{B}|^2\bigr)\mathbf{I}
        -\mathbf{B}\mathbf{B}-\boldsymbol{\tau}
    \Bigr]=0,\\
    &\mathbf{B}_t-\nabla\times(\mathbf{v}\times\mathbf{B})
    +\nabla\times(\eta\mathbf{J})=0,\\
    &E_t+\nabla\cdot\Bigl[
        \bigl(E+p+\tfrac12|\mathbf{B}|^2\bigr)\mathbf{v}
        -(\mathbf{v}\cdot\mathbf{B})\mathbf{B}\\
    &\qquad\qquad\quad
        -\boldsymbol{\tau}\mathbf{v}
        -\eta\mathbf{J}\times\mathbf{B}
    \Bigr]=0,\\
    &\nabla\cdot\mathbf{B}=0.
\end{aligned}
\right.
\label{eq:nonideal_mhd_system}
\end{equation*}
Here \(\mathbf{v}=(v_x,v_y,0)^{\mathrm{T}}\) and
\(\mathbf{B}=(B_x,B_y,0)^{\mathrm{T}}\). The current density, total energy
density, and viscous stress tensor are defined by
\begin{equation*}
\left\{
\begin{aligned}
    \mathbf{J}
    &=\nabla\times\mathbf{B}=(0,0,j_z)^{\mathrm{T}},
    \qquad
    j_z=\frac{\partial B_y}{\partial x}
        -\frac{\partial B_x}{\partial y},\\
    E
    &=\frac{p}{\gamma-1}
      +\frac12\rho|\mathbf{v}|^2
      +\frac12|\mathbf{B}|^2,\\
    \boldsymbol{\tau}
    &=\nu\Bigl[
        \nabla\mathbf{v}
        +(\nabla\mathbf{v})^{\mathrm{T}}
        -\frac23(\nabla\cdot\mathbf{v})\mathbf{I}
    \Bigr].
\end{aligned}
\right.
\end{equation*}
where \(\gamma=5/3\), \(\nu\) is the viscosity, and \(\eta\) is the magnetic
diffusivity. The true parameter values in the reference simulation are \(\nu=\eta=10^{-3}\). The initial condition is the standard Orszag--Tang vortex~\cite{toth2000constraint}:
\begin{equation*}
\left\{
\begin{aligned}
    \mathbf{v}(x,y,0)
    &=(-\sin y,\ \sin x,\ 0)^{\mathrm{T}},\\
    \mathbf{B}(x,y,0)
    &=(-\sin y,\ \sin 2x,\ 0)^{\mathrm{T}},\\
    \rho(x,y,0)
    &=\frac{25}{9},
    \qquad p(x,y,0)=\frac{5}{3}.
\end{aligned}
\right.
\end{equation*}
For \(\mathbf{q}=(\rho,v_x,v_y,B_x,B_y,p)\), the periodic boundary conditions are
\begin{equation*}
    \mathbf{q}(0,y,t)=\mathbf{q}(2\pi,y,t),
    \qquad
    \mathbf{q}(x,0,t)=\mathbf{q}(x,2\pi,t).
\end{equation*}

The model receives sparse magnetic-field observations $(B_x,B_y)$ and, in the pressure-assisted settings, additional boundary or interior pressure observations. The task is to reconstruct $(\rho,v_x,v_y,B_x,B_y,p)$, evaluate the current density as $j_z=\partial_xB_y-\partial_yB_x$, and identify $\nu$ and $\eta$. The magnetic observations are sampled globally from the interior space--time pool, with $N_B\in\{5{,}000,10{,}000,15{,}000\}$, while the pressure-assisted cases use $N_p=10{,}000$. The attention encoder is conditioned only on the magnetic observations, whereas the pressure observations enter through a separate sparse observation loss.

The encoder consists of a two-layer sensor MLP with a hidden width and output dimension of $64$, a four-head cross-attention layer, two bidirectional sensor--grid feedback stages, and an output latent dimension of $32$. Residual-aware weights are used in the sparse observation losses, while the magnetic-observation weights additionally enter the attention logits through a $0.5\log\omega_i$ bias. The gated latent map is decoded by an FNO consisting of four layers with a channel width of $96$. Each layer combines a spectral convolution retaining $32\times32$ Fourier modes with a parallel $1\times1$ convolution, followed by GELU activation. Both $\nu$ and $\eta$ are initialized at $10^{-4}$ and optimized through their logarithms without parameter priors. The loss combines the sparse observation losses, the continuity, momentum, induction, and energy residuals, the divergence-free constraint, and the prescribed initial conditions for $(B_x,B_y,\rho,p)$. Spatial periodicity is incorporated through sine--cosine coordinate features and Fourier-based spatial differentiation.

Training uses all $400$ frames through randomly sampled windows of five consecutive frames. During the first $3{,}000$ epochs, four joint updates are performed per epoch. From epochs $3{,}001$ to $6{,}500$, each epoch contains four field updates, one parameter update, and one joint update, followed by four joint updates per epoch during the final $500$ epochs. Muon-based optimizers are employed~\cite{jordan2024muon,liu2025muon}. During the first $500$ epochs of sensor-weight warm-up, all observation weights are fixed at one. They are first updated from the current PDE residuals at epoch $600$ and are subsequently updated every $200$ epochs.

\begin{table}[H]
\centering
\caption{Reconstruction and parameter inversion errors for the non-ideal MHD inverse problem under different magnetic-field observations, pressure-observation settings, and noise levels. Here $N_b$ and $N_p$ denote the numbers of magnetic-field and pressure observations. ``None'', ``boundary'', and ``Internal'' denote no pressure observation, boundary pressure observation, and internal pressure observation.}
\label{tab:mhd_pressure_obs_noise_results}
\scriptsize
\setlength{\tabcolsep}{4.2pt}
\resizebox{\textwidth}{!}{
\begin{tabular}{cccccccccccc}
\toprule
Noise level
& Pressure obs.
& $N_b$
& $N_p$
& $e_u$
& $e_v$
& $e_{b_x}$
& $e_{b_y}$
& $e_{\rho}$
& $e_p$
& $e_{\nu}$
& $e_{\eta}$ \\
\midrule
\multirow{5}{*}{0.00}
& Internal & 5000  & 10000
& $1.205\times 10^{-1}$
& $1.213\times 10^{-1}$
& $5.354\times 10^{-2}$
& $5.549\times 10^{-2}$
& $7.475\times 10^{-2}$
& $1.767\times 10^{-1}$
& $7.786\times 10^{-2}$
& $8.519\times 10^{-2}$ \\

& Internal & 10000 & 10000
& $1.166\times 10^{-1}$
& $1.153\times 10^{-1}$
& $4.535\times 10^{-2}$
& $4.764\times 10^{-2}$
& $7.225\times 10^{-2}$
& $1.750\times 10^{-1}$
& $8.348\times 10^{-2}$
& $8.668\times 10^{-2}$ \\

& None     & 15{,}000 & --
& $2.349\times 10^{-1}$
& $2.468\times 10^{-1}$
& $4.762\times 10^{-2}$
& $5.226\times 10^{-2}$
& $1.051\times 10^{-1}$
& $4.448\times 10^{-1}$
& $5.560\times 10^{-2}$
& $7.423\times 10^{-2}$ \\

& boundary     & 15{,}000 & 10000
& $1.694\times 10^{-1}$
& $1.708\times 10^{-1}$
& $4.503\times 10^{-2}$
& $4.879\times 10^{-2}$
& $9.099\times 10^{-2}$
& $2.850\times 10^{-1}$
& $3.121\times 10^{-2}$
& $5.600\times 10^{-2}$ \\

& Internal & 15{,}000 & 10000
& $1.197\times 10^{-1}$
& $1.206\times 10^{-1}$
& $4.017\times 10^{-2}$
& $4.286\times 10^{-2}$
& $7.023\times 10^{-2}$
& $1.779\times 10^{-1}$
& $1.425\times 10^{-1}$
& $1.358\times 10^{-1}$ \\
\midrule
\multirow{3}{*}{0.01}
& Internal & 5000  & 10000
& $1.197\times 10^{-1}$
& $1.202\times 10^{-1}$
& $5.408\times 10^{-2}$
& $5.655\times 10^{-2}$
& $7.502\times 10^{-2}$
& $1.777\times 10^{-1}$
& $7.838\times 10^{-2}$
& $8.610\times 10^{-2}$ \\

& Internal & 10000 & 10000
& $1.175\times 10^{-1}$
& $1.179\times 10^{-1}$
& $4.533\times 10^{-2}$
& $4.801\times 10^{-2}$
& $7.193\times 10^{-2}$
& $1.755\times 10^{-1}$
& $8.283\times 10^{-2}$
& $8.641\times 10^{-2}$ \\

& Internal & 15{,}000 & 10000
& $1.170\times 10^{-1}$
& $1.180\times 10^{-1}$
& $4.028\times 10^{-2}$
& $4.321\times 10^{-2}$
& $7.284\times 10^{-2}$
& $1.732\times 10^{-1}$
& $8.188\times 10^{-2}$
& $8.799\times 10^{-2}$ \\
\midrule
\multirow{3}{*}{0.05}
& Internal & 5000  & 10000
& $1.243\times 10^{-1}$
& $1.286\times 10^{-1}$
& $6.198\times 10^{-2}$
& $6.388\times 10^{-2}$
& $7.714\times 10^{-2}$
& $1.839\times 10^{-1}$
& $8.006\times 10^{-2}$
& $8.493\times 10^{-2}$ \\

& Internal & 10000 & 10000
& $1.219\times 10^{-1}$
& $1.237\times 10^{-1}$
& $5.306\times 10^{-2}$
& $5.566\times 10^{-2}$
& $7.271\times 10^{-2}$
& $1.800\times 10^{-1}$
& $8.028\times 10^{-2}$
& $8.497\times 10^{-2}$ \\

& Internal & 15{,}000 & 10000
& $1.199\times 10^{-1}$
& $1.249\times 10^{-1}$
& $4.910\times 10^{-2}$
& $5.245\times 10^{-2}$
& $7.408\times 10^{-2}$
& $1.778\times 10^{-1}$
& $8.097\times 10^{-2}$
& $8.928\times 10^{-2}$ \\
\bottomrule
\end{tabular}
}
\end{table}

\autoref{tab:mhd_pressure_obs_noise_results} shows the reconstruction 
errors of the physical fields and physical parameters for the 
two-dimensional non-ideal MHD inverse problem, where different numbers  of magnetic-field observations, pressure-observation settings, and noise  levels are considered. It can be seen that, in the noise-free case,  pressure information is important for recovering the hydrodynamic  variables. With 15,000 magnetic-field observations, omitting pressure  data substantially increases the velocity, density, and pressure errors.  Among the three noise-free cases with 15,000 magnetic-field observations,  interior pressure observations yield the lowest errors for all reconstructed 
fields, whereas boundary pressure observations give the smallest errors for  $\nu$ and $\eta$. When a mild noise of 1\% is added, the overall field  errors remain close to their noise-free scenarios, indicating that the  proposed framework is largely insensitive to low-level perturbations. At a 5\% noise level, the field errors increase only moderately, while  increasing the number of magnetic-field observations generally improves  the magnetic-field reconstruction. The errors for $\nu$ and $\eta$,  however, do not exhibit a strictly monotonic decrease. These results  demonstrate that the proposed method achieves robust field reconstruction  under moderate noise and that pressure observations are important for  recovering the hydrodynamic variables, while physical-parameter identification remains more sensitive to the observation configuration. \autoref{fig:mhd_frame0050_0200_0399_uvbxbyprho_pred_error_clean} and \autoref{fig:mhd_frame0050_0200_0399_uvbxbyprho_pred_error_noise001} show the reconstructed fields and the point-wise  absolute errors at frames 50, 200, and 399 using clean and 1\% noisy  magnetic-field and interior pressure observations, respectively. It is  seen that the proposed method is able to reconstruct the coupled velocity,  magnetic, pressure, and density fields from sparse observations. The main  spatial structures and their temporal evolution are recovered at all three 
frames. Although the errors become more localized around sharp gradients  and fine-scale structures at later times, the overall reconstructed fields remain consistent with the reference solutions under observation noise.

\autoref{tab:mhd_ablation_latest} presents the ablation study results 
for the two-dimensional non-ideal MHD test problem with 15,000 
magnetic-field observations, 10,000 internal pressure observations, 
and no observation noise. From \autoref{tab:mhd_ablation_latest}, we 
can see that no single configuration yields the lowest errors for all 
reported metrics. The full model attains the lowest reported $e_{b_y}$ 
and $e_\rho$, whereas the case without the feedback layer gives the 
smallest velocity and pressure errors. Removing sensor weighting yields 
the lowest $e_{b_x}$, $e_\nu$, and $e_\eta$. Reducing the feedback path 
can improve the velocity and pressure reconstruction but results in a 
slight loss of magnetic-field accuracy. Removing sensor weighting improves the parameter estimates but does not produce a uniform gain across the field metrics. These results indicate that, in this case, the feedback module and residual-based sensor weighting mainly balance the accuracy of magnetic-field reconstruction, hydrodynamic-field reconstruction, and physical-parameter identification. We further compare the proposed method with a vanilla PINN for the non-ideal MHD inverse problem using 15,000 magnetic-field observations without pressure observations or noise. As shown in \autoref{tab:mhd_vanilla_pinn_comparison}, the proposed method achieves lower errors for all reconstructed fields and both identified parameters, \(\nu\) and \(\eta\). These results indicate that the proposed method recovers the unobserved hydrodynamic fields and governing parameters more accurately from sparse magnetic-field observations.

\begin{table}[H]
\centering
\caption{Ablation study for the 2D non-ideal MHD inverse problem with 15000 internal magnetic-field observations and 10000 internal pressure observations.}
\label{tab:mhd_ablation_latest}
\resizebox{\textwidth}{!}{%
\begin{tabular}{ccccccccccc}
\toprule
Case & Feedback layers & Sensor weighting
& $e_u$ & $e_v$ & $e_{b_x}$ & $e_{b_y}$ & $e_\rho$ & $e_p$ & $e_\nu$ & $e_\eta$ \\
\midrule
Full model
& 2
& Residual
& $1.197\times 10^{-1}$
& $1.206\times 10^{-1}$
& $4.017\times 10^{-2}$
& $4.286\times 10^{-2}$
& $7.023\times 10^{-2}$
& $1.779\times 10^{-1}$
& $1.425\times 10^{-1}$
& $1.358\times 10^{-1}$ \\

No sensor weighting
& 2
& None
& $1.199\times 10^{-1}$
& $1.215\times 10^{-1}$
& $4.012\times 10^{-2}$
& $4.351\times 10^{-2}$
& $7.208\times 10^{-2}$
& $1.781\times 10^{-1}$
& $8.176\times 10^{-2}$
& $8.787\times 10^{-2}$ \\

No feedback layer
& 1
& Residual
& $1.154\times 10^{-1}$
& $1.174\times 10^{-1}$
& $4.115\times 10^{-2}$
& $4.385\times 10^{-2}$
& $7.057\times 10^{-2}$
& $1.717\times 10^{-1}$
& $1.425\times 10^{-1}$
& $1.358\times 10^{-1}$ \\

No feedback layer, no sensor weighting
& 1
& None
& $1.214\times 10^{-1}$
& $1.218\times 10^{-1}$
& $4.240\times 10^{-2}$
& $4.482\times 10^{-2}$
& $7.051\times 10^{-2}$
& $1.812\times 10^{-1}$
& $4.927\times 10^{-1}$
& $4.947\times 10^{-1}$ \\
\bottomrule
\end{tabular}%
}
\end{table}

\begin{table}[htbp]
    \centering
    \caption{Comparison with vanilla PINN for the non-ideal MHD inverse problem using 15000 magnetic-field observation locations without pressure observations or noise.}
    \label{tab:mhd_vanilla_pinn_comparison}
    \renewcommand{\arraystretch}{1.15}
    \setlength{\tabcolsep}{6pt}
    \resizebox{\linewidth}{!}{%
    \begin{tabular}{@{}l*{8}{c}@{}}
        \toprule
        Method & $e_u$ & $e_v$ & $e_{b_x}$ & $e_{b_y}$ & $e_{\rho}$ & $e_p$ & $e_{\nu}$ & $e_{\eta}$ \\
        \midrule
        Vanilla PINN & $3.378\times 10^{-1}$ & $3.142\times 10^{-1}$ & $8.348\times 10^{-2}$ & $9.135\times 10^{-2}$ & $4.392\times 10^{-1}$ & $7.149\times 10^{-1}$ & $5.244\times 10^{-1}$ & $1.527\times 10^{-1}$ \\
        Ours & $\mathbf{2.349\times 10^{-1}}$ & $\mathbf{2.468\times 10^{-1}}$ & $\mathbf{4.762\times 10^{-2}}$ & $\mathbf{5.226\times 10^{-2}}$ & 
        $\mathbf{1.051\times 10^{-1}}$ & $\mathbf{4.448\times 10^{-1}}$ & 
        $\mathbf{5.560\times 10^{-2}}$ & $\mathbf{7.423\times 10^{-2}}$ \\
        \bottomrule
    \end{tabular}%
    }
    \vspace{4pt}
    \begin{minipage}{\linewidth}
    \end{minipage}
\end{table}

\begin{figure}[H]
    \centering
    \includegraphics[width=1.0\textwidth]{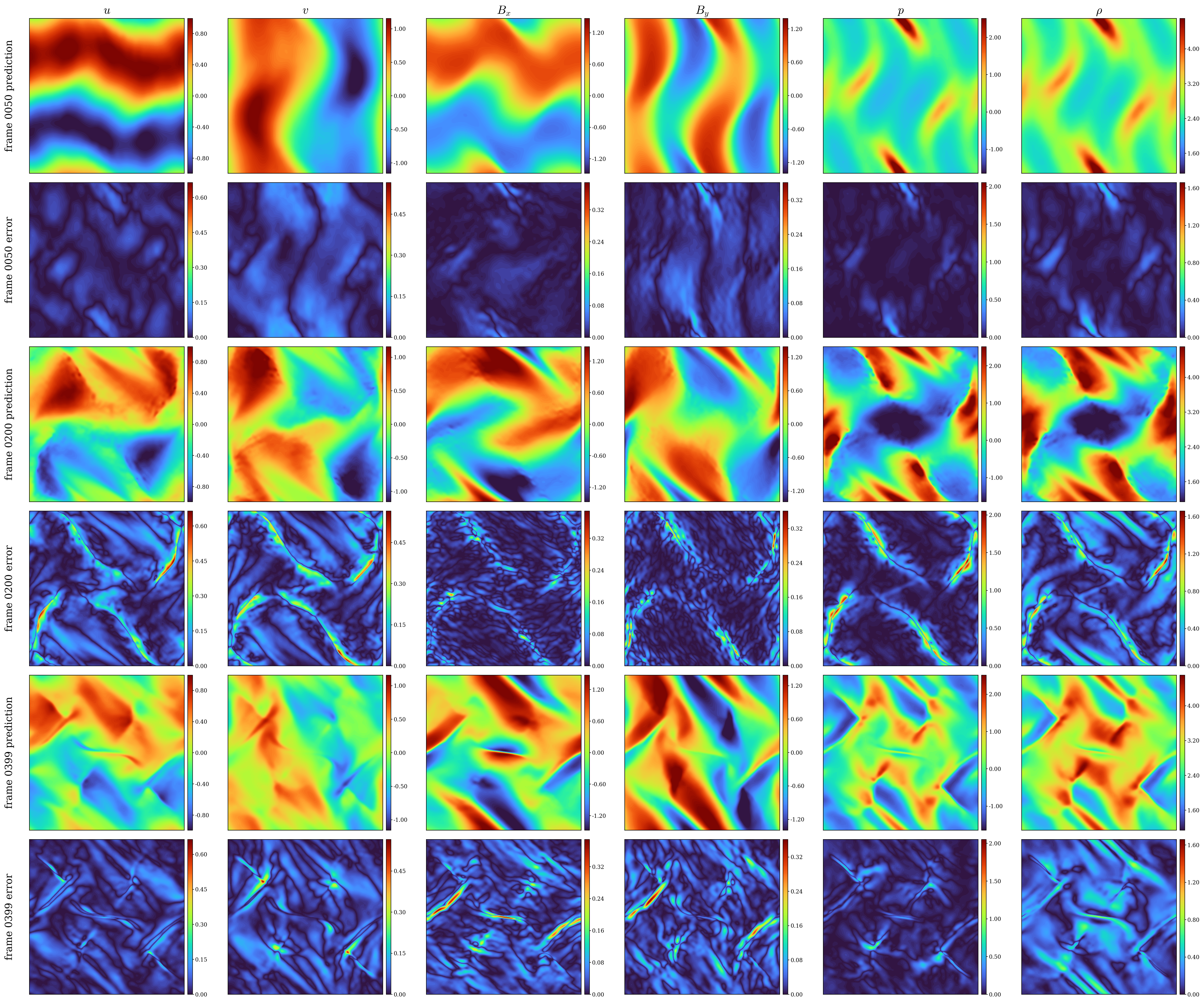}
    \caption{Predicted fields and absolute errors for the 2D non-ideal MHD inverse problem with clean magnetic-field and internal pressure data. The columns show $u$, $v$, $B_x$, $B_y$, mean-free pressure $p$, and density $\rho$. The row pairs correspond to frames 50, 200, and 399.}
    \label{fig:mhd_frame0050_0200_0399_uvbxbyprho_pred_error_clean}
\end{figure}

\begin{figure}[H]
    \centering
    \includegraphics[width=1.0\textwidth]{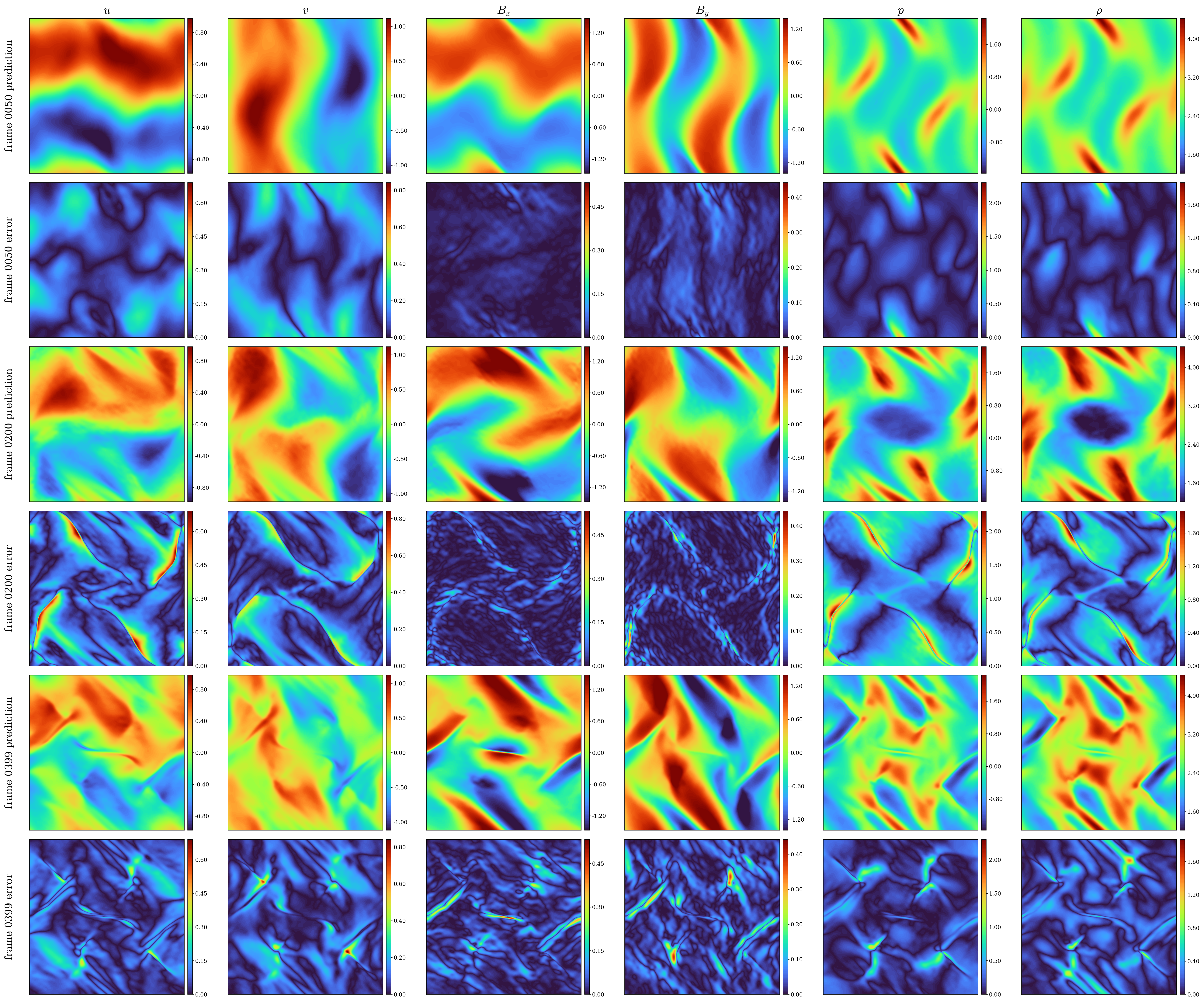}
    \caption{Predicted fields and absolute errors for the 2D non-ideal MHD inverse problem with $1\%$ noisy magnetic-field and internal pressure data. The columns show $u$, $v$, $B_x$, $B_y$, mean-free pressure $p$, and density $\rho$. The row pairs correspond to frames 50, 200, and 399.}
    \label{fig:mhd_frame0050_0200_0399_uvbxbyprho_pred_error_noise001}
\end{figure}

\section{Conclusion}
\label{sec:conclusion}
In this work, we developed a physics-constrained attentive neural operator framework to jointly reconstruct fields and identify parameters in coupled physical systems from incomplete observations. The model represents sparse measurements as sensor tokens, maps them onto a regular grid through cross-attention, and decodes the full fields with an FNO, thereby recovering both observed and unobserved variables without full-field supervision.

We also established a conditional stability estimate that bounds the field and parameter errors under local observability conditions. Numerical experiments on lid-driven cavity flows, a cylinder wake, and a non-ideal MHD problem show that the framework can recover hidden fields and identifies physical parameters from sparse single-field observations. The results also indicate that parameter identification is generally more sensitive to noise and sensor configuration than state reconstruction.

Several limitations remain. We verify local observability numerically for the tested configurations rather than proving it for general systems, and we optimize each case separately, requiring retraining for every new configuration. Future work will pursue generalizable, pre-trained inverse operators across varying Reynolds numbers and boundary conditions. Other directions include adaptive sensor placement, uncertainty quantification, and large-scale three-dimensional time-dependent systems.

\section*{Acknowledgements}
K. Tang is supported by NSFC grant 12501598. X. Sun is supported by NSFC grants 12622131, 12525210, and 92570102. W. Li is supported by the Shanghai Pujiang Program (25PJD058) and NSFC grant 62602407.

\clearpage
\appendix

\section{Multi Seeds Results}
\label{app:Multi_Seeds_Results}

\begin{table}[H]
\centering
\caption{Random-seed evaluation for the 2D lid-driven cavity inverse problem at $\mathrm{Re}=3200$. All settings are reported as mean $\pm$ standard deviation over five runs.}
\label{tab:ldc2d_re3200_random_seed_results}
\resizebox{\textwidth}{!}{
\begin{tabular}{ccccccc}
\toprule
\# Sensors & Noise level & $e_{|\mathbf{u}|}$ & $e_u$ & $e_v$ & $e_p$ & $e_{\mathrm{Re}}$ \\
\midrule
100 & 0.00
& $(8.037 \pm 2.187)\times 10^{-2}$
& $(7.558 \pm 2.184)\times 10^{-2}$
& $(8.948 \pm 2.196)\times 10^{-2}$
& $(2.598 \pm 0.193)\times 10^{-1}$
& $(1.239 \pm 0.557)\times 10^{-1}$ \\
200 & 0.00
& $(5.136 \pm 0.384)\times 10^{-2}$
& $(5.197 \pm 0.388)\times 10^{-2}$
& $(5.669 \pm 0.360)\times 10^{-2}$
& $(1.479 \pm 0.794)\times 10^{-1}$
& $(6.263 \pm 3.650)\times 10^{-2}$ \\
300 & 0.00
& $(4.799 \pm 1.747)\times 10^{-2}$
& $(4.487 \pm 1.774)\times 10^{-2}$
& $(5.607 \pm 1.946)\times 10^{-2}$
& $(1.145 \pm 0.682)\times 10^{-1}$
& $(1.444 \pm 0.823)\times 10^{-1}$ \\
\midrule
100 & 0.01
& $(5.718 \pm 1.010)\times 10^{-2}$
& $(5.379 \pm 0.834)\times 10^{-2}$
& $(6.511 \pm 1.173)\times 10^{-2}$
& $(2.176 \pm 0.685)\times 10^{-1}$
& $(9.281 \pm 4.544)\times 10^{-2}$ \\
200 & 0.01
& $(5.017 \pm 0.244)\times 10^{-2}$
& $(5.138 \pm 0.227)\times 10^{-2}$
& $(5.452 \pm 0.237)\times 10^{-2}$
& $(1.083 \pm 0.034)\times 10^{-1}$
& $(7.717 \pm 1.085)\times 10^{-2}$ \\
300 & 0.01
& $(4.120 \pm 0.165)\times 10^{-2}$
& $(3.806 \pm 0.146)\times 10^{-2}$
& $(4.830 \pm 0.189)\times 10^{-2}$
& $(9.036 \pm 0.944)\times 10^{-2}$
& $(1.155 \pm 0.081)\times 10^{-1}$ \\
\midrule
100 & 0.05
& $(9.713 \pm 1.256)\times 10^{-2}$
& $(9.174 \pm 1.281)\times 10^{-2}$
& $(1.075 \pm 0.129)\times 10^{-1}$
& $(2.420 \pm 0.276)\times 10^{-1}$
& $(2.062 \pm 0.399)\times 10^{-1}$ \\
200 & 0.05
& $(5.694 \pm 1.094)\times 10^{-2}$
& $(5.532 \pm 1.049)\times 10^{-2}$
& $(6.411 \pm 1.129)\times 10^{-2}$
& $(1.200 \pm 0.214)\times 10^{-1}$
& $(5.054 \pm 4.654)\times 10^{-2}$ \\
300 & 0.05
& $(4.343 \pm 0.325)\times 10^{-2}$
& $(4.177 \pm 0.261)\times 10^{-2}$
& $(4.949 \pm 0.389)\times 10^{-2}$
& $(9.427 \pm 0.993)\times 10^{-2}$
& $(1.172 \pm 0.221)\times 10^{-1}$ \\
\bottomrule
\end{tabular}
}
\end{table}

\section{Proofs and supplementary analysis for the sparse inverse framework}
\label{app:theory_details}

This appendix proves the two stability results in
Section~\ref{sec:theory_sparse_inverse}, records the assumption that connects
the sampled loss to continuous residual norms, and spells out the structural
calculations used for the Navier--Stokes and MHD examples. All symbols below
have the meaning fixed in Section~\ref{sec:theory_sparse_inverse}.

\subsection{Proof of linearized joint observability}
\label{app:proof_linear_observability}

\begin{proof}[Proof of Proposition~\ref{prop:forward_sensitivity_implies_observability}]
Recall from \eqref{eq:parameter_to_state_derivative} that
\begin{equation*}
    \mathcal T^\dagger
    =D\mathscr U(\lambda^\dagger)
    =-(\mathcal L^\dagger)^{-1}\mathcal R_\lambda^\dagger.
    \label{eq:appendix_reduced_state_derivative}
\end{equation*}
For a perturbation
\((h,\xi)\in\mathcal X_U\times\mathbb R^m\), set
\begin{equation*}
    \widetilde h:=h-\mathcal T^\dagger\xi.
    \label{eq:appendix_transverse_state_perturbation}
\end{equation*}
Since
\(\mathcal L^\dagger\mathcal T^\dagger
=-\mathcal R_\lambda^\dagger\),
\begin{equation*}
    \mathcal L^\dagger\widetilde h
    =\mathcal L^\dagger h
     +\mathcal R_\lambda^\dagger\xi.
    \label{eq:appendix_reduced_residual_identity}
\end{equation*}
Assumption~\ref{ass:constrained_state_stability} therefore yields
\begin{equation}
    \|\widetilde h\|_{\mathcal X_U}
    \leq
    C_F
    \|\mathcal L^\dagger h
      +\mathcal R_\lambda^\dagger\xi\|_{\mathcal Y}.
    \label{eq:proof_transverse_h_bound}
\end{equation}

The observation perturbation has the decomposition
\begin{equation*}
\begin{aligned}
    \mathcal H_U^\dagger h
    +\mathcal H_\lambda^\dagger\xi
    &=\mathcal H_U^\dagger\widetilde h
      +(\mathcal H_U^\dagger\mathcal T^\dagger
        +\mathcal H_\lambda^\dagger)\xi\\
    &=\mathcal H_U^\dagger\widetilde h+J_\lambda\xi.
\end{aligned}
\label{eq:appendix_observation_decomposition}
\end{equation*}
By Assumption~\ref{ass:parameter_observation_sensitivity},
\begin{equation*}
\begin{aligned}
    \sigma_*|\xi|
    &\leq\|J_\lambda\xi\|_{\mathcal Y_{\rm obs}}\\
    &\leq
    \|\mathcal H_U^\dagger h
      +\mathcal H_\lambda^\dagger\xi\|_{\mathcal Y_{\rm obs}}
    +\|\mathcal H_U^\dagger\|\,
     \|\widetilde h\|_{\mathcal X_U}.
\end{aligned}
\label{eq:proof_parameter_sensitivity_bound}
\end{equation*}
Combining this estimate with \eqref{eq:proof_transverse_h_bound} gives the
explicit parameter bound
\begin{equation}
\begin{aligned}
    |\xi|
    \leq{}&
    \sigma_*^{-1}
    \|\mathcal H_U^\dagger h
      +\mathcal H_\lambda^\dagger\xi\|_{\mathcal Y_{\rm obs}}\\
    &+\sigma_*^{-1}\|\mathcal H_U^\dagger\|C_F
    \|\mathcal L^\dagger h
      +\mathcal R_\lambda^\dagger\xi\|_{\mathcal Y}.
\end{aligned}
\label{eq:appendix_parameter_bound_explicit}
\end{equation}
Finally, \(h=\widetilde h+\mathcal T^\dagger\xi\), and hence
\begin{equation}
    \|h\|_{\mathcal X_U}
    \leq\|\widetilde h\|_{\mathcal X_U}
    +\|\mathcal T^\dagger\|\,|\xi|.
    \label{eq:appendix_state_recovery_from_reduced_part}
\end{equation}
Substituting \eqref{eq:proof_transverse_h_bound} and
\eqref{eq:appendix_parameter_bound_explicit} into the last inequality proves
\eqref{eq:general_linearized_observability_estimate}. If its two residual
terms vanish, then \(\xi=0\) by
\eqref{eq:appendix_parameter_bound_explicit}, and subsequently \(h=0\) by
\eqref{eq:appendix_state_recovery_from_reduced_part}. This proves the final
claim.
\end{proof}

The argument also shows why field reconstruction and parameter identification
can have very different conditioning. The former inherits the state constant
\(C_F\), whereas the latter contains the additional factor
\(\sigma_*^{-1}\). A nearly invisible parameter direction cannot be repaired
by a well-conditioned state equation alone.

\subsection{Proof of nonlinear stability and loss-to-error control}
\label{app:proof_nonlinear_loss}

Equip \(\mathcal Y\times\mathcal Y_{\rm obs}\) with the product norm
\begin{equation*}
    \|(r,q)\|_{\mathcal Y_A}
    :=\|r\|_{\mathcal Y}+\|q\|_{\mathcal Y_{\rm obs}},
    \qquad
    \mathcal A(Z):=(\mathcal R(Z),\mathcal O_S(Z)).
    \label{eq:augmented_product_norm}
\end{equation*}

\begin{proof}[Proof of Theorem~\ref{thm:joint_local_stability}]
Let \(e=Z-Z^\dagger=(h,\xi)\). The derivative of the augmented map at the
reference pair is
\begin{equation*}
    D\mathcal A(Z^\dagger)e
    =\left(
        \mathcal L^\dagger h
        +\mathcal R_\lambda^\dagger\xi,
        \mathcal H_U^\dagger h
        +\mathcal H_\lambda^\dagger\xi
      \right).
    \label{eq:augmented_derivative_at_reference}
\end{equation*}
Proposition~\ref{prop:forward_sensitivity_implies_observability} therefore
implies
\begin{equation*}
    d_{\mathcal X}(Z,Z^\dagger)
    \leq C_{\rm obs}
    \|D\mathcal A(Z^\dagger)e\|_{\mathcal Y_A}.
    \label{eq:linearized_augmented_lower_bound}
\end{equation*}

Local Lipschitz continuity of \(D\mathcal A\) gives
\begin{equation*}
    \mathcal A(Z)-\mathcal A(Z^\dagger)
    =D\mathcal A(Z^\dagger)e+\rho_A(e),
    \qquad
    \|\rho_A(e)\|_{\mathcal Y_A}
    \leq C_A\|e\|_{\mathcal X}^2.
    \label{eq:augmented_taylor_expansion}
\end{equation*}
Combining the last two estimates yields
\begin{equation*}
\begin{aligned}
    d_{\mathcal X}(Z,Z^\dagger)
    \leq{}&
    C_{\rm obs}
    \|\mathcal A(Z)-\mathcal A(Z^\dagger)\|_{\mathcal Y_A}\\
    &+C_{\rm obs}C_A\|e\|_{\mathcal X}^2.
\end{aligned}
\label{eq:appendix_nonlinear_pre_absorption}
\end{equation*}
Because \(d_{\mathcal X}\) is equivalent to the product norm on
\(\mathcal X\), the quadratic term can be absorbed into the left-hand side
after restricting \(Z\) to a sufficiently small neighborhood
\(\mathcal N_0\) of \(Z^\dagger\). Since
\(\mathcal R(Z^\dagger)=0\), this gives
\eqref{eq:general_local_nonlinear_stability}. If both
\(\mathcal R(Z)=0\) and
\(\mathcal O_S(Z)=\mathcal O_S(Z^\dagger)\), the right-hand side vanishes,
so \(Z=Z^\dagger\) within \(\mathcal N_0\). This proves local injectivity on
\(\mathcal M\).

It remains to pass from residual norms to the loss. By
\eqref{eq:residual_product_norm}, positivity of the residual and observation
weights, non-negativity of \(\mathcal L_{\rm reg}\), and equivalence of the
weighted observation norm, there is a constant \(C_\gamma>0\) such that
\begin{equation*}
    \|\mathcal R(\widehat Z)\|_{\mathcal Y}
    +\|\mathcal O_S(\widehat Z)-y^\delta\|_{\mathcal Y_{\rm obs}}
    \leq C_\gamma\mathcal J_\delta(\widehat Z)^{1/2}.
    \label{eq:loss_controls_residual_and_data}
\end{equation*}
The noise model \eqref{eq:general_noisy_observation_model} gives
\begin{equation*}
\begin{aligned}
    \|\mathcal O_S(\widehat Z)-\mathcal O_S(Z^\dagger)
    \|_{\mathcal Y_{\rm obs}}
    &\leq
    \|\mathcal O_S(\widehat Z)-y^\delta\|_{\mathcal Y_{\rm obs}}
    +\|\varepsilon^\delta\|_{\mathcal Y_{\rm obs}}\\
    &\leq
    \|\mathcal O_S(\widehat Z)-y^\delta\|_{\mathcal Y_{\rm obs}}
    +\delta.
\end{aligned}
\label{eq:appendix_noisy_observation_triangle}
\end{equation*}
Substitution into \eqref{eq:general_local_nonlinear_stability} proves
\eqref{eq:general_lipschitz_loss_to_error}.
\end{proof}

For completeness, suppose only the conditional estimate
\begin{equation*}
    d_{\mathcal X}(Z,Z^\dagger)
    \leq C_{\rm stab}\,
    \omega\!\left(
        \|\mathcal R(Z)\|_{\mathcal Y}
        +\|\mathcal O_S(Z)-\mathcal O_S(Z^\dagger)
        \|_{\mathcal Y_{\rm obs}}
    \right)
    \label{eq:general_conditional_stability_modulus}
\end{equation*}
is available, where \(\omega\) is increasing, continuous, and satisfies
\(\omega(0)=0\). The same calculation gives
\begin{equation*}
    d_{\mathcal X}(\widehat Z,Z^\dagger)
    \leq C_{\rm stab}\,
    \omega\!\left(
        C_\gamma\mathcal J_\delta(\widehat Z)^{1/2}+\delta
    \right).
    \label{eq:general_continuous_loss_to_error}
\end{equation*}
The choice \(\omega(s)=s\) recovers local Lipschitz stability;
\(\omega(s)=s^\alpha\), \(0<\alpha<1\), gives H\"older stability. A
logarithmic modulus can be used when no stronger inverse estimate is
available.

\subsection{Empirical loss consistency and discretization terms}
\label{app:empirical_consistency}
\label{subsec:empirical_multiphysics_error}

The continuous norms in \eqref{eq:general_continuous_physics_loss} are
approximated using interior grid or collocation points, boundary and initial
points, and sensor locations. Let \(\widehat{\mathcal J}_\delta\) denote the
resulting empirical objective. For a fixed trained model, assume the
a posteriori consistency estimate
\begin{equation}
\begin{aligned}
    &\|\mathcal R(\widehat Z)\|_{\mathcal Y}
    +\|\mathcal O_S(\widehat Z)-y^\delta\|_{\mathcal Y_{\rm obs}}\\
    &\qquad\leq
    C_{\rm emp}\widehat{\mathcal J}_\delta(\widehat Z)^{1/2}
    +\varepsilon_{\rm disc},
\end{aligned}
\label{eq:empirical_to_continuous_consistency}
\end{equation}
where
\begin{equation*}
    \varepsilon_{\rm disc}
    :=\varepsilon_{\rm der}
     +\varepsilon_{\rm quad}
     +\varepsilon_{\rm int}
     +\varepsilon_{\rm con}.
    \label{eq:discretization_error_decomposition}
\end{equation*}
The four terms represent derivative-discretization, quadrature or collocation,
sensor-interpolation, and numerical constraint-enforcement errors,
respectively. Combining \eqref{eq:empirical_to_continuous_consistency} with
the noisy-data triangle inequality gives
\begin{equation*}
    d_{\mathcal X}(\widehat Z,Z^\dagger)
    \leq C_{\rm stab}\,
    \omega\!\left(
        C_{\rm emp}\widehat{\mathcal J}_\delta(\widehat Z)^{1/2}
        +\varepsilon_{\rm disc}+\delta
    \right).
    \label{eq:empirical_multiphysics_loss_to_error}
\end{equation*}
For \(\omega(s)=s\), this is
\eqref{eq:empirical_multiphysics_lipschitz_bound}. This is an a posteriori
condition rather than a uniform sampling theorem: a small discrete loss is
informative only when the discretization also makes
\(\varepsilon_{\rm disc}\) small.

\subsection{Detailed structural calculations for the present systems}
\label{app:system_structural_details}

\subsubsection{Incompressible Navier--Stokes system}

Although the cavity grouping in \eqref{eq:cavity_two_field_grouping} places
pressure first, we retain the standard physical order
\((\mathbf u,p)\) when writing the Navier--Stokes operator. This is only a
change of coordinates in the product state space. For fixed
\(\nu^\dagger>0\), the state derivative of
\eqref{eq:ns_instance_general_residual} at
\((\mathbf u^\dagger,p^\dagger)\) is the constrained Oseen operator
\begin{equation*}
\begin{aligned}
    \mathcal L_{\rm NS}^\dagger
    (\delta\mathbf u,\delta p)
    =\Big(&
        (\mathbf u^\dagger\cdot\nabla)\delta\mathbf u
        +(\delta\mathbf u\cdot\nabla)\mathbf u^\dagger
        +\nabla\delta p
        -\nu^\dagger\Delta\delta\mathbf u,
        \\
        &\nabla\cdot\delta\mathbf u,
        \delta\mathbf u|_{\partial\Omega},
        \int_\Omega\delta p\,dx
    \Big).
\end{aligned}
\label{eq:linearized_oseen_operator}
\end{equation*}
For divergence-free \(\mathbf u^\dagger\) and homogeneous velocity
perturbations, integration by parts gives
\begin{equation*}
\begin{aligned}
    \big((\mathbf u^\dagger\cdot\nabla)\delta\mathbf u,
         \delta\mathbf u\big)_{L^2(\Omega)}
    &=\frac12\int_{\partial\Omega}
        |\delta\mathbf u|^2\mathbf u^\dagger\cdot\mathbf n\,ds\\
    &\quad-\frac12\int_\Omega
        (\nabla\cdot\mathbf u^\dagger)
        |\delta\mathbf u|^2\,dx
    =0.
\end{aligned}
\label{eq:ns_transport_skew_symmetry}
\end{equation*}
Here \(\mathbf n\) is the outward unit normal to \(\partial\Omega\).
The viscous term satisfies
\begin{equation*}
    \big(-\nu^\dagger\Delta\delta\mathbf u,
         \delta\mathbf u\big)_{L^2(\Omega)}
    =\nu^\dagger
     \|\nabla\delta\mathbf u\|_{L^2(\Omega)}^2,
    \label{eq:ns_viscous_dissipation}
\end{equation*}
while the remaining convection term obeys
\begin{equation*}
    \left|
        \big((\delta\mathbf u\cdot\nabla)\mathbf u^\dagger,
             \delta\mathbf u\big)_{L^2(\Omega)}
    \right|
    \leq
    \|\nabla\mathbf u^\dagger\|_{L^\infty(\Omega)}
    \|\delta\mathbf u\|_{L^2(\Omega)}^2.
    \label{eq:ns_background_gradient_bound}
\end{equation*}
Thus the viscous term is coercive in the velocity energy space, transport by
the divergence-free reference velocity is energy-neutral, and the
background-gradient term is lower order. Pressure is controlled by the
velocity--pressure inf-sup condition after its additive gauge is fixed.
These observations support the state estimate in
Assumption~\ref{ass:constrained_state_stability}, but they do
not prove bounded invertibility by themselves. One still needs compatible
data and a trivial kernel for the homogeneous constrained Oseen problem.

For an unsteady problem, the time derivative supplies the energy increment.
If the initial state is prescribed in \(\mathcal C\), the perturbation at
\(t=0\) is fixed; otherwise that initial state is part of the inverse problem.
The same transport and viscous estimates apply on time intervals over which
the reference solution remains regular.

For the cylinder-wake equations in \eqref{eq:cylinder_param}, the parameter
part of the linearized momentum residual also contains
\begin{equation*}
    \delta\lambda_1\,
    (\mathbf u^\dagger\cdot\nabla)\mathbf u^\dagger
    -\delta\nu\,\Delta\mathbf u^\dagger.
    \label{eq:cylinder_parameter_derivative}
\end{equation*}
These terms belong to
\(\mathcal R_\lambda^\dagger(\delta\lambda_1,\delta\nu)\); the state
operator \(\mathcal L^\dagger\) is formed with
\(\lambda_1^\dagger\) and \(\nu^\dagger\) fixed.

\subsubsection{Non-ideal magnetohydrodynamics}
\label{app:mhd_coupling_details}

Around a regular reference state with positive density and pressure, the
linearized non-ideal MHD system has a coupled hyperbolic--parabolic structure.
Its directly dissipative part is
\begin{equation*}
    \nu^\dagger
    \|\nabla\delta\mathbf v\|_{L^2(\Omega)}^2
    +\eta^\dagger
    \|\nabla\delta\mathbf B\|_{L^2(\Omega)}^2.
    \label{eq:mhd_direct_dissipation}
\end{equation*}
The remaining variables are controlled through the joint energy and the
off-diagonal coupling. For example, the derivative of the magnetic stress in
the momentum equation is
\begin{equation*}
\begin{aligned}
    D_{\mathbf B}\!\left[
        \nabla\cdot\left(
            \frac12|\mathbf B|^2\mathbf I
            -\mathbf B\otimes\mathbf B
        \right)
    \right]_{\mathbf B=\mathbf B^\dagger}
    [\delta\mathbf B]
    =\nabla\cdot\Big(&
        (\mathbf B^\dagger\cdot\delta\mathbf B)\mathbf I\\
        &-\mathbf B^\dagger\otimes\delta\mathbf B
        -\delta\mathbf B\otimes\mathbf B^\dagger
    \Big).
\end{aligned}
\label{eq:linearized_magnetic_stress_coupling}
\end{equation*}
Here \(\mathbf I\) is the identity tensor and \(\otimes\) denotes the
dyadic product.
The linearized induction equation contains
\begin{equation*}
\begin{aligned}
    \partial_t\delta\mathbf B
    &-\nabla\times\left(
        \delta\mathbf v\times\mathbf B^\dagger
        +\mathbf v^\dagger\times\delta\mathbf B
    \right)\\
    &+\nabla\times\left(
        \eta^\dagger\nabla\times\delta\mathbf B
        +\delta\eta\,\mathbf J^\dagger
    \right),
\end{aligned}
\label{eq:linearized_induction_coupling}
\end{equation*}
where
\(\mathbf J^\dagger:=\nabla\times\mathbf B^\dagger\). The first coupling
shows how magnetic perturbations enter the momentum balance; the second shows
how velocity and resistivity perturbations enter the magnetic evolution.
These terms are not sign-definite. Under bounded-coefficient and regularity
assumptions, they can be estimated by a small multiple of the dissipative
terms plus a multiple of the joint energy, leading to
\eqref{eq:mhd_linearized_energy_estimate}.

The initial and periodic conditions, the constraint
\(\nabla\cdot\mathbf B=0\), and Gr\"onwall's inequality then give an a
priori state estimate on a regular time interval. Verifying
Assumption~\ref{ass:constrained_state_stability}
still requires well-posedness and compatibility for the complete linearized
initial-boundary value problem. Shock formation, vanishing dissipation, loss
of positivity, and loss of regularity are outside this local argument. The
coupling provides a route by which observations of \(\mathbf B\) can constrain
the hydrodynamic variables, but it does not guarantee
Assumption~\ref{ass:parameter_observation_sensitivity} for an arbitrary
sensor configuration.

\subsection{Scale-aware numerical diagnostic for Assumption~A2}
\label{app:a2_diagnostic}

For a finite sensor set, identify \(\mathcal Y_{\rm obs}\) with
\(\mathbb R^{n_y}\), where \(n_y\) is the number of scalar measured values,
including all observed channels. Then
\(J_\lambda\in\mathbb R^{n_y\times m}\). Its singular values depend on the
units used for both parameters and observations. Let
\begin{equation}
    D_\lambda
    :=\operatorname{diag}
       (s_{\lambda,1},\ldots,s_{\lambda,m}),
    \qquad
    D_y
    :=\operatorname{diag}(s_{y,1},\ldots,s_{y,n_y}),
    \label{eq:a2_scaling_matrices}
\end{equation}
where all scales are strictly positive. In the normalized coordinates
\begin{equation*}
    \bar\lambda
    :=D_\lambda^{-1}(\lambda-\lambda^\dagger),
    \qquad
    \bar y:=D_y^{-1}(y-y^\dagger),
\end{equation*}
where \(y^\dagger:=\Phi(\lambda^\dagger)\). The dimensionless sensitivity
matrix is
\begin{equation*}
    \widetilde J_\lambda
    :=D_y^{-1}J_\lambda D_\lambda.
    \label{eq:normalized_parameter_sensitivity}
\end{equation*}

The \(j\)-th column can be estimated along the local PDE solution branch by a
centered finite difference,
\begin{equation*}
    [\widetilde J_\lambda]_{:j}
    \approx
    D_y^{-1}
    \frac{
        \Phi(\lambda^\dagger+\tau_js_{\lambda,j}e_j)
        -\Phi(\lambda^\dagger-\tau_js_{\lambda,j}e_j)
    }{2\tau_j},
    \label{eq:a2_centered_finite_difference}
\end{equation*}
where \(e_j\in\mathbb R^m\) is the \(j\)-th coordinate vector and
\(\tau_j>0\) is dimensionless. The step must remain inside \(\Lambda\). Automatic
differentiation through a differentiable PDE solver may be used instead.

Normalization preserves rank and makes sensitivity magnitudes comparable
across parameterizations. The main diagnostic is
\(\sigma_{\min}(\widetilde J_\lambda)\). When \(m>1\), the condition number
\begin{equation*}
    \kappa(\widetilde J_\lambda)
    :=\frac{\sigma_{\max}(\widetilde J_\lambda)}
            {\sigma_{\min}(\widetilde J_\lambda)}
\end{equation*}
also reveals nearly collinear parameter directions. For a single parameter,
\(\sigma_{\min}(\widetilde J_\lambda)
=\|[\widetilde J_\lambda]_{:1}\|_2\). The calculation should be repeated for
several step sizes and grid resolutions. Stable values across these choices
help distinguish genuine physical insensitivity from numerical error.

A positive value at one reference pair gives numerical evidence only for that
pair and sensor configuration; it is not a uniform lower bound on a parameter
region. A reproducible diagnostic should therefore report the reference
state, sensor layout, normalization scales, perturbation sizes,
\(\sigma_{\min}\), and, when applicable, the condition number.

\subsubsection{Numerical check for the two-dimensional cavity problem}
We apply the preceding diagnostic to the pressure-observed cavity problem at
\(\mathrm{Re}^\dagger=5000\), for which
\(\nu^\dagger=2\times10^{-4}\). The trainable coefficient in the numerical
model is represented by the log-viscosity
\(\theta=\log\nu\). At the reference value
\(\theta^\dagger=\log\nu^\dagger\), the corresponding sensitivity satisfies
\begin{equation*}
    J_\theta
    :=\frac{d}{d\theta}\Phi(e^\theta)
      \bigg|_{\theta=\theta^\dagger}
    =\nu^\dagger J_\nu.
    \label{eq:a2_log_viscosity_relation}
\end{equation*}
Thus, taking \(s_\nu=\nu^\dagger\) in
\eqref{eq:a2_scaling_matrices} makes \(J_\theta\) equal to the parameter-scaled
column of \(\widetilde J_\lambda\). The pressure is nondimensionalized with
the cavity scale \(s_p=U_{\rm lid}^2=1\) (equivalently,
\(\rho U_{\rm lid}^2=1\) when \(\rho=1\)), so \(D_y=I\) for this
calculation. Each pressure field is shifted to have zero spatial mean before
sampling, which removes the arbitrary additive pressure constant.

For each perturbation size
\(\tau\in\{0.005,0.010,0.020\}\), two OpenFOAM solutions were computed at
\begin{equation*}
    \nu_\pm=\nu^\dagger e^{\pm\tau},
    \qquad
    \mathrm{Re}_\pm=\mathrm{Re}^\dagger e^{\mp\tau}.
    \label{eq:a2_cavity_perturbed_parameters}
\end{equation*}
The resulting Reynolds-number pairs
\((\mathrm{Re}_+,\mathrm{Re}_-)\) are
\((4975.062,5025.063)\), \((4950.249,5050.251)\), and
\((4900.993,5101.007)\), respectively. The domain, boundary conditions,
discretization, solver settings, and pressure convention were kept fixed, and
all solutions were exported to the same \(129\times129\) grid. No observation
noise was added. Within each 200-, 300-, or 400-sensor experiment, the same
boundary pressure mask was applied to both members of every centered pair.

Let \(\Phi_{N_s}\) denote the vector of mean-free pressures at a fixed mask of
\(N_s\) active sensors. We estimate
\begin{equation*}
    \widehat J_{\theta,N_s}(\tau)
    :=\frac{
        \Phi_{N_s}(\theta^\dagger+\tau)
        -\Phi_{N_s}(\theta^\dagger-\tau)
    }{2\tau},
    \qquad
    \widehat\sigma_{\min,N_s}(\tau)
    :=\|\widehat J_{\theta,N_s}(\tau)\|_2.
    \label{eq:a2_cavity_finite_difference}
\end{equation*}
The second equality follows because viscosity is the only unknown physical
parameter in this example. To facilitate comparison between masks of different
sizes, we also report the root-mean-square sensitivity
\begin{equation*}
    \widehat\sigma_{\min,N_s}^{\rm rms}(\tau)
    :=N_s^{-1/2}\widehat\sigma_{\min,N_s}(\tau).
    \label{eq:a2_cavity_rms_sensitivity}
\end{equation*}

\begin{table}[htb]
\centering
\caption{Finite-difference diagnostic for Assumption~A2 in the
two-dimensional cavity problem at \(\mathrm{Re}=5000\).}
\label{tab:a2_ldc_re5000_diagnostic}

\small
\setlength{\tabcolsep}{5pt}

\begin{tabular}{lccccc}
\toprule
\(N_s\)
& \(\widehat{\sigma}_{\min}(0.005)\)
& \(\widehat{\sigma}_{\min}(0.010)\)
& \(\widehat{\sigma}_{\min}(0.020)\)
& \(\widehat{\sigma}_{\min}^{\rm rms}(0.010)\)
& \(\Delta_\tau\) \\
\midrule
200 & 0.400947 & 0.406944 & 0.409901 & 0.028775 & \(1.47\%\) \\
300 & 0.374331 & 0.382087 & 0.385947 & 0.022060 & \(2.03\%\) \\
400 & 0.467187 & 0.476120 & 0.480551 & 0.023806 & \(1.88\%\) \\
\bottomrule
\end{tabular}

\end{table}

Here
\begin{equation*}
    \Delta_\tau
    :=\max_{\tau\in\{0.005,0.020\}}
    \frac{
        \left|\widehat\sigma_{\min,N_s}(\tau)
        -\widehat\sigma_{\min,N_s}(0.010)\right|
    }{\widehat\sigma_{\min,N_s}(0.010)}
    \label{eq:a2_cavity_step_variation}
\end{equation*}
measures the largest relative change from the middle step size. All nine
estimates are strictly positive, and the largest step-size variation is
\(2.03\%\). Since the sensitivity matrix has only one column in this case,
these nonzero values show that the estimated matrix has full column rank. The
corresponding one-parameter condition number is identically one and is
therefore not listed separately. The
small variation across the three perturbation sizes also indicates that the
observed sensitivity is not tied to a single finite-difference step. The
results therefore provide direct local numerical evidence for
Assumption~\ref{ass:parameter_observation_sensitivity} at the selected
reference state and for each of the three sensor configurations.

The values are not monotone in \(N_s\). This is not a contradiction: the
three sensor masks are not nested, so changing \(N_s\) also changes the sensor
locations. Moreover, the Euclidean singular value has an inherent dependence
on the number of observations. The RMS column removes the elementary
\(\sqrt{N_s}\) scaling but not the effect of sensor placement. This behavior
is consistent with the configuration dependence stated in
Assumption~\ref{ass:parameter_observation_sensitivity}. Finally, the present
test is local to \(\mathrm{Re}=5000\) and the exported \(129\times129\) grid.

\clearpage

\bibliographystyle{plain}
\bibliography{References}

@article{ghia1982high,
  title={{{High-Re} solutions for incompressible flow using the {Navier-Stokes} equations and a multigrid method}},
  author={Ghia, UKNG and Ghia, Kirti N and Shin, CT},
  journal={{Journal of Computational Physics}},
  volume={48},
  number={3},
  pages={387--411},
  year={1982},
  publisher={Elsevier}
}

@article{mignone2007pluto,
  title={{{PLUTO}: a numerical code for computational astrophysics}},
  author={Mignone, Andrea and Bodo, G and Massaglia, S and Matsakos, Titos and Tesileanu, O ea and Zanni, C and Ferrari, Anthony},
  journal={{The Astrophysical Journal Supplement Series}},
  volume={170},
  number={1},
  pages={228--242},
  year={2007}
}

@inproceedings{loshchilov2019decoupled,
  title={{Decoupled Weight Decay Regularization}},
  author={Loshchilov, Ilya and Hutter, Frank},
  booktitle={{International Conference on Learning Representations}},
  year={2019}
}

@article{bruneau2006two,
  title={{The {2D} lid-driven cavity problem revisited}},
  author={Bruneau, Charles-Henri and Saad, Mazen},
  journal={{Computers \& Fluids}},
  volume={35},
  number={3},
  pages={326--348},
  year={2006},
  doi={10.1016/j.compfluid.2004.12.004}
}

@article{albensoeder2005accurate,
  title={{Accurate three-dimensional lid-driven cavity flow}},
  author={Albensoeder, Stefan and Kuhlmann, Hendrik C.},
  journal={{Journal of Computational Physics}},
  volume={206},
  number={2},
  pages={536--558},
  year={2005},
  doi={10.1016/j.jcp.2004.12.024}
}

@incollection{schafer1996benchmark,
  title={{Benchmark computations of laminar flow around a cylinder}},
  author={Sch{\"a}fer, M. and Turek, S. and Durst, F. and Krause, E. and Rannacher, R.},
  booktitle={{Flow Simulation with High-Performance Computers II}},
  editor={Hirschel, Ernst Heinrich},
  series={Notes on Numerical Fluid Mechanics},
  volume={48},
  pages={547--566},
  publisher={Vieweg+Teubner Verlag},
  year={1996},
  doi={10.1007/978-3-322-89849-4_39}
}

@article{orszag1979small,
  title={{Small-scale structure of two-dimensional magnetohydrodynamic turbulence}},
  author={Orszag, S. A. and Tang, C. M.},
  journal={{Journal of Fluid Mechanics}},
  volume={90},
  number={1},
  pages={129--143},
  year={1979},
  doi={10.1017/S002211207900210X}
}

@article{toth2000constraint,
  title={{The {$\nabla \cdot \mathbf{B}=0$} constraint in shock-capturing magnetohydrodynamics codes}},
  author={T{\'o}th, G{\'a}bor},
  journal={{Journal of Computational Physics}},
  volume={161},
  number={2},
  pages={605--652},
  year={2000},
  doi={10.1006/jcph.2000.6519}
}

@article{raissi2019physics,
  title={{Physics-informed neural networks: A deep learning framework for solving forward and inverse problems involving nonlinear partial differential equations}},
  author={Raissi, Maziar and Perdikaris, Paris and Karniadakis, George E},
  journal={{Journal of Computational Physics}},
  volume={378},
  pages={686--707},
  year={2019},
  publisher={Elsevier}
}

@article{mishra2022estimates,
  title={{Estimates on the generalization error of physics-informed neural networks for approximating a class of inverse problems for {PDEs}}},
  author={Mishra, Siddhartha and Molinaro, Roberto},
  journal={{IMA Journal of Numerical Analysis}},
  volume={42},
  number={2},
  pages={981--1022},
  year={2022},
  publisher={Oxford University Press}
}

@article{weller1998tensorial,
  title={{A tensorial approach to computational continuum mechanics using object-oriented techniques}},
  author={Weller, Henry G and Tabor, Gavin and Jasak, Hrvoje and Fureby, Christer},
  journal={{Computers in Physics}},
  volume={12},
  number={6},
  pages={620--631},
  year={1998},
  publisher={American Institute of Physics}
}

@article{willcox2006unsteady,
  title={{Unsteady flow sensing and estimation via the gappy proper orthogonal decomposition}},
  author={Willcox, Karen},
  journal={{Computers \& Fluids}},
  volume={35},
  number={2},
  pages={208--226},
  year={2006},
  publisher={Elsevier}
}

@article{fukami2021global,
  title={{Global field reconstruction from sparse sensors with {Voronoi} tessellation-assisted deep learning}},
  author={Fukami, Kai and Maulik, Romit and Ramachandra, Nesar and Fukagata, Koji and Taira, Kunihiko},
  journal={{Nature Machine Intelligence}},
  volume={3},
  number={11},
  pages={945--951},
  year={2021},
  publisher={Nature Publishing Group UK London}
}

@article{santos2023development,
  title={{Development of the senseiver for efficient field reconstruction from sparse observations}},
  author={Santos, Javier E and Fox, Zachary R and Mohan, Arvind and O’Malley, Daniel and Viswanathan, Hari and Lubbers, Nicholas},
  journal={{Nature Machine Intelligence}},
  volume={5},
  number={11},
  pages={1317--1325},
  year={2023},
  publisher={Nature Publishing Group UK London}
}

@article{lu2021learning,
  title={{Learning nonlinear operators via {DeepONet} based on the universal approximation theorem of operators}},
  author={Lu, Lu and Jin, Pengzhan and Pang, Guofei and Zhang, Zhongqiang and Karniadakis, George Em},
  journal={{Nature Machine Intelligence}},
  volume={3},
  number={3},
  pages={218--229},
  year={2021},
  publisher={Nature Publishing Group UK London}
}

@article{vaswani2017attention,
  title={{Attention is all you need}},
  author={Vaswani, Ashish and Shazeer, Noam and Parmar, Niki and Uszkoreit, Jakob and Jones, Llion and Gomez, Aidan N and Kaiser, {\L}ukasz and Polosukhin, Illia},
  journal={{Advances in Neural Information Processing Systems}},
  volume={30},
  year={2017}
}

@article{li2020fourier,
  title={{{Fourier} neural operator for parametric partial differential equations}},
  author={Li, Zongyi and Kovachki, Nikola and Azizzadenesheli, Kamyar and Liu, Burigede and Bhattacharya, Kaushik and Stuart, Andrew and Anandkumar, Anima},
  journal={{arXiv preprint arXiv:2010.08895}},
  year={2020}
}

@article{kovachki2023neural,
  title={{Neural operator: Learning maps between function spaces with applications to {PDEs}}},
  author={Kovachki, Nikola and Li, Zongyi and Liu, Burigede and Azizzadenesheli, Kamyar and Bhattacharya, Kaushik and Stuart, Andrew and Anandkumar, Anima},
  journal={{Journal of Machine Learning Research}},
  volume={24},
  number={89},
  pages={1--97},
  year={2023}
}

@article{zhao2024recfno,
  title={{{RecFNO}: A resolution-invariant flow and heat field reconstruction method from sparse observations via {Fourier} neural operator}},
  author={Zhao, Xiaoyu and Chen, Xiaoqian and Gong, Zhiqiang and Zhou, Weien and Yao, Wen and Zhang, Yunyang},
  journal={{International Journal of Thermal Sciences}},
  volume={195},
  pages={108619},
  year={2024},
  publisher={Elsevier}
}

@article{raissi2020hidden,
  title={{Hidden fluid mechanics: Learning velocity and pressure fields from flow visualizations}},
  author={Raissi, Maziar and Yazdani, Alireza and Karniadakis, George Em},
  journal={{Science}},
  volume={367},
  number={6481},
  pages={1026--1030},
  year={2020},
  publisher={American Association for the Advancement of Science}
}

@article{raissi2019deep,
  title={{Deep learning of vortex-induced vibrations}},
  author={Raissi, Maziar and Wang, Zhicheng and Triantafyllou, Michael S and Karniadakis, George Em},
  journal={{Journal of Fluid Mechanics}},
  volume={861},
  pages={119--137},
  year={2019},
  publisher={Cambridge University Press}
}

@article{he2020physics,
  title={{Physics-informed neural networks for multiphysics data assimilation with application to subsurface transport}},
  author={He, QiZhi and Barajas-Solano, David and Tartakovsky, Guzel and Tartakovsky, Alexandre M},
  journal={{Advances in Water Resources}},
  volume={141},
  pages={103610},
  year={2020},
  publisher={Elsevier}
}

@article{parfenyev2024inferring,
  title={{Inferring parameters and reconstruction of two-dimensional turbulent flows with physics-informed neural networks}},
  author={Parfenyev, Vladimir and Blumenau, Mark and Nikitin, Ilia},
  journal={{JETP Letters}},
  volume={120},
  number={8},
  pages={599--607},
  year={2024},
  publisher={Springer}
}

@article{toscano2025aivt,
  title={{{AIVT}: Inference of turbulent thermal convection from measured {3D} velocity data by physics-informed {Kolmogorov-Arnold} networks}},
  author={Toscano, Juan Diego and K{\"a}ufer, Theo and Wang, Zhibo and Maxey, Martin and Cierpka, Christian and Karniadakis, George Em},
  journal={{Science Advances}},
  volume={11},
  number={19},
  pages={eads5236},
  year={2025},
  publisher={American Association for the Advancement of Science}
}

@article{guan2025trajectory,
  title={{Trajectory-guided multiscale deep inversion networks for solving incompressible magnetohydrodynamics inverse problems}},
  author={Guan, Xiaofei and Hu, Boya and Mao, Shipeng},
  journal={{Physics of Fluids}},
  volume={37},
  number={10},
  year={2025},
  publisher={AIP Publishing}
}

@article{wang2021learning,
  title={{Learning the solution operator of parametric partial differential equations with physics-informed {DeepONets}}},
  author={Wang, Sifan and Wang, Hanwen and Perdikaris, Paris},
  journal={{Science Advances}},
  volume={7},
  number={40},
  pages={eabi8605},
  year={2021},
  publisher={American Association for the Advancement of Science}
}

@article{li2024physics,
  title={{Physics-informed neural operator for learning partial differential equations}},
  author={Li, Zongyi and Zheng, Hongkai and Kovachki, Nikola and Jin, David and Chen, Haoxuan and Liu, Burigede and Azizzadenesheli, Kamyar and Anandkumar, Anima},
  journal={{ACM/IMS Journal of Data Science}},
  volume={1},
  number={3},
  pages={1--27},
  year={2024},
  publisher={ACM New York, NY}
}

@article{barman2026efficient,
  title={{An Efficient Wavelet-based Physics Informed Residual Neural Networks for Flow Field Reconstruction with Extremely Sparse Data}},
  author={Barman, Biswanath and Ray, Rajendra K},
  journal={{arXiv preprint arXiv:2601.18848}},
  year={2026}
}

@article{barman2026physicsformer,
  title={{PhysicsFormer: An Efficient and Fast Attention-Based Physics Informed Neural Network for Solving Incompressible Navier Stokes Equations}},
  author={Barman, Biswanath and Chatterjee, Debdeep and Ray, Rajendra K},
  journal={{arXiv preprint arXiv:2601.03613v1}},
  year={2026},
  url={https://arxiv.org/abs/2601.03613v1}
}

@article{qian2026distributed,
  title={{Distributed physics-informed neural networks via domain decomposition for fast flow reconstruction}},
  author={Qian, Yixiao and Liu, Jiaxu and Xia, Zewei and Chen, Song and Xu, Chao and Cai, Shengze},
  journal={{arXiv preprint arXiv:2602.15883}},
  year={2026}
}

@article{sarkar2026learning,
  title={{Learning hidden physics and system parameters with deep operator networks}},
  author={Sarkar, Dibakar Roy and Kag, Vijay and Pal, Birupaksha and Goswami, Somdatta},
  journal={{Computer Methods in Applied Mechanics and Engineering}},
  volume={456},
  pages={118926},
  year={2026},
  publisher={Elsevier}
}

@article{sirovich1987turbulence,
  title={{Turbulence and the dynamics of coherent structures. Part I: Coherent structures}},
  author={Sirovich, Lawrence and others},
  journal={Quarterly of applied mathematics},
  volume={45},
  number={3},
  pages={561--571},
  year={1987}
}

@article{jordan2024muon,
  title={{Muon: An optimizer for hidden layers in neural networks, 2024}},
  author={Jordan, Keller and Jin, Yuchen and Boza, Vlado and You, Jiacheng and Cesista, Franz and Newhouse, Laker and Bernstein, Jeremy},
  journal={URL https://kellerjordan. github. io/posts/muon},
  volume={6},
  number={3},
  pages={4},
  year={2024}
}

@article{liu2025muon,
  title={{Muon is scalable for llm training}},
  author={Liu, Jingyuan and Su, Jianlin and Yao, Xingcheng and Jiang, Zhejun and Lai, Guokun and Du, Yulun and Qin, Yidao and Xu, Weixin and Lu, Enzhe and Yan, Junjie and others},
  journal={arXiv preprint arXiv:2502.16982},
  year={2025}
}

@inproceedings{krishnapriyan2021characterizing,
  title={{Characterizing possible failure modes in physics-informed neural networks}},
  author={Krishnapriyan, Aditi and Gholami, Amir and Zhe, Shandian and Kirby, Robert and Mahoney, Michael W.},
  booktitle={{Advances in Neural Information Processing Systems}},
  volume={34},
  year={2021}
}

@article{wang2022when,
  title={{When and why {PINNs} fail to train: A neural tangent kernel perspective}},
  author={Wang, Sifan and Yu, Xinling and Perdikaris, Paris},
  journal={{Journal of Computational Physics}},
  volume={449},
  pages={110768},
  year={2022},
  doi={10.1016/j.jcp.2021.110768}
}

@article{cai2021espresso,
  author  = {Cai, Shengze and Wang, Zhicheng and Fuest, Frederik and Jeon, Young Jin and Gray, Callum and Karniadakis, George Em},
  title   = {{Flow over an espresso cup: inferring {3-D} velocity and pressure fields from tomographic background oriented {Schlieren} via physics-informed neural networks}},
  journal = {{Journal of Fluid Mechanics}},
  volume  = {915},
  pages   = {A102},
  year    = {2021},
  doi     = {10.1017/jfm.2021.135}
}

@article{cotter2009bayesian,
  author  = {Cotter, Simon L. and Dashti, Masoumeh and Robinson, James C. and Stuart, Andrew M.},
  title   = {{{Bayesian} inverse problems for functions and applications to fluid mechanics}},
  journal = {{Inverse Problems}},
  volume  = {25},
  number  = {11},
  pages   = {115008},
  year    = {2009},
  doi     = {10.1088/0266-5611/25/11/115008}
}

@article{ghattas2021learning,
  author  = {Ghattas, Omar and Willcox, Karen},
  title   = {{Learning physics-based models from data: perspectives from inverse problems and model reduction}},
  journal = {{Acta Numerica}},
  volume  = {30},
  pages   = {445--554},
  year    = {2021},
  doi     = {10.1017/S0962492921000064}
}

@article{xu2022physics,
  author  = {Xu, Kailai and Darve, Eric},
  title   = {{Physics constrained learning for data-driven inverse modeling from sparse observations}},
  journal = {{Journal of Computational Physics}},
  volume  = {453},
  pages   = {110938},
  year    = {2022},
  doi     = {10.1016/j.jcp.2021.110938}
}

@article{rosofsky2023mhd,
  author  = {Rosofsky, Shawn G. and Huerta, E. A.},
  title   = {{Magnetohydrodynamics with physics informed neural operators}},
  journal = {{Machine Learning: Science and Technology}},
  volume  = {4},
  number  = {3},
  pages   = {035002},
  year    = {2023},
  doi     = {10.1088/2632-2153/ace30a}
}

@article{stefani2006inverse,
  author  = {Stefani, Frank and Gundrum, Thomas and Gerbeth, Gunter and G{\"u}nther, Uwe and Xu, Mingtian},
  title   = {{Inverse problems in magnetohydrodynamics: theoretical and experimental aspects}},
  journal = {{Inverse Problems in Science and Engineering}},
  volume  = {14},
  number  = {4},
  pages   = {411--422},
  year    = {2006},
  doi     = {10.1080/17415970600573791}
}

@article{mo2025reconstructing,
  author  = {Mo, Yaxin and Magri, Luca},
  title   = {{Reconstructing unsteady flows from sparse, noisy measurements with a physics-constrained convolutional neural network}},
  journal = {{Physical Review Fluids}},
  volume  = {10},
  number  = {3},
  pages   = {034901},
  year    = {2025},
  doi     = {10.1103/PhysRevFluids.10.034901}
}

\end{document}